\documentclass{ecai} 
\usepackage{latexsym}
\usepackage{amssymb}
\usepackage{amsmath}
\usepackage{amsthm}
\usepackage{amsfonts}
\usepackage{empheq}
\usepackage{mathtools}
\usepackage{bm}
\usepackage{dsfont}
\usepackage{cases}
\usepackage[ruled,vlined,linesnumbered]{algorithm2e}
\usepackage{booktabs}
\usepackage{enumitem}
\usepackage{graphicx}
\usepackage{multirow}
\usepackage{subcaption} %  for subfigures environments 
\usepackage{resizegather}
\usepackage[normalem]{ulem}
\usepackage{soul}
\usepackage{color}
\usepackage{tikz}
\usepackage{pgfplots}
\usetikzlibrary{automata, positioning, arrows, patterns}
\pgfplotsset{compat=newest}

\newtheorem{theorem}{Theorem}
\newtheorem{lemma}{Lemma}
\newtheorem{remark}{Remark}
\newtheorem{example}{Example}
\newtheorem{corollary}[theorem]{Corollary}
\newtheorem{proposition}{Proposition}
\newtheorem{assumption}{Assumption}

\DeclareMathOperator*{\argmin}{arg\,min}
\DeclareMathOperator*{\E}{\mathbb{E}}
\DeclareMathOperator{\Var}{Var}

\DeclareMathOperator{\VarVec}{{\mathbb{V}ar}}

\newcommand{\BibTeX}{B\kern-.05em{\sc i\kern-.025em b}\kern-.08em\TeX}

\newcommand{\norm}[1]{\left\lVert#1\right\rVert}
\newcommand\scalar[2]{\langle #1, #2 \rangle}
\newcommand{\Span}[1]{\operatorname{Span}(#1)}
\newcommand{\R}{\mathbb{R}}

\newcommand{\fedavg}{\texttt{FedAvg}}
\newcommand{\fedavgu}{\texttt{U-FedAvg}}
\newcommand{\fedvarp}{\texttt{FedVARP}}
\newcommand{\fedvarpu}{\texttt{U-FedVARP}}
\newcommand{\fedstale}{\texttt{FedStale}}

\newcommand{\cl}{i}
\newcommand{\param}[2][]{%
    \ifthenelse{\equal{#1}{}}%
    {% First argument (subscript) is empty
        \ifthenelse{\equal{#2}{}}%
        {\bm{w}}% Both arguments are empty
        {\bm{w}^{(#2)}}% Only the first argument is empty
    }
    {% First argument (subscript) is not empty
        \ifthenelse{\equal{#2}{}}%
        {\bm{w}_{#1}}% Only the second argument is empty
        {\bm{w}_{#1}^{(#2)}}% Neither argument is empty
    }
}
\newcommand{\batch}[2][]{%
    \ifthenelse{\equal{#1}{}}%
    {% First argument (subscript) is empty
        \ifthenelse{\equal{#2}{}}%
        {\xi}% Both arguments are empty
        {\xi^{(#2)}}% Only the first argument is empty
    }
    {% First argument (subscript) is not empty
        \ifthenelse{\equal{#2}{}}%
        {\xi_{#1}}% Only the second argument is empty
        {\xi_{#1}^{(#2)}}% Neither argument is empty
    }
}
\newcommand{\lr}[1][]{%
    \ifthenelse{\equal{#1}{}}%
    {\eta}% if #1 is empty
    {\eta_{#1}}% if #1 is not empty
}
\newcommand{\sumcl}{\sum_{\cl=1}^{N}}
\newcommand{\avgcl}{\frac{1}{N}\sumcl}
\newcommand{\ber}[1][]{%
    \ifthenelse{\equal{#1}{}}%
    {\xi^{(t)}}% if #1 is empty
    {\xi_{#1}^{(t)}}% if #1 is not empty
}
\newcommand{\prob}{p_{\cl}} % bernoulli probability
\newcommand{\gps}[1]{\Delta^{(#1)}} % global pseudo-gradient
\newcommand{\pg}[2][]{%
    \ifthenelse{\equal{#1}{}}%
    {% First argument (subscript) is empty
        \ifthenelse{\equal{#2}{}}%
        {\bm{g}}% Both arguments are empty
        {\bm{g}^{(#2)}}% Only the first argument is empty
    }
    {% First argument (subscript) is not empty
        \ifthenelse{\equal{#2}{}}%
        {\bm{g}_{#1}}% Only the second argument is empty
        {\bm{g}_{#1}^{(#2)}}% Neither argument is empty
    }
}
\newcommand{\barpg}[2][]{%
    \ifthenelse{\equal{#1}{}}%
    {% First argument (subscript) is empty
        \ifthenelse{\equal{#2}{}}%
        {\bar{\bm{g}}}% Both arguments are empty
        {\bar{\bm{g}}^{(#2)}}% Only the first argument is empty
    }
    {% First argument (subscript) is not empty
        \ifthenelse{\equal{#2}{}}%
        {\bar{\bm{g}}_{#1}}% Only the second argument is empty
        {\bar{\bm{g}}_{#1}^{(#2)}}% Neither argument is empty
    }
}
\newcommand{\mem}[2][]{%
    \ifthenelse{\equal{#1}{}}%
    {% First argument (subscript) is empty
        \ifthenelse{\equal{#2}{}}%
        {\bm{h}}% Both arguments are empty
        {\bm{h}^{(#2)}}% Only the first argument is empty
    }
    {% First argument (subscript) is not empty
        \ifthenelse{\equal{#2}{}}%
        {\bm{h}_{#1}}% Only the second argument is empty
        {\bm{h}_{#1}^{(#2)}}% Neither argument is empty
    }
}
\newcommand{\lyap}[1]{\psi^{(#1)}} % lyapunov
\newcommand{\obj}[2][]{%
    \ifthenelse{\equal{#1}{}}%
    {% First argument (subscript) is empty
        \ifthenelse{\equal{#2}{}}%
        {F(\param[]{})}% Both arguments are empty
        {F(\param[]{#2})}% Only the first argument is empty
    }
    {% First argument (subscript) is not empty
        \ifthenelse{\equal{#2}{}}%
        {F_{#1}(\param[]{})}% Only the second argument is empty
        {F_{#1}(\param[]{#2})}% Neither argument is empty
    }
}
\newcommand{\grad}[3][]{%
    \ifthenelse{\equal{#1}{} \OR \equal{#2}{}}%
    {% Either the first or the second argument is empty
        \ifthenelse{\equal{#3}{}}%
        {% Third argument is also empty
            \ifthenelse{\equal{#1}{}}%
            {\nabla F({#2})}% Only the second argument is used
            {\nabla F_{#1}}% Only the first argument is used
        }
        {\PackageError{grad}{Cannot use third parameter without using both first two parameters}{}}% Error if third argument is used alone
    }
    {% Both first and second arguments are not empty
        \ifthenelse{\equal{#3}{}}%
        {\nabla F_{#1}({#2})}% First and second arguments are used
        {\nabla F_{#1}({#2,#3})}% All three arguments are used
    }
} % stochastic gradient
\newcommand{\gexp}{\E_{\xi^{(t)}, \batch[]{t} \mid \mathcal{H}^{(t)}}}
\newcommand{\berexp}{\E_{\xi^{(t)} \mid \batch[]{t}, \mathcal{H}^{(t)}}}

\AtBeginDocument{
\addtolength{\abovedisplayskip}{-2pt}
\addtolength{\abovedisplayshortskip}{-2pt}
\addtolength{\belowdisplayskip}{-2pt}
\addtolength{\belowdisplayshortskip}{-2pt}
\addtolength{\textfloatsep}{-2pt}
\addtolength{\dblfloatsep}{-2pt}
\addtolength{\abovecaptionskip}{2pt}
\addtolength{\belowcaptionskip}{-2pt}
}

\begin{document}

%%%%%%%%%%%%%%%%%%%%%%%%%%%%%%%%%%%%%%%%%%%%%%%%%%%%%%%%%%%%%%%%%%%%%%%%

\begin{frontmatter}

%%% Use this command to specify your submission number.
%%% In doubleblind mode, it will be printed on the first page.

\paperid{2590\vspace{-0.5cm}} 

%%% Use this command to specify the title of your paper.

% \title{Guidelines for Preparing a Paper for the \\
% European Conference on Artificial Intelligence}

% \title{FedStale: the multiple facets of Variance Reduction in Client Participation Heterogeneous Federated Learning}

% \title{The multiple facets of Variance Reduction \\ in Federated Learning}

% \title{The multiple facets of variance reduction in federated learning with client participation heterogeneity}

\title{\huge FedStale: leveraging stale client updates in federated learning}
%%% Use this combinations of commands to specify all authors of your 
%%% paper. Use \fnms{} and \snm{} to indicate everyone's first names 
%%% and surname. This will help the publisher with indexing the 
%%% proceedings. Please use a reasonable approximation in case your 
%%% name does not neatly split into "first names" and "surname".
%%% Specifying your ORCID digital identifier is optional. 
%%% Use the \thanks{} command to indicate one or more corresponding 
%%% authors and their email address(es). If so desired, you can specify
%%% author contributions using the \footnote{} command.

\author[A]{\fnms{Angelo}~\snm{Rodio}\footnote{This research was supported by the French government through the 3IA Côte d’Azur Investments by the National Research Agency (ANR-19-P3IA-0002), and by Groupe La Poste through the FedMalin Inria Challenge.}}
\author[A]{\fnms{Giovanni}~\snm{Neglia}\footnotemark\vspace{-0.2cm}}
% \author[B,C]{\fnms{Third}~\snm{Author}\orcid{....-....-....-....}} 

\address[A]{Inria, Université Côte d’Azur, France. \textit{Email: \{firstname.lastname\}@inria.fr}\vspace{-0.5cm}}
% \address[B]{Short Affiliation of Second Author and Third Author}
% \address[C]{Short Alternate Affiliation of Third Author}

\begin{abstract}
Federated learning algorithms, such as \fedavg{}, are negatively affected by data heterogeneity and partial client participation.
To mitigate the latter problem, global variance reduction methods, like \fedvarp{}, leverage stale model updates for non-participating clients.
These methods are effective under homogeneous client participation. 
Yet, this paper shows that, when some clients participate much less than others, aggregating updates with different levels of staleness can detrimentally affect the training process. 
Motivated by this observation, we introduce \fedstale{}, a novel algorithm that updates the global model in each round through a convex combination of ``fresh'' updates from participating clients and ``stale'' updates from non-participating ones.
By adjusting the weight in the convex combination, \fedstale{} interpolates between \fedavg{}, which only uses fresh updates, and \fedvarp{}, which treats fresh and stale updates equally.
Our analysis of \fedstale{} convergence yields the following novel findings: 
\emph{i)}~it integrates and extends previous \fedavg{} and \fedvarp{} analyses to heterogeneous client participation;
\emph{ii)}~it underscores how the least participating client influences convergence error; 
\emph{iii)}~it provides practical guidelines to best exploit stale updates, showing that their usefulness diminishes as data heterogeneity decreases and participation heterogeneity increases.
Extensive experiments featuring diverse levels of client data and participation heterogeneity not only confirm these findings but also show that \fedstale{} outperforms both \fedavg{} and \fedvarp{} in many settings.

\end{abstract}

\end{frontmatter}

\thispagestyle{plain}
\pagestyle{plain}

\section{Introduction}
Edge devices generate critical data for training machine learning models. However, centralizing this data is often impractical due to substantial communication overhead or simply impossible due to privacy regulations. Federated Learning (FL)~\cite{mcmahanCommunicationEfficientLearningDeep2017, konecnyFederatedLearningStrategies2017} offers a solution. In this paradigm, edge devices---also referred to as clients---collaborate to train a shared machine learning model. This collaboration is coordinated by a central server and allows data to remain decentralized, effectively addressing privacy and communication challenges.

In Federated Averaging (\fedavg{})~\cite{mcmahanCommunicationEfficientLearningDeep2017} and similar FL algorithms~\cite{liFederatedOptimizationHeterogeneous2020, reddiAdaptiveFederatedOptimization2023, acar2021federated, karimireddySCAFFOLDStochasticControlled2020}, clients perform multiple stochastic gradient descent (SGD) steps on their local datasets and then transmit their updated models to the central server. The server aggregates these client models to form a new global model, which is subsequently disseminated to the clients for further iterations.

The \emph{multiple} local updates 
performed by each client are crucial for enhancing communication efficiency. However, these updates can negatively impact the training process, as local client models  progressively diverge towards client-specific local minimizers due to \emph{data heterogeneity}~\cite{liConvergenceFedAvgNonIID2019, karimireddySCAFFOLDStochasticControlled2020}.

%at each client are needed to improve communication efficiency, but they adversely affect the training process because local client model diverge towards local minimizer due to \emph{statistical data heterogeneity} across clients.

%cause client models to progressively adapt to their specific datasets, introducing an error term in the convergence arising from \emph{data heterogeneity}.

% Data heterogeneity across clients
% % Heterogeneity in client data and participation dynamics
% negatively affects the performance of these algorithms, differentiating them from traditional centralized and distributed systems.
% The \emph{multiple} local updates, proposed to improve communication efficiency, cause client models to progressively adapt to their specific datasets, introducing an error term in the convergence arising from \emph{data heterogeneity}.

% \textbf{Option 2)} Heterogeneity in client data and participation dynamics negatively affects the convergence of \fedavg{}-based algorithms, deviating their performance from traditional centralized and distributed systems.
% A subtle, yet important distinction is allowing clients to perform \emph{multiple} local updates before communicating with the server. This approach improves communication efficiency but skews client models towards their local minimizers, introducing an error term in the convergence rate arising from \emph{data heterogeneity}.

Another significant source of heterogeneity arises from varying levels of client participation in the training process. This \emph{participation heterogeneity} is influenced by factors beyond server control~\cite{bonawitzFederatedLearningScale2019, wangFieldGuideFederated2021, yangAnarchicFederatedLearning2022}, such as diverse hardware specifications (CPU power, memory capacity), network connectivity types (WiFi, 5G), and power availability (e.g., clients may only participate when charging to prevent battery drain)~\cite{verbraeken2020survey, kairouzAdvancesOpenProblems2021a, ludwig2022federated}.
Despite this, much of the prior research assumes partial yet homogeneous client participation~\citep{liConvergenceFedAvgNonIID2019, karimireddySCAFFOLDStochasticControlled2020, liFederatedOptimizationHeterogeneous2020, yangAchievingLinearSpeedup2020, fraboniClusteredSamplingLowVariance2021, chenOptimalClientSampling2022, rizkFederatedLearningImportance2022, choConvergenceFederatedAveraging2023b}, overlooking the impact of such heterogeneity on the convergence of \fedavg{}-like algorithms. We identify and illustrate the two main problems due to heterogeneous participation.

% Heterogeneous client participation, stemming from varied device capabilities and connectivity, adds further complexity through diverse client participation dynamics.
% also degrades the convergence of \fedavg{}-based algorithms.
% While most existing research assumes client participation is homogeneous and under the server  control~\citep{liConvergenceFedAvgNonIID2019, karimireddySCAFFOLDStochasticControlled2020, liFederatedOptimizationHeterogeneous2020, yangAchievingLinearSpeedup2020, fraboniClusteredSamplingLowVariance2021, wangUnifiedAnalysisFederated2022, chenOptimalClientSampling2022, rizkFederatedLearningImportance2022, choConvergenceFederatedAveraging2023b}, in reality clients experience variable and unpredictable computing and network resources, making their participation heterogeneous. Moreover, if we consider the inherent randomness of whether a client successfully completes a global communication round, client participation dynamics are always beyond server control~\cite{bonawitzFederatedLearningScale2019, wangFieldGuideFederated2021, yangAnarchicFederatedLearning2022}.
% As a result, client participation in federated learning is partial across iterations and non-uniform across clients, leading to an additional error term in the convergence from \emph{partial and heterogeneous client participation}~\cite{wangUnifiedAnalysisFederated2022, wangLightweightMethodTackling2024}. 

First, heterogeneous participation risks biasing the global model in favor of clients that participate more frequently. Intuitively, when some clients participate more often than others, the global model may disproportionately reflect the local objectives of these more participating clients, thereby disadvantaging those who participate less. To counteract this bias, recent studies~\cite{wangUnifiedAnalysisFederated2022, wangLightweightMethodTackling2024} propose an unbiased version of \fedavg{}, which scales clients' model updates inversely with their participation frequency. By assigning greater weight to less participating clients, this approach ensures that the global model fairly represents all clients.

% these \emph{unbiased} aggregation strategies
Second, even if the potential bias is mitigated, partial and heterogeneous client participation still exacerbates the variability of the learning process. Indeed, the unbiased scaling process amplifies variations in the magnitude of client updates, leading to increased variance in the learned model and slower convergence. Although a few recent works focus on global variance reduction \citep{guFastFederatedLearning2021, yangAnarchicFederatedLearning2022, jhunjhunwalaFedvarpTacklingVariance2022, yanFederatedOptimizationIntermittent2024}, they are limited to scenarios involving homogeneous client participation. Specifically, \fedvarp{} (\uline{Fed}erated \uline{VA}riance \uline{R}eduction for \uline{P}artial client participation) \citep{jhunjhunwalaFedvarpTacklingVariance2022} utilizes the most recent, albeit potentially stale, model updates in place of unavailable updates from non-participating clients. \fedvarp{} has demonstrated, both theoretically and empirically, its capability to effectively lower variance and consistently outperform \fedavg{} in settings with partial yet homogeneous client participation. It is anticipated to perform similarly well even in heterogeneous settings \cite{jhunjhunwalaFedvarpTacklingVariance2022}. However, when client participation varies widely, global variance reduction methods, including \fedvarp{}, must address the challenge of updates of varying staleness---a complex issue that remains unexplored and is the focus of this paper.

% the server must aggregate client updates of varying staleness, and the impact of client participation heterogeneity on the convergence of variance reduction methods remains unexplored.
% While this approach requires additional memory to store updates from all clients, servers typically possess sufficient resources to meet this demand.
This paper specifically addresses the following questions: \\
\emph{
1) Is it really true that \fedvarp{} outperforms the unbiased \fedavg{} under \uline{heterogeneous client participation}? \\
2) Assuming that each method may be preferable in different settings, can we design an \uline{unbiased} algorithm that combines \uline{fresh and stale updates} and adapts to specific levels of participation heterogeneity?
}

Addressing these questions is challenging and requires a deeper understanding of how stale client updates influence convergence. 

\textbf{Our contributions.} We thoroughly analyze this problem and make the following novel contributions: \\
1) We analytically and experimentally refute the belief that \fedvarp{} consistently outperforms \fedavg{}. Our convergence analysis reveals that leveraging stale updates can be either beneficial or detrimental, depending on the specific level of \uline{client data and participation heterogeneity}. \\
2) We propose \fedstale{} (\uline{Fed}erated Averaging with \uline{Stale} Updates), a novel FL algorithm that updates the global model through a convex, unbiased combination of fresh and stale updates, parameterized by a weight $\beta$.
\fedstale{} spans the spectrum from \fedavg{} ($\beta=0$, exclusively fresh updates) to \fedvarp{} ($\beta=1$, equal weighting of fresh and stale updates). Our analysis provides guidelines to tune the parameter~$\beta$ to match  specific data and client participation heterogeneity scenarios. %We provide practical guidelines for selecting $\beta$. 
\\
3) We evaluate \fedavg{}, \fedvarp{}, and \fedstale{} across multiple levels of client data and participation heterogeneity. \fedstale{} outperforms both \fedavg{} and \fedvarp{} across the vast majority of heterogeneity levels examined.

The remainder of this paper is organized as follows. Section~\ref{sec:problem} reviews the problem and related work.
% , in particular on global variance reduction for federated learning. 
Section~\ref{sec:algorithm} introduces \fedstale{}, our staleness-aware algorithm, through a motivating example. Section~\ref{sec:analysis} provides a convergence analysis of \fedstale{} under heterogeneous client participation. \fedstale{} is extensively evaluated in Section~\ref{sec:experiments}, and Section~\ref{sec:conclusion} concludes the paper. Proof outlines are included in the Appendix, while detailed proofs are available in the supplementary material.

\begin{figure*}[!ht]
\centering
\begin{subfigure}[t]{0.245\textwidth}
    \centering
    \includegraphics[width=\textwidth]{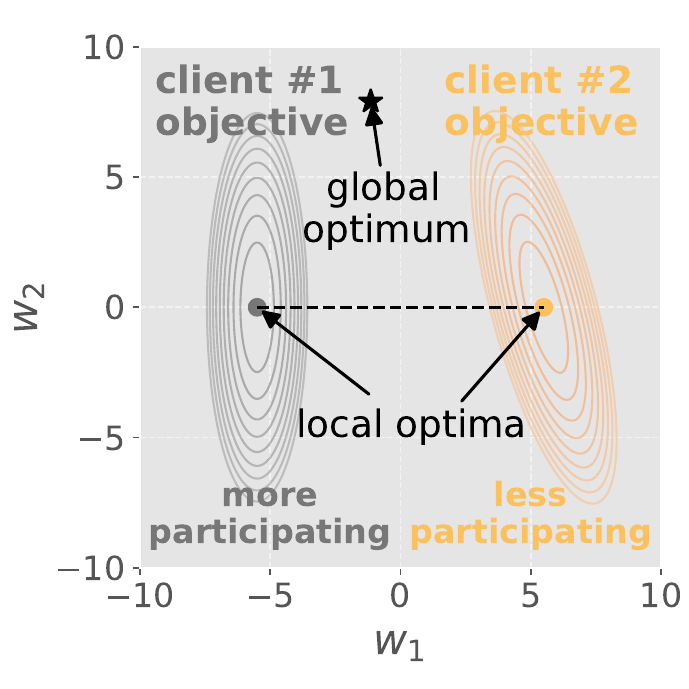}
    \caption{\footnotesize Optimization objectives.} 
    \label{fig:1}
\end{subfigure}
\hfill
\begin{subfigure}[t]{0.245\textwidth}
    \centering
    \includegraphics[width=\textwidth]{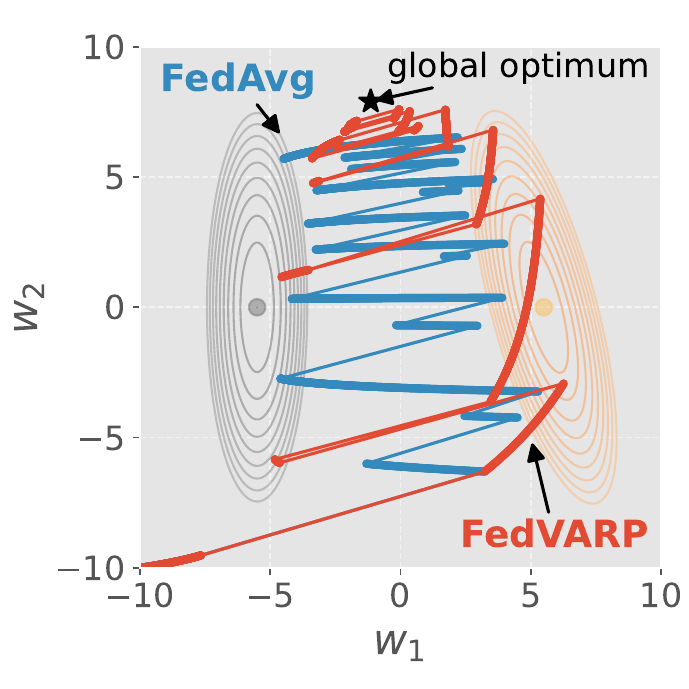}
    \caption{\footnotesize Trajectories: FedAvg, FedVARP.} 
    \label{fig:2}
\end{subfigure}
\hfill
\begin{subfigure}[t]{0.245\textwidth}
    \centering
    \includegraphics[width=\textwidth]{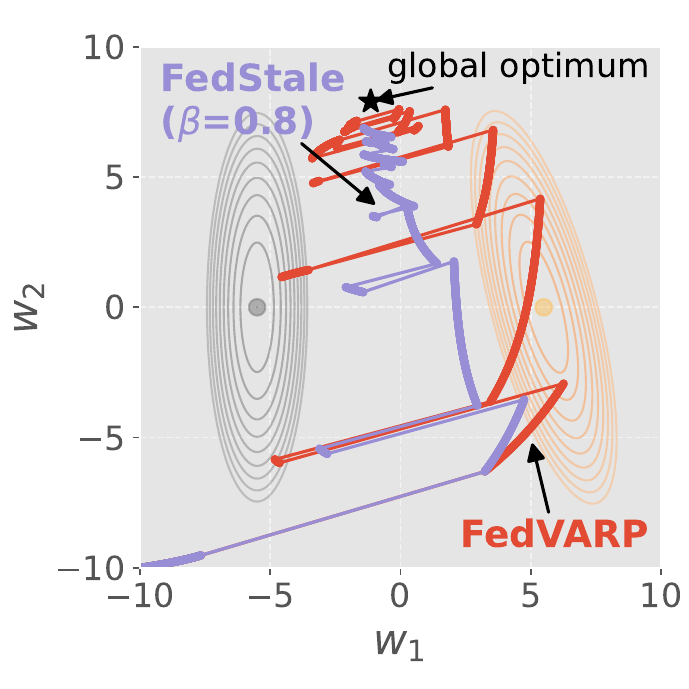}
    \caption{\footnotesize Trajectories: FedVARP, FedStale.} 
    \label{fig:3}
\end{subfigure}
\hfill
\begin{subfigure}[t]{0.245\textwidth}
    \centering
    \includegraphics[width=\textwidth]{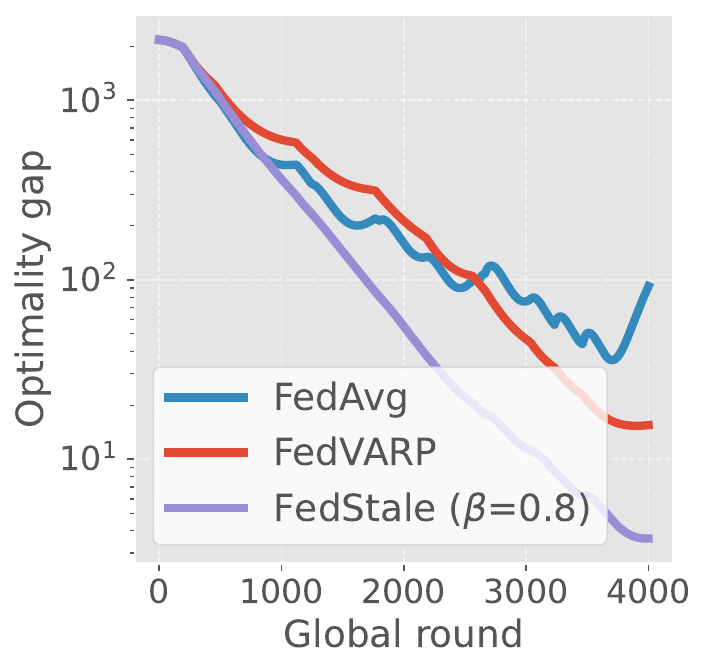}
    \caption{\footnotesize Convergence curves.} 
    \vspace{0.5cm}
    \label{fig:4}
\end{subfigure}
\caption{
\footnotesize
Comparison of FedAvg, FedVARP, and FedStale in a two-clients, 2D quadratic setting with \emph{heterogeneous} client participation. 
\textbf{Fig.~\ref{fig:1}:}
Contour plots of client objectives, their local optima, and global optimum. Client participation ratio is $p_1/p_2=100$. 
\textbf{Fig.~\ref{fig:2}:}
Trajectories by FedAvg and FedVARP over T=4000 rounds with K=5 local iterations each. While both algorithms target the global optimum, FedAvg struggles with large variance and FedVARP follows suboptimal paths due to stale updates.
\textbf{Fig.~\ref{fig:3}:} FedStale ($\beta$=0.8) follow a more stable  trajectory under heterogeneous client participation.
\textbf{Fig.~\ref{fig:4}:} 
Learning curves of FedAvg, FedVARP, and FedStale over 10 runs. With a lower weight on stale updates  ($\beta$=0.8), FedStale achieves faster convergence to the global optimum. 
}
\vspace{0.4cm}
\label{fig:example}
\end{figure*}

\section{Problem Definition and Background}
\label{sec:problem}

We consider a FL setting where $N$ clients, each client $i$ equipped with a dataset $\mathcal{D}_i$ consisting of $n_i$ samples, collaboratively learn the parameters $\param{}\in\R^d$ of a global ML model
(e.g., the weights of a neural network).
Orchestrated by a central server, these clients cooperate to minimize the \emph{global} objective:\footnote{Objective~\eqref{eq:g_obj} corresponds to a ``per-client fairness'' criterion,
while ``per-sample fairness'' weights each local loss by the client's number of samples. 
% Although we consider the former for clarity, 
Our analysis can be directly extended to any weighted sum of local losses.
% Another common choice is to weight each local objective proportionally to the client's number of samples (per-sample fairness). 
% We consider objective~\eqref{eq:g_obj} for the sake of concreteness, but the analysis in this paper can be immediately extended to any weighted sum of local objectives.
}
\begin{align}
    \min_{\param{} \in \R^d} \obj{}{} \triangleq \avgcl 
    \left[ \obj[\cl]{}{} \triangleq \frac{1}{n_i} \sum_{\xi_i \in \mathcal{D}_i} f(\param{}, \xi_{\cl}) \right],
    \label{eq:g_obj}
\end{align}
where client $i$ has \emph{local} objective $\obj[\cl]{}{}$ and $f(\param{},\xi_{\cl})$ is the loss function evaluating model performance on data sample $\xi_{\cl}\in\mathcal{D}_i$.
% Although Eq.~\eqref{eq:g_obj} can be extended to a weighted average with positive coefficients, we simplify notation and integrate such coefficients directly into $\obj[\cl]{}{}$.

In this paper, we consider algorithms obeying the general operation in Algorithm~\ref{alg:general_algo}, differing in the \texttt{ComputeUpdate}() procedure.

% A key aspect of Problem~\eqref{eq:g_obj} is local data retention.
% As a consequence, each client's \emph{local} objective $\obj[\cl]{}{}$ can be computed and optimized only at the $i$-th client level.
\fedavg{} iteratively solves Problem~\eqref{eq:g_obj} while maintaining data decentralization. 
Model training involves $T$ rounds of communication between server and clients: 
at the beginning of each round $t>0$, the server sends the current global model, $\param{t}$, to a random subset of participating clients $\mathcal{S}^{(t)}$, usually $|\mathcal{S}^{(t)}|\ll N$. 
Each client in $\mathcal{S}^{(t)}$ runs multiple ($K\geq1$) iterations of local stochastic gradient descent (SGD) on its local dataset:
\begin{align*}
    \param[\cl]{t,k+1} = \param[\cl]{t,k} - \lr[c] \grad[\cl]{\param[\cl]{t,k}}{\batch[\cl]{t,k}} 
    \quad
    \text{for } k=0,\dots,K-1
\end{align*}
producing the local model $\param[\cl]{t,K}$,
and the sends the model update $\Delta_{\cl}^{(t)} = (\param{t}-\param[\cl]{t,K})$ to the server. The server aggregates these client updates into the \emph{global} update:
\begin{align}
    \Delta_{\fedavg{}}^{(t)} = \frac{1}{|\mathcal{S}^{(t)}|} \sum_{i\in\mathcal{S}^{(t)}} \Delta_{\cl}^{(t)},
    \label{eq:aggregation_biased}
\end{align} 
and then applies this update to the previous global model in a manner similar to a gradient descent step to produce the new global model $\param{t+1} = \param{t} - \lr[s] \Delta_{\fedavg{}}^{(t)}$.

\begin{algorithm}[t]
\caption{FL algorithm with pluggable global update}
\label{alg:general_algo}

\textbf{Input: }{$\param{1}$, $K$, $\lr[s]$, $\lr[c]$ };
\textbf{Output: }{$\{\param{t} : \forall t\}$}

\For{$t = 1, \dots, T$}{

    \For{$\cl \in \mathcal{S}^{(t)}$, in parallel \label{alg-line:client_opt}}{

    $\param[\cl]{t,0} \gets \param{t}$

    \For{$k = 0,1\dots,K-1$}{
    
    $\param[\cl]{t,k+1} \gets \param[\cl]{t,k} - \lr[c] \grad[\cl]{\param[\cl]{t,k}}{\batch[\cl]{t,k}}$

    }

    $\Delta_{\cl}^{(t)} \gets (\param[]{t} - \param[\cl]{t,K})$

    }

    $\Delta^{(t)} 
    \gets 
    \texttt{ComputeUpdate}(\{ \Delta_{\cl}^{(t)} \}_{\cl \in \mathcal{S}^{(t)}}, \dots)$

    $\param{t+1} = \param{t} - \lr[s] \Delta^{(t)}$ \label{alg-line:server_opt}

}
\end{algorithm}

% Set $\param[\cl]{t,0} = \param{t}$

% \For{$k = 0,1\dots,K-1$}{

%     Sample data $\xi_i^{(t,k)} \stackrel{iid}{\sim} \mathcal{D}_i$

%     Compute stochastic gradient $\grad[\cl]{\param[\cl]{t,k}}{\batch[\cl]{t,k}}$
    
%     Update $\param[\cl]{t,k+1} = \param[\cl]{t,k} - \lr[c] \grad[\cl]{\param[\cl]{t,k}}{\batch[\cl]{t,k}}$

% }
% \textbf{return }{$(\param{t} - \param[\cl]{t,K})$}
% \end{algorithm}

Following standard assumptions~\cite{wangUnifiedAnalysisFederated2022, rodioFederatedLearningHeterogeneous, wangLightweightMethodTackling2024}, we model client participation heterogeneity through the \emph{participation probability} $\prob$:
\begin{align}
    \prob \triangleq \E_{\mathcal{S}^{(t)}} \left[ \mathbb{P}(\cl \in \mathcal{S}^{(t)}) \right].
\end{align}
%a measure of how often client $\cl$ is expected to participate in training.
 
When client participation is \emph{homogeneous} ($\prob = p , \forall \cl$), $\E_{\mathcal{S}^{(t)}}[\Delta_{\fedavg{}}^{(t)}] = \frac{1}{N} \sum_{i=1}^{N} \Delta_{\cl}^{(t)}$. Under this condition,  
Eq.~\eqref{eq:aggregation_biased} is then an unbiased estimator of the model update as if \emph{all} clients were to participate~\cite{liConvergenceFedAvgNonIID2019, fraboniClusteredSamplingLowVariance2021}. This ensures that the final model fairly represents all clients.

Conversely, under \emph{heterogeneous} participation, where probabilities~$\{\prob\}$ vary among clients, Eq.~\eqref{eq:aggregation_biased} becomes a biased estimator of $\frac{1}{N} \sum_{i=1}^{N} \Delta_{\cl}^{(t)}$. This \emph{bias} in the global update tends to overrepresent clients that participate more frequently, disadvantaging those that participate less. Participation heterogeneity can then lead to objective inconsistency, causing \fedavg{} to effectively minimize the \emph{biased} objective:
\begin{align}
    \widetilde{F}(\param{}) = \avgcl \frac{\prob}{\sum_{j=1}^{N} p_j} \obj[\cl]{},
    \label{eq:obj_inconsitency}
\end{align}
which may arbitrarily deviate from the global objective~\eqref{eq:g_obj}.

To effectively minimize objective~\eqref{eq:g_obj} when client participation is heterogeneous, recent works~\cite{fraboniClusteredSamplingLowVariance2021, wangUnifiedAnalysisFederated2022, 
fraboniGeneralTheoryFederated2023,
rodioFederatedLearningHeterogeneous, wangLightweightMethodTackling2024} have discussed the need to debias $\Delta_{\fedavg{}}^{(t)}$. Specifically, Eq.~\eqref{eq:aggregation_biased} has been modified into Eq.~\eqref{eq:fedavg}, resulting in an unbiased version of \fedavg{}, denoted here as \fedavgu{}~\cite{wangUnifiedAnalysisFederated2022, rodioFederatedLearningHeterogeneous, wangLightweightMethodTackling2024}: 
\begin{align}
    \Delta_{\fedavgu{}}^{(t)} 
    = 
    \frac{1}{N} \sum_{i\in\mathcal{S}^{(t)}} \frac{\Delta_{\cl}^{(t)}}{p_i}.
    \label{eq:fedavg}
\end{align}
Intuitively, reweighting each client update by $p_i^{-1}$ compensates for less participating clients by amplifying their update when they do participate.
\fedavgu{} naturally extends \fedavg{} to accommodate heterogeneous client participation---reducing to \fedavg{} when participation is uniform ($\prob = \frac{|\mathcal{S}^{(t)}|}{N}, \forall \cl$)---and effectively \emph{unbiases} the global update ($\E_{\mathcal{S}^{(t)}}[\Delta_{\fedavgu{}}^{(t)}] = \frac{1}{N} \sum_{i=1}^{N} \Delta_{\cl}^{(t)}$). However, 
it also introduces a drawback: the variance of each client updates is now proportional to $p_i^{-2}$. As participation probabilities decrease, this variance rapidly increases, becoming the dominant factor that slows down \fedavgu{}'s convergence~\cite{rodioFederatedLearningHeterogeneous, wangLightweightMethodTackling2024}.
% (we will detail more in Sections~\ref{sec:algorithm} and~\ref{sec:analysis}).

A few recent works have addressed the variance introduced by partial client participation through global variance reduction, leveraging stale updates to compensate for non-participating clients~\citep{guFastFederatedLearning2021, yangAnarchicFederatedLearning2022, jhunjhunwalaFedvarpTacklingVariance2022, yanFederatedOptimizationIntermittent2024}. 
These methods were originally proposed for \emph{homogeneous} participation and, if applied in their original form,  would introduce a bias when client participation becomes heterogeneous.
Fortunately, unbiasing them to work in \emph{heterogeneous} participation scenarios is straightforward, similar to what was done for \fedavg{} in Eq.~\eqref{eq:fedavg}. We select \fedvarp{}~\cite{jhunjhunwalaFedvarpTacklingVariance2022} as the representative algorithm and adapt it into \fedvarpu{} (Unbiased \fedvarp{}).

In \fedvarpu{}, the server retains the most recent, though potentially stale, update for each client:
\begin{align}
    \mem[\cl]{t} =
    \begin{cases}
        \Delta_{\cl}^{(t-1)} & \text{if } \cl \in \mathcal{S}^{(t-1)} \\
        \mem[\cl]{t-1} & \text{otherwise}
    \end{cases},
\end{align}
and then uses these stale updates as proxies for missing contributions from non-participating clients in the current round:
\begin{align}
    \Delta_{\fedvarpu{}}^{(t)} = \avgcl \mem[\cl]{t} + \frac{1}{N} \sum_{\cl \in \mathcal{S}^{(t)}} \frac{\Delta_{\cl}^{(t)} - \mem[\cl]{t}}{p_i}.
    \label{eq:fedvarp}
\end{align}
Unlike \fedavgu{}, which essentially ignores non-participating clients, \fedvarpu{} utilizes their last updates, albeit stale, when they do not participate in the training process. When they participate again, \fedvarpu{} 
subtracts these stale updates to eliminate any inconsistency caused by using stale information, and applies the fresh update. Both corrections are reweighed by $p_i^{-1}$, similarly to \fedavgu{}, ensuring that $\E\left[\Delta_{\fedvarpu{}}^{(t)}\right] = \avgcl \Delta_\cl^{(t)}$.
\fedvarpu{}'s aggregation~\eqref{eq:fedvarp} is then \emph{unbiased}. Moreover, by leveraging stale updates for non-participating clients, \fedvarpu{} acts as a \texttt{SAGA}-like~\cite{defazioSAGAFastIncremental2014} variance reduction method, aiming to reduce the variance caused by partial client participation.
This strategy incurs an additional memory cost of $N \times d$, which must be allocated by the server.

% In its original formulation, also \fedvarp{}~\citep[Eq.~(4)]{jhunjhunwalaFedvarpTacklingVariance2022} , but it is easy to modify  (we call it \fedvarpu{}). 
% We focus exclusively on \emph{unbiased} approaches, for which, to the best of our knowledge, \fedvarp{}~\cite{jhunjhunwalaFedvarpTacklingVariance2022} is the sole representative.
% These methods and are particularly prone to introducing bias in the global model. 

%Current understanding
Although variance reduction methods like \fedvarp{} are often believed to outperform simpler algorithms like \fedavg{} under partial and heterogeneous client participation, as suggested for example in \cite{jhunjhunwalaFedvarpTacklingVariance2022, wangLightweightMethodTackling2024}, theoretical support for this belief has been provided only for homogeneous participation scenarios~\cite[Theorem~2]{jhunjhunwalaFedvarpTacklingVariance2022} and empirical results do not lead to definitive conclusions~\cite[Table~5]{wangLightweightMethodTackling2024}. 

% Current theoretical results on \fedvarp{} convergence are limited to homogeneous client participation, where it has been proven that \fedvarp{} outperforms \fedavg{}.
% % Current research gap
% %\emph{Empirical} beliefs
% While there is 

% some experimental results~\cite{jhunjhunwalaFedvarpTacklingVariance2022, wangLightweightMethodTackling2024, wangDELTADiverseClient2023} suggest that these findings may also apply to heterogeneous partecipation settings, theoretical evidence is still missing.
% Contribution
This paper challenges the presumed superiority of \fedvarpu{} under client participation heterogeneity. Both theoretical and experimental contributions indicate that the relative effectiveness of \fedvarpu{} and \fedavgu{} varies depending on the specific levels of data heterogeneity and client participation heterogeneity.

In the remainder of the paper, we focus on the unbiased versions of the two algorithms. However, for simplicity, we refer to them simply as \fedvarp{} and \fedavg{}.

% current theoretical results of \fedvarp{} convergence are limited to homogeneous setting and prove that \fedvarp{} indeed outperforms \fedavg{} in this scenario.
% Moreover, in the literature there is a prevailing belief that \fedvarp{} outperforms \fedavg{} also in heterogeneous settings, but this was only proved in the homogeneous setting and this paper questions this belief.
% A theoretical analysis of \fedvarpu{} convergence is a contribution of this paper.

\section{The FedStale Algorithm}
\label{sec:algorithm}
% We introduce the following example to question the superiority of \fedvarp{} under client participation heterogeneity.

We start questioning the expected superiority of \fedvarp{} under client participation heterogeneity though the following illustrative example.

\subsection{A motivating example}
\label{sec:algorithm:example}

Figure~\ref{fig:1} considers a two-clients scenario with quadratic bidimensional objectives $\{\obj[\cl]{}, i=1,2, \param{} \in \mathbb R^2\}$.
The global optimum~$\param{}^*$, minimizer of $\obj{}\triangleq\frac{1}{2}\obj[1]{}+\frac{1}{2}\obj[2]{}$, does not align with the average of the local optima $\{\param[\cl]{}^*, i=1,2\}$. 
Clients participate according to Bernoulli distributions with parameters $\{\prob, i=1,2\}$ and a skewed participation ratio $p_1/p_2=100$.

Figure~\ref{fig:2} compares the model trajectories of \fedavg{} and \fedvarp{} over $T=4000$ rounds, starting from $\param{1}$=(-10,-10) and running the experiments with same clients participation processes for comparability. 
Both algorithms initially share the same trajectory, driven solely by the participation of client 1, who targets $\param[1]{}^*$.
When client~2 first participates, the global update dramatically shifts towards $\param[2]{}^*$ due to the reweighting factor $1/p_2$. 
As client~2 stops participating, the two trajectories diverge: \fedavg{} reverts to approaching $\param[1]{}^*$, influenced only by the participating client~1, while \fedvarp{} continues to factor in stale updates from client~2. 
Both algorithms eventually converge to the global optimum $\param[]{}^*$, consistently with the fact that both Eqs.~\eqref{eq:fedavg} and~\eqref{eq:fedvarp} are \emph{unbiased}. However, \fedavg{} suffers large variance and slow convergence due to significant shifts whenever client 2 participates, whereas \fedvarp{} is affected by progressively more outdated updates from the less participating client, also resulting in  suboptimal trajectories with abrupt corrections.
Figure~\ref{fig:4} compares the losses over these trajectories and confirms that both \fedavg{} and \fedvarp{} exhibit high variability for distinct reasons.
A hybrid approach that combines these two dynamics can potentially improve overall performance.

\subsection{A convex combination of fresh and stale updates}

In Figs.~\ref{fig:3} and~\ref{fig:4}, a convex combination of \fedavg{} and \fedvarp{} updates with a weighting parameter $\beta=0.8$ results in a more stable trajectory and achieves faster convergence than either algorithm alone.
This suggests that, in environments with heterogeneous client participation, parameterizing the weight to stale updates allows us to interpolate the two negative extremes of large variance (\fedavg{}) and outdated trajectories (\fedvarp{}). 
Motivated by these observations, we propose \fedstale{} (\uline{Fed}erated Averaging with \uline{Stale} Updates), outlined in Algorithm~\ref{alg:fedstale}.
In each round, \fedstale{} updates the global model through a convex combination of fresh and stale updates, with parameter $\beta$ in the range $[0,1]$:
\begin{align}
    \Delta_{\fedstale{}}^{(t)} 
    &= (1-\beta) \Delta_{\fedavg{}}^{(t)} + \beta \Delta_{\fedvarp{}}^{(t)} \label{eq:fedstale_convex} \\
    &= \avgcl \beta \mem[\cl]{t} + \frac{1}{N} \sum_{\cl \in \mathcal{S}^{(t)}} \frac{\Delta_{\cl}^{(t)} - \beta \mem[\cl]{t}}{p_i}. \label{eq:fedstale_weight}
\end{align}

\fedstale{} interpolates between the behaviors of \fedavg{} when $\beta=0$ and \fedvarp{} when $\beta=1$, merging the two algorithms into a single, versatile framework.
% Second, \fedstale{} extends \fedvarp{} to heterogeneous client participation, resolving the objective inconsistency issues previously noted in Eq.~\eqref{eq:fedvarp}.
Moreover, by adjusting~$\beta$, \fedstale{} can control the influence of stale updates, allowing for a continuum of behaviors that adapts with the specific level of client data and participation heterogeneity.

\begin{algorithm}[t]
\caption{Global update computation in \fedstale{}}
\label{alg:fedstale}

\SetKwProg{Proc}{Procedure:}{}{}
\SetKwFunction{FComputeUpdate}{ComputeUpdate}
\DontPrintSemicolon

\textbf{Input: }{$\{ \mem[\cl]{1}=\bm{0}, \prob : \forall \cl \}$, $\beta$};
\textbf{Output: }{$\{\Delta^{(t)}_{\scriptscriptstyle \fedstale{}} : \forall t\}$}

\For{$t = 1, \dots, T$}{

\textbf{Procedure} \texttt{ComputeUpdate}($\{\Delta_{\cl}^{(t)}\}_{\scriptscriptstyle \cl \in \mathcal{S}^{(t)}}$, $\beta$):

    $\Delta^{(t)}_{\scriptscriptstyle \fedstale{}} 
    \gets \frac{\beta}{N} \sum_{\scriptscriptstyle i=1}^{\scriptscriptstyle N} \mem[\cl]{t} + \frac{1}{N} \sum_{\scriptscriptstyle \cl \in \mathcal{S}^{(t)}} ( \Delta_{\cl}^{(t)} - \beta \mem[\cl]{t} )/\prob$

    \For{$\cl \in \mathcal{S}^{(t)}$}{

    $\mem[\cl]{t+1} \gets \Delta_{\cl}^{(t)}$ \tcp{Update memory}

    }
        
}
\end{algorithm}

\textbf{Requirements.} 
In its operation, \fedstale{} maintains the same computational and communication complexity as \fedvarp{}, with tuning $\beta$ as the only additional requirement. Section~\ref{sec:experiments} shows that a coarse adjustment of $\beta$ (e.g., $\beta \in {0,0.2,0.5,0.8,1}$) provides reasonably good performance across varied settings, thus eliminating the need for fine-tuning.

As for storage requirements, \fedstale{} mirrors \fedvarp{} and other global variance reduction methods by storing stale updates from \emph{all clients} at the server.
Typically, servers possess more resources than clients, mitigating potential storage issues. Methods that avoid additional storage would otherwise escalate computational and communication demands on clients or necessitate \emph{full client participation} in certain rounds---a requirement that may be overly demanding or even impractical, as will be discussed in the following section.

\subsection{Comparison to related work}
\label{sec:algorithm:comparison}

% Since the development of \texttt{SAG}~\citep{schmidtMinimizingFiniteSums2017} and \texttt{SAGA}~\citep{defazioSAGAFastIncremental2014}, which inspired \fedvarp{}, 
We discuss variance reduction methods emerged for centralized and distributed optimization. Some have already been adapted to federated learning, while others are discusses for potential applicability.

% numerous variance-reduction methods have emerged for centralized stochastic problems that eschew additional storage. We discuss existing or potential adaptations to federated learning, focusing on the heterogeneous client participation.

\textbf{\texttt{FedLaAvg}~\citep{yanFederatedOptimizationIntermittent2024}, \texttt{MIFA}~\citep{guFastFederatedLearning2021}, \texttt{AFA-CD} and \texttt{AFA-CS}~\citep{yangAnarchicFederatedLearning2022}}, similarly to \fedvarp{}, address partial yet homogeneous client participation by storing the stale model updates for each client. However, their approach of uniformly weighting fresh and stale updates, through a \texttt{SAG}-based~\citep{schmidtMinimizingFiniteSums2017} global variance reduction step, \emph{biases} the global model leading to objective inconsistency.

\textbf{\texttt{SVRG}-based Variance Reduction Methods~\citep{johnsonAcceleratingStochasticGradient2013, leiNonconvexFiniteSumOptimization2017, nguyenSARAHNovelMethod2017, fangSPIDERNearOptimalNonConvex2018}} trade storage demands with computation needs by periodically calculating, in centralized settings, full or large-batch gradients. Although offering superior theoretical performance over \texttt{SAGA}-based~\citep{defazioSAGAFastIncremental2014} variance reduction methods like~\fedvarp{}, their extension to FL settings is constrained by the impractical requirement for \emph{all clients} to participate simultaneously during certain training rounds.

\textbf{\texttt{SCAFFOLD}~\citep{karimireddySCAFFOLDStochasticControlled2020}}
uses control variates to correct for data heterogeneity errors. Adapting this method to handle participation heterogeneity would require clients to perform local \texttt{SAGA}-like~\citep{defazioSAGAFastIncremental2014} corrections, thereby \emph{doubling the communication overhead} as clients must transmit both the model updates and correction vectors to the server. While this extension remains a topic for future research, we underscore the additional communication complexity involved.

In contrast to previous work, \fedstale{}, much like \fedvarp{}, performs corrections at the server level without involving clients in variance reduction, thus maintaining the same communication overhead as \fedavg{} and still matching \texttt{SCAFFOLD}’s convergence rates.

\section{Convergence Analysis}
\label{sec:analysis}

% Analyzing the convergence of \fedstale{}, we aim to investigate the impact of stale updates when client participation is heterogeneous.

% \subsection{Main assumptions}

\begin{assumption}[$L$-smoothness]
    \label{asm:smoothness}
    The local objective functions are $L$-smooth, i.e., 
    $\norm{\nabla F_i(\bm{u}) - \nabla F_i(\bm{v})} \leq L \norm{\bm{u} - \bm{v}}$,
    $\forall \bm{u}, \bm{v}, i$.
\end{assumption}

\begin{assumption}[Bounded variance at client level]
    \label{asm:stochastic_gradient}
    The stochastic gradient at each client is an unbiased estimator of the local gradient: $\E_{\xi_i\sim \mathcal{D}_i} [\nabla F_i(\bm{w},\xi_i)] = \nabla F_i(\bm{w})$, 
    with bounded variance: $\E_{\xi_i\sim \mathcal{D}_i} \norm{\nabla F_i(\bm{w},\xi_i) - \nabla F_i(\bm{w})}^2 \leq \sigma^2$, $\forall \bm{w}, i$.
    % The stochastic gradient noise is independent across clients, rounds, and local steps. 
\end{assumption}

\begin{assumption}[Bounded variance across clients]
    \label{asm:heterogeneity}
    There exists a constant $\sigma_g^2 > 0$ such that the difference between the local gradient at the $i$-th client and the global gradient is bounded, that is $\norm{\nabla F_i(\bm{w}) - \nabla F(\bm{w})}^2 \leq \sigma_g^2$, $\forall \bm{w}, i$.
\end{assumption}

% We denote the participation of the $i$-th client in the $t$-th round with the indicator variable $\mathds{1}_{i}^{(t)}$, which equals one if and only if the client participates, and zero otherwise.
\begin{assumption}[Partial and heterogeneous client participation]
    \label{asm:participation}
    % The participation outcome $\mathds{1}_{i}^{(t)}$ is independent for each client and round. 
    % Client participation $\mathds{1}_{i}^{(t)}$ follows a Bernoulli distribution with parameter $p_i$, 
    % independent for each client and round.
    % $\mathds{1}_{i}^{(t)} \sim \Bernoulli(p_i)$
    % The expected value $\E[\mathds{1}_{i}^{(t)}]=p_i$ denotes the participation probability for client $i$, which can vary across clients.
    In each round $t$, client $\cl$ participates with a probability $\prob$, independently of previous rounds and other clients. 
\end{assumption}

Assumptions~\ref{asm:smoothness}–\ref{asm:heterogeneity} are standard in federated learning convergence analysis~\cite{yangAchievingLinearSpeedup2020, wangUnifiedAnalysisFederated2022, choConvergenceFederatedAveraging2023b}. 
The terms $\sigma^2$ and $\sigma_g^2$ denote the variances from \emph{stochastic gradients} and \emph{data heterogeneity}, respectively. 
Assumption 4, which models \emph{client participation heterogeneity}, also appears in some prior works~\cite{wangUnifiedAnalysisFederated2022, wangLightweightMethodTackling2024}.
Exploring more complex participation dynamics, following the methodologies in~\cite{wangUnifiedAnalysisFederated2022, rodioFederatedLearningHeterogeneous}, remains a task for future research.

% \subsection{Main results: upper and lower bounds}

We start by providing an upper bound for \fedstale{}'s convergence. 
To focus the discussion on our main results, we defer proof outlines to the appendix and detailed proofs to the supplementary material.
% \textbf{Shall we write that the theoretical analysis is far from a straightforward combination of existing FL analyses?}

\begin{theorem}[Convergence of \fedstale{}, upper bound]
    \label{thm:fedstale_upper}
    Under Assumptions~\ref{asm:smoothness}--\ref{asm:participation}, if the client and server learning rates, $\eta_c$~and $\eta_s$, are chosen such that $\lr[c] \leq \frac{1}{8LK}$ and $\lr[s] \leq \min \left\{ \frac{N p_{\text{var}}}{12(1-\beta)^2}, \frac{p_{\text{var}}p_{\text{min}}}{3\beta^2 p_{\text{avg}}} \right\}$,
    % \begin{align*}
    %     \textstyle
    %     \lr[c] \leq \frac{1}{8LK}
    %     \quad 
    %     \text{and}
    %     \quad
    %     \lr[s] \leq \min \left\{ \frac{N p_{\text{var}}}{12(1-\beta)^2}, \frac{p_{\text{var}}p_{\text{min}}}{3\beta^2 p_{\text{avg}}} \right\},
    % \end{align*}
    the sequence of 
    \fedstale{} iterates satisfies 
    \begin{align}
        &\min_{t\in\{1,T\}} \E \norm{\nabla F(\bm{w}_{\fedstale{}}^{(t)})}^2 
        ~ \leq ~
        \underbrace{\mathcal{O} \biggl( \frac{F(\bm{w}^{(1)}) - F^*}{\lr[s]\lr[c] K T} \biggr)}_{\text{iterate initialization error}}
        \label{eq:fedstale_upper} \\[-0.12cm]
        &\quad + \underbrace{\mathcal{O} \biggl( \frac{\beta^2 \lr[s]\lr[c] L K H^{(1)}}{p_{\text{var}}p_{\text{min}}T} \biggr)}_{\text{memory initialization error}} 
        + \underbrace{\mathcal{O} \biggl( \left[ \frac{1}{N} + \beta^2 \frac{p_{\text{avg}}}{p_{\text{min}}} \right] \frac{\lr[s]\lr[c] L \sigma^2}{p_{\text{var}}}   \biggr)}_{\text{stochastic gradient error}}
        \notag \\[-0.12cm]
        &\quad + \underbrace{\mathcal{O} \biggl( \left[ \frac{(1-\beta)^2}{N} + \beta^2 \eta_c^2 L^2 K (K-1) \frac{ p_{\text{avg}}}{p_{\text{min}}}  \right] \frac{\lr[s]\lr[c] L K \sigma_g^2}{p_{\text{var}}}\biggr)}_{\text{error from data heterogeneity}}, \notag
    \end{align}
    where $F^* \triangleq \min_{\bm{w}} F(\bm{w})$,
    $H^{(1)} \triangleq \avgcl ||\grad[\cl]{\param[]{1}}{} - \mem[\cl]{1}||^2$,
    $p_{\text{var}} \triangleq ( \avgcl \frac{1-p_i}{p_i})^{-1}$, $p_{\text{avg}} \triangleq \avgcl p_i$, 
    and $p_{\text{min}} \triangleq \min_i p_i$.
\end{theorem}

Theorem~\ref{thm:fedstale_upper} relates \fedstale{}'s convergence to the iterate and memory initial errors, and variances from stochastic gradients ($\sigma^2$) and data heterogeneity ($\sigma_g^2$). It also quantifies the impact of client participation heterogeneity through the terms $p_{\text{var}}$, $p_{\text{avg}}$, and $p_{\text{min}}$.
By scaling the client learning rate as $\mathcal{O}(\frac{1}{\sqrt{T}})$, all error components asymptotically vanish, proving the \emph{unbiasedness} of update~\eqref{eq:fedstale_weight}. 

Theorem~\ref{thm:fedstale_upper} integrates \fedavg{} and \fedvarp{} convergence analyses in a single framework, 
providing new insights on their different behaviors. 
First, for $\beta=0$, the bound provides a convergence result for \fedavg{}.
% matching existing results.

 % \resizebox{\linewidth}{!}{
    % \begin{minipage}{\linewidth}
    % \begin{align*}
    %     &\min_{t\in\{1,T\}} \E \norm{\nabla F(\bm{w}_{\fedstale{}}^{(t)})}^2 
    %     \leq
    %     \underbrace{\mathcal{O} \biggl( \frac{F(\bm{w}^{(1)}) - F^*}{\lr[s]\lr[c] T} + \frac{\beta^2 \lr[s]\lr[c] H^{(1)}}{p_{\text{var}}p_{\text{min}}T} \biggr)}_{\text{initialization error}}
    %     \notag \\
    %     &\textstyle
    %         + \underbrace{\mathcal{O} \biggl( \left[ 1 + \frac{\beta^2p_{\text{avg}}}{p_{\text{min}}} \right] \frac{\lr[s]\lr[c] \sigma^2}{p_{\text{var}}} \biggr)}_{\text{stochastic gradient error}}
    %     + \underbrace{\mathcal{O} \biggl( \left[ (1-\beta)^2 + \frac{\beta^2\eta_c^2p_{\text{avg}}}{p_{\text{min}}} \right] \frac{\lr[s]\lr[c] \sigma_g^2}{p_{\text{var}}} \biggr)}_{\text{error from data heterogeneity}},
    % \end{align*}
    % \end{minipage}
    % }

\begin{corollary}[Convergence of \fedavg{}, upper bound]
\label{crl:fedavg}
    Under same assumptions as Theorem~\ref{thm:fedstale_upper},
    % Assumptions~\ref{asm:smoothness}--\ref{asm:participation}, by setting the learning rates $\lr[c], \lr[s]$ such that $\lr[c] \leq \frac{1}{8LK}$ and $\lr[s] \leq \frac{N p_{\text{var}}}{12}$, 
    the sequence of \fedavg{} iterates satisfies
    % \begin{align}
    %     &\min_{t\in\{1,T\}} \E \norm{\nabla F(\bm{w}_{\fedavg{}}^{(t)})}^2 
    %     \leq \notag \\
    %     &\underbrace{\mathcal{O} \biggl( \frac{F(\bm{w}^{(1)}) - F^*}{\lr[s]\lr[c] T}\biggr)}_{\text{initialization error}}
    %     + \underbrace{\mathcal{O} \biggl( \frac{\lr[s]\lr[c] \sigma^2}{p_{\text{var}}}   \biggr)}_{\text{stochastic gradient error}}
    %     + \underbrace{\mathcal{O} \biggl( \frac{\lr[s]\lr[c] \sigma_g^2}{p_{\text{var}}}  \biggr)}_{\text{data heterogeneity error}},
    % \end{align}
    \begin{align}
        &\min_{t\in\{1,T\}} \E \norm{\nabla F(\bm{w}_{\fedavg{}}^{(t)})}^2 
        \leq \label{eq:fedavg_convergence} \\
        &\underbrace{\mathcal{O} \biggl( \frac{F(\bm{w}^{(1)}) - F^*}{\lr[s] \lr[c] K T}\biggr)}_{\text{iterate initialization error}}
        + \underbrace{\mathcal{O} \biggl( \frac{\lr[s]\lr[c] L \sigma^2}{N p_{\text{var}}}   \biggr)}_{\text{stochastic gradient error}}
        + \underbrace{\mathcal{O} \biggl( \frac{\lr[s]\lr[c] L K \sigma_g^2}{N p_{\text{var}}}  \biggr)}_{\text{error from data heterogeneity}}. \notag
    \end{align}
\end{corollary}

Corollary~\ref{crl:fedavg} shows that client participation heterogeneity only affects \fedavg{} convergence through the variance factor $1/p_{\text{var}}$. This term captures the variability of participation probabilities $\prob$ and is minimized---and equal to $(1-p_{\text{avg}})/p_{\text{avg}}$---when client participation is homogeneous. 
Conversely, this variance term increases with larger participation heterogeneity, and may become the dominant factor in Eq.~\eqref{eq:fedavg_convergence} that slows down \fedavg{} convergence. This justifies our observations for \fedavg{} in Figure~\ref{fig:2}.

Second, for $\beta=1$, Theorem~\ref{thm:fedstale_upper} extends \fedvarp{} known convergence results~\cite[Theorem~2]{jhunjhunwalaFedvarpTacklingVariance2022} to heterogeneous client participation. 
% We introduce a novel term, the ratio $p_{\text{avg}}/p_{\text{min}}$, to explicitly quantify the impact of stale updates on the convergence. 
\begin{corollary}[Convergence of \fedvarp{}, upper bound]
\label{crl:fedvarp}
    Under the same assumptions as in Theorem~\ref{thm:fedstale_upper},
    % Assumptions~\ref{asm:smoothness}--\ref{asm:participation}, with client and server learning rates $\lr[c] \leq \frac{1}{8LK}$ and $\lr[s] \leq \frac{p_{\text{var}}p_{\text{min}}}{3 p_{\text{avg}}}$, 
    \fedvarp{}'s iterates satisfy
    % \begin{align}
    %     \min_{t\in\{1,T\}} \E \norm{\nabla F(\bm{w}_{\fedvarp{}}^{(t)})}^2 
    %     \leq
    %     \underbrace{\mathcal{O} \biggl( \frac{F(\bm{w}^{(1)}) - F^*}{\lr[s]\lr[c] T} + \frac{\lr[s]\lr[c] H^{(1)}}{p_{\text{var}}p_{\text{min}}T} \biggr)}_{\text{initialization error}}
    %     \notag \\
    %     + \underbrace{\mathcal{O} \biggl( \frac{p_{\text{avg}}}{p_{\text{min}}} \frac{\lr[s]\lr[c] \sigma^2}{p_{\text{var}}}  \biggr)}_{\text{stochastic gradient error}}
    %     + \underbrace{\mathcal{O} \biggl( \frac{p_{\text{avg}}}{p_{\text{min}}} \frac{\lr[s]\lr[c]^3 \sigma_g^2}{p_{\text{var}}}\biggr)}_{\text{error from data heterogeneity}},
    % \end{align}
    \begin{align}
       &\min_{t\in\{1,T\}} \E \norm{\nabla F(\bm{w}_{\fedvarp{}}^{(t)})}^2 
        \leq
        \underbrace{\mathcal{O} \biggl( \frac{F(\bm{w}^{(1)}) - F^*}{\lr[s]\lr[c] T} + \frac{\lr[s]\lr[c] H^{(1)}}{p_{\text{var}}p_{\text{min}}T} \biggr)}_{\text{iterate and memory initialization errors}}
        \notag \\
       &~~ + \underbrace{\mathcal{O} \biggl( \frac{\lr[s] \lr[c] L p_{\text{avg}} \sigma^2}{p_{\text{var}}p_{\text{min}}}  \biggr)}_{\text{stochastic gradient error}}
        + \underbrace{\mathcal{O} \biggl( \frac{\lr[s]\lr[c]^3 L^3 K^2 (K-1) p_{\text{avg}} \sigma_g^2}{p_{\text{var}}p_{\text{min}}}\biggr)}_{\text{error from data heterogeneity}}.
    \end{align}
\end{corollary}

We highlight two differences with respect to \fedavg{}. %\fedvarp{} has a dual influence on convergence.
First, 
\fedvarp{} mitigates \emph{data heterogeneity error}: by scaling the learning rate $\lr[c]$ as $\mathcal{O}(T^{-1/2})$, the term in $\sigma_g^2$ decreases as $\mathcal{O}(T^{-3/2})$ versus $\mathcal{O}(T^{-1/2})$ for \fedavg{} in~\eqref{eq:fedavg_convergence}. 
However, \fedvarp{} amplifies the stochastic gradient error through the ratio~$p_{\text{avg}}/p_{\text{min}}$, and this terms may become dominant 
as \emph{client participation} becomes more \emph{heterogeneous}.
% , it amplifies the stochastic gradient error through the ratio~$p_{\text{avg}}/p_{\text{min}}$. 
This drawback, caused from stale updates, was not highlighted by earlier analyses, which considered only \emph{homogeneous} client participation. 

% It is expected to arise---since the variance reduction step in Eq.~\eqref{eq:fedvarp} inherently depends on the least participating client, represented by~$p_{\text{min}}$---but does not affect \fedavg{}, which does not rely on stale updates. 
One may wonder whether the appearance of the factor $1/p_{\text{min}}$ in \fedvarp{} bound may not be just an artifact of our proof technique. The following lower bound for \fedvarp{} and \fedstale{} convergence suggests that this is not the case.
% To verify that the term $1/p_{\text{min}}$ is necessary---not just an artifact of our proof technique---we establish a lower bound on \fedvarp{} and \fedstale{} convergence.
\begin{theorem}[Convergence of \fedstale{}, lower bound]
\label{thm:fedstale_lower}
    Under Assumption~\ref{asm:smoothness}, 
    for any time horizon $T \leq \frac{d-1}{2}$,
    %in a $d$-dimensional space,
    there exist $N$ local objectives $\{\obj[\cl]{}: \mathbb{R}^d \rightarrow \mathbb{R}\}$ for which the iterates of any first-order black-box optimization procedure which leverages both fresh and stale updates satisfy
    \begin{align}
        \min_{t\in\{1,T\}} \E \norm{\nabla F(\bm{w}_{\fedstale{}}^{(t)})}^2
        \geq
        \Omega \left( \frac{F(\param{1}) - F^*}{p_{\text{min}}^3 T^3 + 1} \right)
        \label{eq:fedstale_lower}.
    \end{align}
\end{theorem}
% The rates in Theorems~\ref{thm:fedstale_upper} and~\ref{thm:fedstale_lower} do not match because \texttt{SAG}-based methods like \fedstale{} and \fedvarp{} achieve slower rates than \texttt{SVRG}-based methods, as discussed in Section~\ref{sec:algorithm:comparison}. Moreover, our construction for the lower bound relies on convex functions, while the upper bound holds for general non-convex objectives. Nevertheless, 
Theorem~\ref{thm:fedstale_lower} proves that %any optimization procedure that leverages fresh and stale updates, including 
\fedstale{} for any $\beta >0$, and then \fedvarp{}, requires at least $T \geq \Omega(1/p_{\text{min}})$ rounds to minimize objective~\eqref{eq:g_obj}.

\subsection{Finding the optimal weight $\beta^*$}

\fedstale{} leverages the parameter $\beta$ to balance the multiple sources of variance in Theorem~\ref{thm:fedstale_upper}: stochastic gradients~($\sigma^2$), data heterogeneity ($\sigma_g^2$), and client participation heterogeneity (through the ratio $p_{\text{avg}}/p_{\text{min}}$).

The quadratic dependency on $\beta$ of the bound in Theorem~\ref{thm:fedstale_upper}, Eq.~\eqref{eq:fedstale_upper}, guarantees a unique minimizer $\beta^*\in[0,1]$, generally different from the boundaries values of 0 and 1. 
The optimal $\beta^*$ is:
% \begin{align}
%     \beta^* 
%     \propto
%     \frac{\sigma_g^2}{\frac{p_{\text{avg}}}{p_{\text{min}}}\sigma^2 + \left[ 1 + \lr[c]^2 \frac{p_{\text{avg}}}{p_{\text{min}}} \right] \sigma_g^2}
% \end{align}
\begin{align}
    \beta^* 
    =
    \frac{\sigma_g^2/N}{a_1 \frac{p_{\text{avg}}}{p_{\text{min}}}\frac{\sigma^2}{K} + \left[ \frac{1}{N} + a_2 \frac{p_{\text{avg}}}{p_{\text{min}}} \lr[c]^2 L^2 K (K-1)\right] \sigma_g^2},
    \label{eq:beta_opt}
\end{align}
where $a_1$ and $a_2$ are positive constants.

In practice, computing $\beta^*$ is challenging due to the unknowns $L$, $\sigma^2$, and $\sigma_g^2$ in Eqs.~\eqref{eq:fedstale_upper} and~\eqref{eq:beta_opt}, which are difficult to estimate since they depend on the client objectives and on the specific heterogeneity setting. 
Moreover, Eq.~\eqref{eq:fedstale_upper} provides a worst-case upper bound for the gradient norm, but convergence may be significantly faster. For instance, the bound becomes vacuous as $p_{\text{min}}$ approaches zero, yet, if all clients share the same local objective, convergence is unaffected by non-participating clients.
% the tightness of the upper bound in Eq.~\eqref{eq:fedstale_upper} 
% the upper bound in~\eqref{eq:fedstale_upper} is an approximation for the unknown in its left hand side,
% is debatable---for instance, it diverges when $p_{\text{min}}=0$, while the lower bound in Eq.~\eqref{eq:fedstale_lower} does not.
Therefore, we primarily use Eq.~\eqref{eq:beta_opt} to derive \emph{qualitative}, 
%analysis and derive broader, 
yet important guidelines.

The monotonically increasing behavior of $\beta^*$ with $\sigma_g^2$ in Eq.~\eqref{eq:beta_opt} suggests
\ul{\emph{Guideline A: Increase the weight to stale updates, $\beta$, when data heterogeneity, $\sigma_g^2$, increases.}} 

Guideline A is in line with our previous comparison of Corollary~\ref{crl:fedvarp} and Corollary~\ref{crl:fedavg}.
As we observed, stale updates become more beneficial when data heterogeneity ($\sigma_g^2$) is dominant.
%, in line with Corollary~\ref{crl:fedvarp}. 
Conversely, as data heterogeneity decreases, the benefit from stale updates diminishes. This outcome is intuitive: in the extreme case where all clients share same datasets, each local objective aligns with the global objective.
%, eliminating any objective inconsistency in Eq.~\eqref{eq:obj_inconsitency}: theoretically, 
Relying solely on updates from participating clients is then optimal, as stale updates may only introduce unnecessary noise. 

The monotonically decreasing behavior of $\beta^*$ with the ratio $p_{\text{avg}}/p_{\text{min}}$ in Eq.~\eqref{eq:beta_opt} informs
\ul{\emph{Guideline B: Decrease the weight to stale updates, $\beta$, as client participation heterogeneity, $p_{\text{avg}}/p_{\text{min}}$, increases.}}

Also Guideline B is grounded in intuition: as client participation is more \emph{heterogeneous} ($p_{\text{min}} \ll p_{\text{avg}}$), the least participating clients refresh their stale update less frequently, leading to more \emph{outdated} global updates: leveraging them may yield poor results.
Conversely, when client participation is \emph{homogeneous} ($p_{\text{min}} \approx p_{\text{avg}}$), all clients \emph{uniformly} refresh their update, and global variance reduction methods perform best.

\section{Experimental Results}
\label{sec:experiments}

We evaluate the performance of \fedstale{} in experiments. 
The source code of our experimental framework is in the  supplementary material and will made publicly available after publication.
% The source code of our framework, including the first public implementation of \fedvarp{}, is available in the supplementary material.

\subsection{Experimental setup}

\hspace{1ex}
\textbf{System, Datasets, and Models.}
We simulate a FL system with $N=24$ clients. We consider two image classification tasks: handwritten digits recognition on MNIST~\cite{lecun-mnisthandwrittendigit-2010} and natural image classification on CIFAR-10~\citep{krizhevsky2009learning}. Each dataset has 10 classes, or labels. We train two convolutional neural network (CNN) models with slightly different architectures. These models, with cross-entropy loss, define non-convex objectives~\eqref{eq:g_obj}. %, conform to our assumptions.

\textbf{Participation heterogeneity.} 
Client participation follows a Bernoulli distribution, in line with Assumption~\ref{asm:participation}.
To simulate heterogeneity in client participation, we randomly divide clients into two groups based on their participation dynamics: one group of clients always participate, while the other, less participating group, have participation probabilities $p_{\text{min}}$ varying in the range $\{50,20,10,5,2,1,0.5,0.2\}\%$. 
The ratio $p_{\text{avg}}/p_{\text{min}}$ specifies the degree of client participation heterogeneity.

\textbf{Data heterogeneity.}
Following existing work~\cite{sattlerClusteredFederatedLearning2021}, we simulate data heterogeneity across clients' local datasets by: 1) randomly partitioning the dataset among clients; 2)~swapping a fraction $\hat{\sigma}_g^2$ of two labels in  the second group, with $\hat{\sigma}_g^2 \in\{0.0, 0.2, 0.4, 0.6, 0.8, 1.0\}$. 
The empirical parameter $\hat{\sigma}_g^2$ mirrors the theoretical variance $\sigma_g^2$ in Assumption~\ref{asm:heterogeneity}, measuring the degree of data heterogeneity: $\hat{\sigma}_g^2=0$ represents homogeneous (IID) data distributions and $\hat{\sigma}_g^2=1$ indicates maximum %(non-IID) 
heterogeneity among client datasets.

\textbf{Baselines.}
We compare \fedavg{} ($\beta=0$), \fedvarp{} ($\beta=1$), and \fedstale{} (for $\beta\in\{0.2,0.5,0.8\}$) across diverse heterogeneity settings.
Previous work~\cite{jhunjhunwalaFedvarpTacklingVariance2022} showed that, under partial client participation, \fedvarp{} consistently outperformed both \texttt{MIFA}~\cite{guFastFederatedLearning2021}, due to its biased variance correction, and \texttt{SCAFFOLD}~\cite{karimireddySCAFFOLDStochasticControlled2020}, that also incurs higher communication costs.
% The experiments involve $T=4000$ communication rounds, with each client performing $K=5$ local iterations per round.
We benchmark all algorithms over a consistent time horizon, corresponding, on average, to the first ten participation instances by the least participating client.
Clients perform $K=5$ local iterations. 
% To ensure fair evaluations, we benchmark all algorithms over a consistent time horizon, corresponding, on average, to the first ten participations of the least participating client.
We use a batch size of 128 in all experiments. 
For all algorithms, we fix the server learning rate $\lr[s]$ to 1 and tune the client learning rate $\lr[c]$ over the grid $\{10^{-2},10^{-2.5},10^{-3},10^{-3.5},10^{-4}\}$. 
While we initially assume all algorithms have exact knowledge of client participation probabilities, we relax this assumption in Section~\ref{sec:experiments:online}.
We average results over three random seeds.

% \paragraph{Runtime.} Experiments required 1600 hours and 400 kWh, approximately 35\% of a Paris-London flight's per-passenger emissions~\citep{lannelongueGreenAlgorithmsQuantifying2021}.

% tuned $\lr[]$ via grid search, and 
 
% Experiments reported 1600 runtime hours, 400 kWh energy consumption, and 20 kg $\mathrm{CO_2e}$ emissions---around 45\% of a Paris-London flight's carbon footprint.\footnote{Estimated by \url{https://calculator.green-algorithms.org/}.}

\subsection{Existence of different regimes}

\begin{figure}[t]
    \centering
    \includegraphics[width=\linewidth]{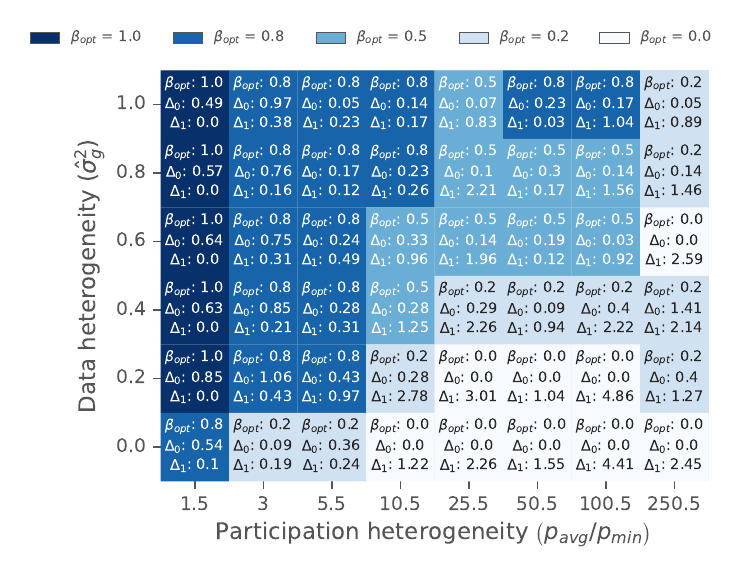}
    \vspace{-0.3cm}
    \caption{\footnotesize $\beta_{\text{opt}}$ values for FedAvg ($\beta$=0), FedVARP ($\beta$=1), and FedStale ($\beta$$\in$\{0.2, 0.5, 0.8\}) across 48 heterogeneity settings on the MNIST dataset. 
    Color gradients range from lighter shades ($\beta_{\text{opt}}$=0) to darker shades ($\beta_{\text{opt}}$=1).} 
    \label{fig:optimal_beta}
\vspace{0.8cm}
\end{figure}

\begin{figure*}
\centering
\begin{minipage}{.49\textwidth}
\centering
\begin{subfigure}[t]{0.495\textwidth}
    \centering
    \includegraphics[width=\textwidth]{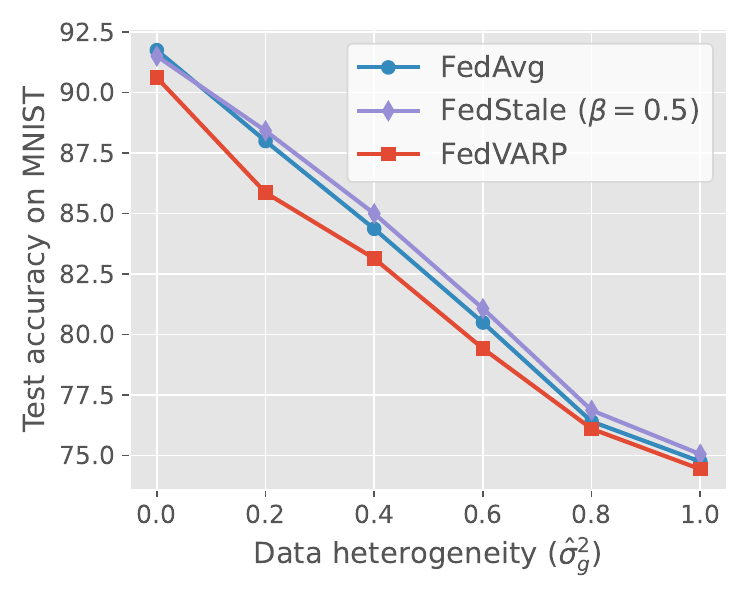}
    \vspace{-0.5cm}
    \caption{\footnotesize MNIST dataset} 
    \label{fig:data_heterogeneity:mnist}
\end{subfigure}
\hfill
\begin{subfigure}[t]{0.495\textwidth}
    \centering
    \includegraphics[width=\textwidth]{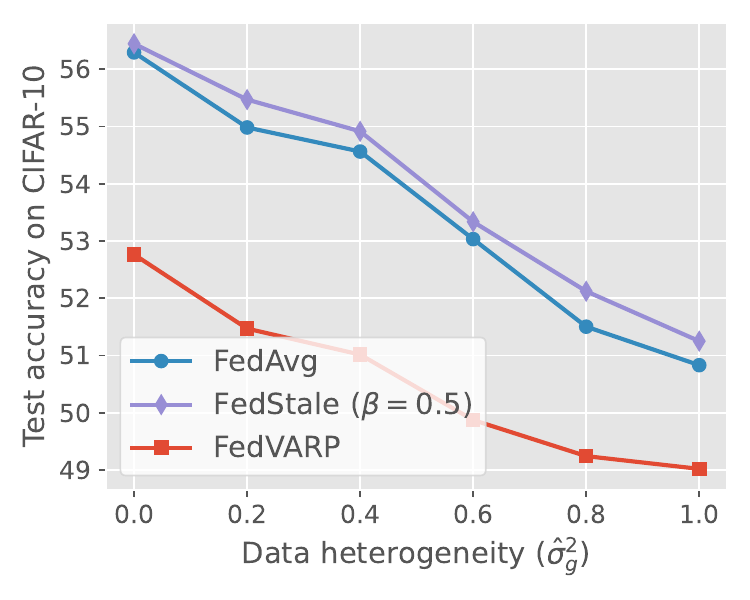}
    \vspace{-0.5cm}
    \caption{\footnotesize CIFAR-10 dataset} 
    \vspace{0.4cm}
    \label{fig:data_heterogeneity:cifar}
\end{subfigure}
\caption{
\footnotesize
Test accuracy of FedAvg ($\beta$=0), FedVARP ($\beta$=1), and FedStale ($\beta$=0.5) varying data heterogeneity at fixed participation ratio $p_{\text{avg}}/p_{\text{min}}=10$.
}
\label{fig:data_heterogeneity}
\end{minipage}%
\hfill
\begin{minipage}{.49\textwidth}
\centering
\begin{subfigure}[t]{0.495\textwidth}
    \centering
    \includegraphics[width=\textwidth]{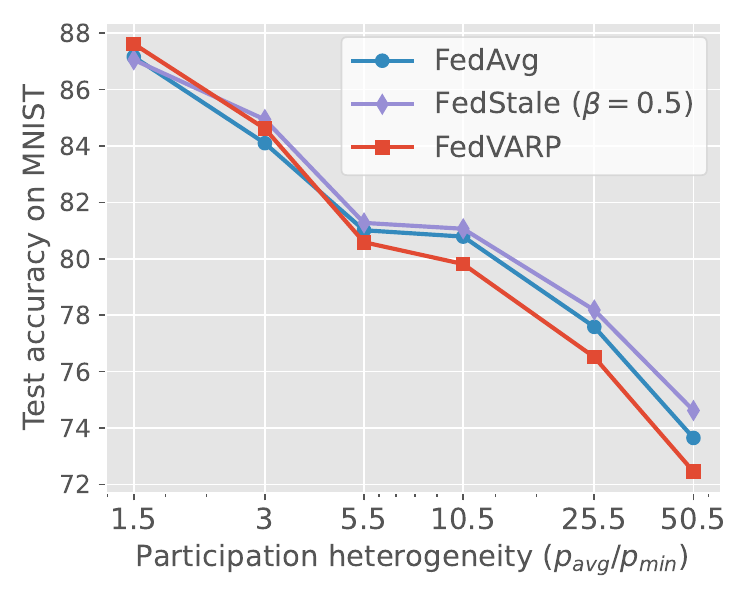}
    \vspace{-0.5cm}
    \caption{\footnotesize MNIST dataset} 
    \label{fig:participation_heterogeneity:mnist}
\end{subfigure}
\hfill
\begin{subfigure}[t]{0.495\textwidth}
    \centering
    \includegraphics[width=\textwidth]{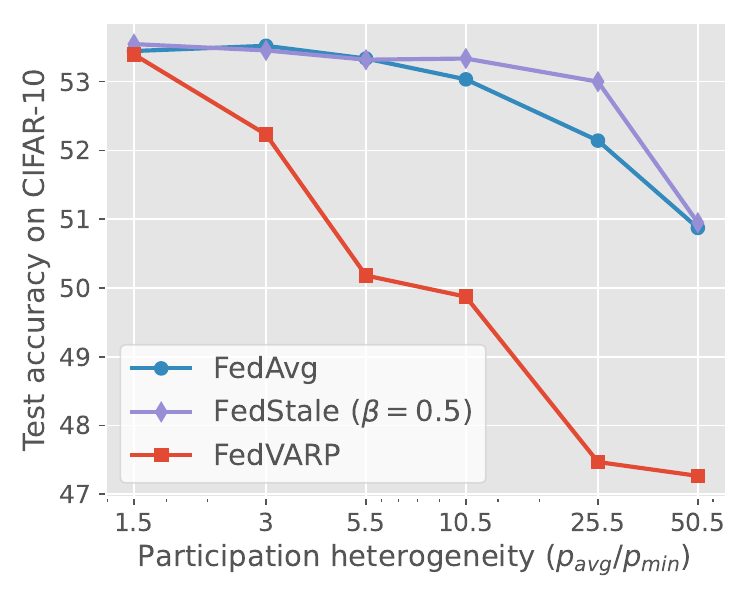}
    \vspace{-0.5cm}
    \caption{\footnotesize CIFAR-10 dataset} 
    \vspace{0.4cm}
    \label{fig:participation_heterogeneity:cifar}
\end{subfigure}
\caption{
\footnotesize
Test accuracy of FedAvg ($\beta$=0), FedVARP ($\beta$=1), and FedStale ($\beta$=0.5) varying client participation ratio at fixed data heterogeneity $\hat{\sigma}_g^2=0.6$.
}
\label{fig:participation_heterogeneity}
\end{minipage}%
\vspace{0.4cm}
\end{figure*}

% Moreover, Figures~\ref{fig:data_heterogeneity} and~\ref{fig:participation_heterogeneity} show how varying one heterogeneity source influences the test accuracies of \fedavg{} ($\beta=0$), \fedvarp{} ($\beta=1$), and \fedstale{} ($\beta=0.5$) over $T=4000$ communication rounds. 
% Specifically, Figure~\ref{fig:data_heterogeneity} fixes the participation heterogeneity at $p_{\text{avg}}/p_{\text{min}}=10$ and varies the data heterogeneity ($\hat{\sigma}_g^2$), while Figure~\ref{fig:participation_heterogeneity} holds the data heterogeneity constant at $\hat{\sigma}_g^2=0.6$ and varies the participation heterogeneity ($p_{\text{avg}}/p_{\text{min}}$).

In Figure~\ref{fig:optimal_beta}, we show the empirical values of $\beta$ that yield the highest test accuracies among \fedavg{} ($\beta=0$), \fedvarp{} ($\beta=1$), and \fedstale{} ($\beta \in \{0.2, 0.5, 0.8\}$) across diverse heterogeneity settings on the MNIST dataset. We denote these values as $\beta_{\text{opt}}$. 

The heatmap shows how $\beta_{\text{opt}}$ varies with client participation heterogeneity ($p_{\text{avg}}/p_{\text{min}}$, in the x-axis) and data heterogeneity ($\hat{\sigma}_g^2$, in the y-axis). 
% The heatmap shows how $\beta_{\text{opt}}$ varies depending on two types of heterogeneity: client participation heterogeneity (represented on the x-axis by the ratio $p_{\text{avg}}/p_{\text{min}}$) and data heterogeneity (shown on the y-axis by $\hat{\sigma}_g^2$).
Moreover, each cell reports the performance gains of the best setting for \fedstale{}. $\Delta_0$ and~$\Delta_1$ denote, respectively, the accuracy improvements of \fedstale{}($\beta_{\text{opt}}$) over \fedavg{} ($\beta=0$) and \fedvarp{} ($\beta=1$).
% Each cell in the heatmap also includes performance gains, represented as $\Delta0$ and $\Delta1$. These values measure the accuracy improvements of the algorithm with the highest performance over \fedavg{} ($\beta=0$) and \fedvarp{} ($\beta=1$), respectively.
This visualization aggregates results from 720 training runs, across 8 participation heterogeneity setups and 6 data heterogeneity setups, each comparing 5 algorithms for 3 independent seeds.

\textbf{Multiple regimes in heterogeneity settings.}
No single algorithm consistently outperforms others across all settings. Instead, Figure~\ref{fig:optimal_beta} shows different zones where the best-performing algorithm depends on the interplay between data heterogeneity ($\hat{\sigma}_g^2$) and client participation heterogeneity ($p_{\text{avg}}/p_{\text{min}}$).
% A clear best algorithm, that outperforms competitors in every setup, does not exists. On the contrary, Figure~\ref{fig:optimal_beta} indicates the existence of multiple zones in which, depending on the interplay between $\hat{\sigma}_g^2$ and the ratio $p_{\text{avg}}/p_{\text{min}}$, one among \fedavg{} ($\beta=0$), \fedvarp{} ($\beta=1$), and \fedstale{} ($\beta \in \{0.2, 0.5, 0.8\}$) outperforms the others.
% We note that achieving perfect monotonicity in Figure~\ref{fig:optimal_beta} is challenging due to inherent non-linearities in model learning behaviors and the granularity of the learning rate discrete search space. 
The observed trends reflect our qualitative guidelines.

Specifically, Figure~\ref{fig:optimal_beta} identifies three distinct zones where specific patterns in performance emerge:
\emph{i)}~\fedvarp{} yields the best results for large data heterogeneity ($\hat{\sigma}_g^2 \geq 0.2$) and homogeneous client participation ($p_{\text{min}} \approx p_{\text{avg}}$), favoring larger weights to stale updates ($\beta_{\text{opt}}=1$);
\emph{ii)}~conversely, \fedavg{} best fits settings with low data heterogeneity ($\hat{\sigma}_g^2 \leq 0.2$) and large participation heterogeneity ($p_{\text{avg}} \geq 25 p_{\text{min}}$), where using stale updates overall reduces performance;
\emph{iii)}~finally, a significant transitional zone exists where moderate heterogeneity levels ($3 p_{\text{min}} \leq p_{\text{avg}} \leq 25 p_{\text{min}}$) favor intermediate $\beta_{\text{opt}}$ values ($\beta_{\text{opt}} \in \{0.2, 0.5, 0.8\}$), which yield the best performance.

Overall, \fedstale{} prevails in 72\% of scenarios within our $6 \times 8$ grid, against \fedvarp{}, 18\%, and \fedavg{}, 10\%.
Therefore, \fedstale{} plays a key role---we believe---in bridging the gaps posed by \fedavg{} and \fedvarp{} in real-world federated settings, which often exhibit intermediate levels of client data and participation heterogeneity.

% especially when $3 \leq p_{\text{avg}}/p_{\text{min}}$, a common scenario .

% there exists a significant transitional phase, for intermediate values of client data and participation heterogeneity, where $\beta_{\text{opt}}$ is neither 0 nor 1, but rather assumes intermediate values. We observe that this transition phase is actually larger than the either 0 or 1 phase: in our $6\times8$ grid, we observe that \fedvarp{} ($\beta=1$) outperforms in the 10\% of the grid, \fedavg{} ($\beta=0$) outperforms in the 18\% of the grid, while \fedstale{} ($\beta \in \{0.2, 0.5, 0.8\}$) outperforms in the remaining 72\% of the space. 
% These intermediate heterogeneity settings, $3 \leq p_{\text{avg}}/p_{\text{min}}$, are realistic in typical federated learning settings, therefore highlighting the importance of \fedstale{} to complement 

\textbf{Effect of data heterogeneity.}
Figure~\ref{fig:optimal_beta} shows that $\beta_{\text{opt}}$ increases with data heterogeneity, in line with Guideline A.
Figure~\ref{fig:data_heterogeneity} explores this trend in more detail, by holding participation heterogeneity constant at $p_{\text{avg}}/p_{\text{min}}=10$ and varying data heterogeneity ($\hat{\sigma}_g^2$).
For all algorithms, increased data heterogeneity corresponds to lower test accuracies.
In Figures~\ref{fig:data_heterogeneity:mnist} and~\ref{fig:data_heterogeneity:cifar}, \fedstale{} ($\beta=0.5$), without particular fine-tuning,
consistently outperforms \fedvarp{} in settings of moderate participation heterogeneity and improves over \fedavg{} as client data become heterogeneous (already at $\hat{\sigma}_g^2 \geq 0.2$).
Moreover, Figure~\ref{fig:data_heterogeneity:cifar} shows that \fedvarp{}, despite its overall lower accuracy, proves to perform better in extremely heterogeneous data scenarios (when $\hat{\sigma}_g^2 \geq 0.8$).

\textbf{Effect of participation heterogeneity.}
Figure~\ref{fig:optimal_beta} shows that $\beta_{\text{opt}}$ decreases as the participation heterogeneity ($p_{\text{avg}}/p_{\text{min}}$) increases, in line with Guideline B.
Figure~\ref{fig:participation_heterogeneity} details this dynamic by fixing data heterogeneity at $\hat{\sigma}_g^2=0.6$ and only varying participation heterogeneity.
In both Figures~\ref{fig:participation_heterogeneity:mnist} and~\ref{fig:participation_heterogeneity:cifar},
it is evident how \fedvarp{} performs well when client participation is homogeneous ($p_{\text{min}} \approx p_{\text{avg}}$), yet struggles with increasing participation heterogeneity.
\fedavg{} exhibits dual behavior, which confirms that the usefulness of stale updates progressively diminishes as participation heterogeneity increases (already at $p_{\text{avg}} \geq 3 p_{\text{min}}$).
Figure~\ref{fig:participation_heterogeneity:cifar} also shows that \fedstale{} ($\beta=0.5$), without specific tuning, maintains robust performance across a wide range of participation levels (until $p_{\text{avg}} \approx 25 p_{\text{min}}$), and only drops accuracy at $p_{\text{avg}} \approx 50 p_{\text{min}}$.

\subsection{Online estimation of participation probabilities}
\label{sec:experiments:online}

% Implementing \fedstale{} presents a practical challenge in dealing with \emph{unknown client participation probabilities}. This issue, common to all algorithms that handle heterogeneous client participation, can lead to objective inconsistencies if not properly managed. 
% Prior research has explored the impacts of estimation errors on convergence~\cite{wangLightweightMethodTackling2024},
% recommending mitigation strategies like \emph{i)}~pre-training collection of participation statistics (e.g., through computing and network measurements)~\citep{rodioFederatedLearningHeterogeneous, riberoFederatedLearningIntermittent2023} or \emph{ii)}~leveraging historical participation patterns for online estimation~\cite{wangLightweightMethodTackling2024}. 
% We show in Section~\ref{sec:experiments:online} how \fedstale{} integrates well with these strategies.
% ---a common limitation in current theoretical analyses under heterogeneous client participation---

We evaluate \fedstale{} with online estimation of client participation probabilities, to simulate scenarios where these probabilities are unknown before training
~\citep{riberoFederatedLearningIntermittent2023, rodioFederatedLearningHeterogeneous, wangLightweightMethodTackling2024}.
To this purpose, 
we integrate \fedstale{} with \texttt{FedAU}~\cite{wangLightweightMethodTackling2024}, 
a state-of-the-art algorithm for tracking client participation dynamics, that balances bias and variance in the estimation through a cutoff mechanism.

Figure~\ref{fig:estimate} shows that the integration of \fedstale{} with \texttt{FedAU}'s estimation technique still aligns with our guidelines.
Moreover, \fedvarp{} performs significantly worse than \fedstale{}($\beta_{\text{opt}}$) when client participation probabilities are estimated ($\Delta_1$ values in Fig.~\ref{fig:estimate}). Also, we observe overall lower $\beta_{\text{opt}}$ values in this scenario. These trends suggest that methods leveraging stale updates, like~\fedvarp{}, might be particularly sensitive to inaccurate $\prob$ estimations.
% : underestimating $\prob$ may disproportionately weight stale updates, causing larger deviation from the optimal learning trajectory. Further investigation is required to confirm these insights.

\begin{figure}[ht]
    \centering
    \vspace{-0.3cm}
    \includegraphics[width=0.85\columnwidth]{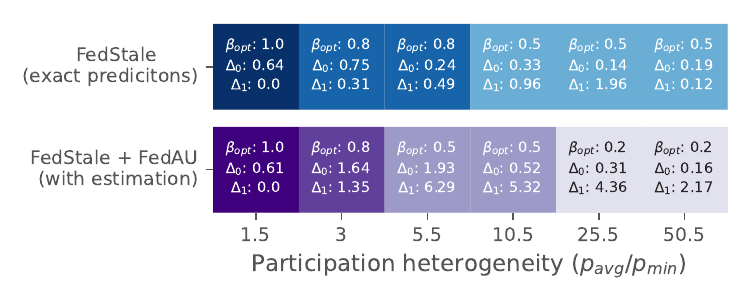}
    \caption{\footnotesize ``Exact'' vs. ``Estimated'' participation probabilities, $\hat{\sigma}_g^2=0.6$.}
    \label{fig:estimate}
\end{figure}

% \begin{figure}[ht]
% \centering
% \begin{subfigure}[t]{\columnwidth}
%     \centering
%     \includegraphics[width=\columnwidth, height=1.5cm]{example-image}
%     % \vspace{0.1cm}
%     \caption{\footnotesize Known participation probabilities (FedStale)} 
%     \label{fig:exact}
%     \vspace{0.5cm}
% \end{subfigure}
% \begin{subfigure}{\columnwidth}
%     \centering
%     \includegraphics[width=\columnwidth, height=1.5cm]{example-image}
%     \caption{\footnotesize Estimated participation probabilities (FedStale + FedAU)} 
%     \label{fig:estimate}
%     \vspace{0.5cm}
% \end{subfigure}
% \caption{
% \footnotesize
% Figure
% }
% \vspace{0.4cm}
% \label{fig:exact_vs_estimate}
% \end{figure}

\section{Conclusion}
\label{sec:conclusion}

This paper addresses global variance reduction in federated learning beyond the common assumption of homogeneous client participation. 
Unlike prior work, our research explores not only the advantages but also the challenges of leveraging stale client updates across varying heterogeneity scenarios. 
Our algorithm, \fedstale{}, is equipped with guidelines: practitioners can decide whether storing stale updates is worthwhile or if solely relying on participating client updates is more efficient.
Exploring this tradeoff paves the way---we believe---for developing federated learning algorithms more attuned to the varied dynamics of client data and participation heterogeneity.

\appendix
\section{Appendix}

\subsection{Proof sketch, Theorem~\ref{thm:fedstale_upper}}

Previous work analyzed convergence of FL algorithms in various settings. Closest to our setting are~\citet{wangTacklingObjectiveInconsistency2020} and~\citet{jhunjhunwalaFedvarpTacklingVariance2022}, that analyzed \fedavg{} and \fedvarp{}, respectively, under non-iid data and partial yet \emph{homogeneous} client participation.

Our analysis in Theorem~\ref{thm:fedstale_upper} builds on~\citep{jhunjhunwalaFedvarpTacklingVariance2022} and relies on a similar Lyapunov optimization function as in~\citep[Appendix C.2, Eq.~(20)]{jhunjhunwalaFedvarpTacklingVariance2022}:
\begin{align}
    \lyap{t} 
    \triangleq
    \obj[]{t} 
    + 
    \delta \norm{\gps{t}_{\fedstale{}}}^2 
    + 
    \gamma H^{(t)},
    \quad
    \delta, \gamma >0,
    \label{eq:lyapunov}
\end{align}
where $H^{(t)} \triangleq \avgcl \norm{\grad[\cl]{\param[]{t}}{} - \mem[\cl]{t}}^2$ quantifies the deviation of stale client updates from the ``true'' local gradients.

This section provides the proof outline for Theorem~\ref{thm:fedstale_upper}, focusing mostly on the novel contributions of our analysis:
\begin{itemize}[noitemsep,topsep=0pt]
    \item We apply the standard descent lemma~\cite{wangTacklingObjectiveInconsistency2020}, for smooth and non-convex objectives, to Eq.~\eqref{eq:lyapunov} [Supplementary, Lemma 1];
    \item Under Assumptions~\ref{asm:smoothness}--\ref{asm:participation}, the variance of the global update $\Delta^{(t)}_{\fedstale{}}$ is bounded by variances from stochastic gradients ($\sigma^2$) and data heterogeneity ($\sigma_g^2$), by the square norm of the previous global update $\Delta^{(t-1)}_{\fedstale{}}$, and by the memory error $H^{(t)}$. 
    The last two terms, emerging from stale updates for non-participating clients, are multiplied by $\beta^2$ and contribute to both optimization and error. Additionally, the client participation variance $1/p_{\text{var}}$, consequence of the Bernoulli assumption on the client participation (Assumption~\ref{asm:participation}), equally impacts all these terms.
    More details are provided in [Supplementary, Lemmas~6 and~8];
    \item Under Assumptions~\ref{asm:smoothness}--\ref{asm:participation}, the error from stale updates ($H^{(t)}$) is also bounded by the variances $\sigma^2$ and $\sigma_g^2$, the square norm of the previous global update $\Delta^{(t-1)}_{\fedstale{}}$, and the previous memory error $H^{(t-1)}$. 
    This error, under Assumption~\ref{asm:participation}, depends on the participation probability of the least participating client ($p_{\text{min}}$), is consistently weighted by $\beta^2$, and does not affect \fedavg{}.
    % This happens because the global error $H^{(t)}$ from leveraging stale updates intrinsically depends on the least participating client---for instance, if a client appears once in the beginning of the training and then never participates again, \fedvarp{} will always keep considering the stale, noisy, update of this client even near optimally is reached. Reasonably, this variance is weighted by $\beta$, and does not concern \fedavg{}.
    More details in [Supplementary, Lemmas~7 and~9];
    \item Through the Lyapunov recursion, the dependency on $p_{\text{min}}$ remains consistent across all $\beta^2$-weighted terms. This suggests that the influence of $p_{\text{min}}$ stems from stale updates and can be balanced by controlling $\beta$.
    More details in [Supplementary, Theorem 1].
\end{itemize}

\vspace{-0.3cm}
\subsection{Proof sketch, Theorem~\ref{thm:fedstale_lower}}

Our proof builds upon~\citet{nesterovIntroductoryLecturesConvex2004} and~\citet{bubeckConvexOptimizationAlgorithms2015}, who established lower bounds in \emph{centralized} settings,
and~\citet{scamanOptimalConvergenceRates2019}, for general \emph{decentralized} setting.
We adapt the analysis to non-convex federated settings with heterogeneous client participation:
\begin{itemize}[noitemsep,topsep=0pt]
    \item We split Nesterov's function for \emph{centralized}  optimization~\citep{nesterovIntroductoryLecturesConvex2004, bubeckConvexOptimizationAlgorithms2015} between the most and least participating clients ($p_{\text{max}}$ and $p_{\text{min}}$);
    \item Most dimensions of the parameters $\param[\fedstale{}]{t}$ remains zero, and (fresh or stale) client updates only increase non-zero dimensions once every $1/p_{\text{min}}$ steps on average;
    \item We standardize the lower bound measure to squared gradient norms for direct comparison with non-convex counterparts (Theorem~\ref{thm:fedstale_upper}), in expectation over the randomness in client participation.
        
    % prove an upper-bound relating the number of non-zero components revealed by $\param[\fedstale]{t} \in \R^d$ with the time step $t$ and least participating client~$p_{\text{min}}$;
    % \item We extend results in~\citet{nesterovIntroductoryLecturesConvex2004, bubeckConvexOptimizationAlgorithms2015}, which consider optimality gap in terms of objectives, to the expected squared gradient norm, where expectation is taken with respect to the randomness in client participation.
\end{itemize}

% \begin{ack}
% This research was supported by the French government through the 3IA Côte d’Azur Investments in the Future project by the National Research Agency (ANR) with reference ANR-19-P3IA-0002, and by Groupe La Poste in the framework of FedMalin Inria Challenge.
% \end{ack}

\clearpage
\bibliography{bibliography}

\newpage
\onecolumn

\setlength{\parindent}{0pt}
\setlength{\parskip}{0.1in}

% \clearpage
\title{\LARGE Supplementary Material. FedStale: leveraging stale client updates in federated learning}
\author{}
{\maketitle}
\begingroup
\normalsize
\section{FedStale, Upper bound}
\subsection{Preliminaries}

In this section, we provide an overview of the \texttt{FedStale} algorithm and establish the necessary notation used throughout this supplementary material.

\subsubsection*{Algorithm Description}

The algorithm's structure, as outlined in Algorithm~\ref{alg:algorithm}, is detailed below:

\begin{algorithm}
    \For{each round $t=1,\dots,T$}{
        \For{all clients $\cl=1,\dots,N$, in parallel}{
            Initialize $\param[\cl]{t,0} = \param[]{t}$ \\
            \For{local iterations $k=0,\dots,K-1$}{
            Sample data $\batch[\cl]{t,k} \stackrel{\text{iid}}{\sim} \mathcal{D}_{\cl}$ \\
            Compute stochastic gradient $\grad[\cl]{\param[\cl]{t,k}}{\batch[\cl]{t,k}}$ \\
            Update $\param[\cl]{t,k+1} = \param[\cl]{t,k} - \lr[c] \grad[\cl]{\param[\cl]{t,k}}{\batch[\cl]{t,k}}$
            } 
            Compute and return $\pg[\cl]{t} = \frac{1}{\lr[c] K} (\param[\cl]{t,0} - \param[\cl]{t,E})$ to the server
            }
        Aggregate client updates $\Delta^{(t)} = \avgcl \beta \mem[\cl]{t} + \avgcl \frac{\ber[i]}{p_i} ( \pg[\cl]{t} - \beta \mem[\cl]{t} )$ \\
        Update global model $\param[]{t+1} = \param[]{t} - \lr[] \Delta^{(t)}_{\texttt{FedStale}}, ~\lr[]=\lr[s]\lr[c] K$ \\
        At the server level, update memory $\forall \cl, \mem[\cl]{t+1} = 
        \begin{dcases}
            \pg[\cl]{t} & \text{if } \ber[\cl]=1\\
            \mem[\cl]{t} & \text{otherwise}
        \end{dcases}$
        }
    \caption{\texttt{FedStale}($\beta$)}
    \label{alg:algorithm}
\end{algorithm}

We detail some modifications from the main text, introduced to streamline notation for this proof:
\begin{enumerate}
    \item In Algorithm~\ref{alg:algorithm}, we assume that \emph{all} clients partake in the local optimization step and compute $\pg[\cl]{t}$. However, the server aggregates only the model updates from participating clients (where $\ber[\cl] = 1$). This assumption simplifies notation and is equivalent to a scenario where only participating clients return their model updates to the server.
    \item To condense notation, we normalize the client update $\Delta_{\cl}^{(t)}$ by the client learning rate $\lr[c]$ and the number of local iterations $K$. This results in rescaling the client update $\Delta_{\cl}^{(t)} = (\param[\cl]{t,0} - \param[\cl]{t,E})$ by $\lr[c]$ and $K$. 
    The server update step is then rewritten as $\param[]{t+1} = \param[]{t} - \lr[] \Delta^{(t)}$, where $\lr[]=\lr[s]\lr[c] K$ represents an ``equivalent'' learning rate at the server level. 
    \item We explicitly write the participation indicator function $\ber[\cl]$ in the server update $\Delta^{(t)}$. This formulation not only brings transparency to the notation but also allows for a more clear understanding of the \texttt{FedAvg}, \texttt{FedVARP}, and \texttt{FedStale} updates:
    \begin{align}
        \Delta_{\texttt{FedAvg}}^{(t)} &= \avgcl \frac{\ber[i]}{p_i} \pg[\cl]{t} 
        \label{eq:app:fedavg} \\
        \Delta_{\texttt{FedVARP}}^{(t)} &= \avgcl \mem[\cl]{t} + \avgcl \frac{\ber[i]}{p_i} \left( \pg[\cl]{t} - \mem[\cl]{t} \right) 
        \label{eq:app:fedvarp} \\
        \Delta^{(t)}_{\texttt{FedStale}}
        &= (1-\beta) \Delta_{\texttt{FedAvg}}^{(t)} + \beta \Delta_{\texttt{FedVARP}}^{(t)}
        \label{eq:fedstale1} \\
        &=\avgcl \beta \mem[\cl]{t} + \avgcl \frac{\ber[i]}{p_i} \left( \pg[\cl]{t} - \beta \mem[\cl]{t} \right) 
        \label{eq:fedstale2} \\
        &= \avgcl \frac{\ber[i]}{p_i} \pg[\cl]{t} - \frac{\beta}{N} \sumcl \left(\frac{\ber[i]}{p_i} - 1 \right) \mem[\cl]{t}
        \label{eq:fedstale3}
    \end{align}
\end{enumerate}

The comparison of Equations~\eqref{eq:app:fedavg}--\eqref{eq:fedstale3} allows for the following considerations:
\begin{enumerate}
    \item \texttt{FedVARP}'s update (Eq.~\eqref{eq:app:fedvarp}) recovers \texttt{FedAvg}'s update (Eq.~\eqref{eq:app:fedavg}) when:
    \begin{enumerate}
        \item All clients participate in the current round ($\ber[\cl]=1, \forall \cl$), or
        \item All memory terms are set to zero ($\mem[\cl]{t} = 0, \forall \cl$).
    \end{enumerate}
    \item \texttt{FedStale}'s update can be rewritten in three different forms (Equations~\eqref{eq:fedstale1}--\eqref{eq:fedstale3}), each offering a unique perspective:
    \begin{enumerate}
        \item Eq.~\eqref{eq:fedstale1} interprets \texttt{FedStale}'s update as a convex combination of \texttt{FedAvg}'s update (Eq.~\eqref{eq:app:fedavg}) and \texttt{FedVARP}'s update (Eq.~\eqref{eq:app:fedvarp}), where $\beta$ is the parameter of the convex combination;
        \item Eq.~\eqref{eq:fedstale2} relates \texttt{FedStale}'s update to \texttt{FedVARP}  (Eq.~\eqref{eq:app:fedvarp}), where $\beta$ acts as a weight for the memory terms $\{\mem[\cl]{t}, \forall\cl\}$;
        \item Eq.~\eqref{eq:fedstale3} frames \texttt{FedStale}'s in relation to \texttt{FedAvg}(Eq.~\eqref{eq:app:fedavg}), introducing the memory term $\mem[\cl]{t}$ whenever client $\cl$ does not participate and subtracting cumulative memory terms $\mem[\cl]{t}/\prob$ when client $\cl$ does participate again.
    \end{enumerate}
\end{enumerate}

The normalized client update $\pg[\cl]{t}$ is the average of $K$ local stochastic gradients computed by client $\cl$ during round $t$.\\ We denote it as \emph{local stochastic pseudo-gradient}:
\begin{remark}[]
The local update $\pg[\cl]{t}$ can be considered as a local stochastic pseudo-gradient:
    \begin{align}
        \pg[\cl]{t} = \frac{1}{\lr[c] K} \Delta_{\cl}^{(t)} 
        = \frac{1}{\lr[c] K} (\param[\cl]{t,0} - \param[\cl]{t,E}) 
        = \frac{1}{K} \sum_{k=0}^{K-1} \grad[\cl]{\param[\cl]{t,k}}{\batch[\cl]{t,k}}.
        \label{eq:remark1}
    \end{align}
\end{remark}
\begin{proof}
    Unroll the recursion $\param[\cl]{t,k+1} = \param[\cl]{t,k} - \lr[c] \grad[\cl]{\param[\cl]{t,k}}{\batch[\cl]{t,k}}$ for $k=0,\dots,K-1$.
\end{proof}

\subsubsection*{Additional Notation}
\begin{align}
\text{Local Stochastic Pseudo-Gradient:} \quad \pg[\cl]{t} &= \frac{1}{K} \sum_{k=0}^{K-1} \grad[\cl]{\param[\cl]{t,k}}{\batch[\cl]{t,k}}; \label{def:local_stochastic_pg} \\
\text{Local Pseudo-Gradient:} \quad \barpg[\cl]{t} &= \frac{1}{K} \sum_{k=0}^{K-1} \grad[\cl]{\param[\cl]{t,k}}{}; \label{def:local_pg} \\
\text{Global Stochastic Pseudo-Gradient:} \quad \pg[]{t} &= \avgcl \pg[\cl]{t}; \label{def:global_stochastic_pg} \\
\text{Global Pseudo-Gradient:} \quad \barpg[]{t} &= \avgcl \barpg[\cl]{t}; \label{def:global_pg} \\
\text{Global Stale Pseudo-Gradient:} \quad \mem[]{t} &= \avgcl \mem[\cl]{t}. \label{def:global_mem} 
\end{align}

\subsubsection*{Main Assumptions}

\begin{assumption}
    \label{asm:app:smoothness}
    The gradients of $\obj[\cl]{}$ are $L$-Lipschitz continuous, $\forall \param{}, \cl$.
\end{assumption}
\begin{assumption}
    \label{asm:app:stochastic_grad}
    The stochastic gradients are unbiased: $\E_{\batch[]{}\sim\mathcal{D}_{\cl}}[\grad[\cl]{\param[]{}}{\batch[]{}}] = \grad[\cl]{\param[]{}}{}$ \\
    with bounded variance: $\E_{\batch[]{}\sim\mathcal{D}_{\cl}} \norm{\grad[\cl]{\param[]{}}{\batch[]{}} - \grad[\cl]{\param[]{}}{}}^2 \leq \sigma^2$, $\forall \param{}, \cl$.
\end{assumption}
\begin{assumption}
    \label{asm:app:heterogeneity}
    The divergence between local and global gradients is uniformly bounded: $\norm{\grad[\cl]{\param[]{}}{} - \grad[]{\param[]{}}{}}^2 \leq \sigma_g^2$, $\forall \param{}, \cl$.
\end{assumption}
\begin{assumption}
    \label{asm:app:participation}
    The client participation outcomes follow a Bernoulli distribution with parameter $\prob$, i.e., $\ber[\cl] \sim \text{Bern}(\prob)$.
\end{assumption}

\subsubsection*{Sources of Randomness}
In this system, we model two sources of randomness. The first arises from the partial and heterogeneous participation of clients, which follows a Bernoulli distribution as stated in Assumption~\ref{asm:app:participation}. The second source of randomness originates from the random sampling of data points for stochastic gradients computation.
Recall that $\mathds{1}^{(t)}$ denotes the random set of clients participating at the $t$-th communication round and that $\xi_{\cl}^{(t,k)}$ denotes the random data point independently sampled from client-$\cl$'s local dataset at round~$t$, local iteration~$k$. For the analysis, we introduce the following additional notation:
\begin{itemize}
    \item $\mathds{1}^{(s:q)} \coloneqq \{ \mathds{1}^{(s)}, \dots, \mathds{1}^{(q)} \}$: the random set of clients participating from the $s$-th to the $q$-th communication rounds, $s {<} q$;
    \item $\xi^{(t)}_{\cl} \coloneqq \{ \xi^{(t,k)}_{\cl}\}_{k=0}^{K-1}$: the set of random batches sampled by the $\cl$-th client at the $t$-th communication round;
    \item $\xi^{(t)} \coloneqq \{ \xi^{(t)}_{\cl}  \}_{\cl \in \mathds{1}^{(t)}}$: the set of random batches sampled by the participating clients ($\mathds{1}^{(t)}$) in the $t$-th round;
    \item $\xi^{(t,s:q)}_{\cl} \coloneqq \{ \xi^{(t,s)}_{\cl}, \dots, \xi^{(t,q)}_{\cl} \}$: the set of random batches sampled by the $\cl$-th client at the $t$-th communication round between the $s$-th and the $q$-th local iterations, $s < q$;
    \item $\xi^{(s:q)} \coloneqq \{ \xi^{(s)}, \dots, \xi^{(q)} \}$: the set of random batches sampled by the available clients ($\mathds{1}^{(s:q)}$) between the $s$-th and $q$-th communication rounds, $s < q$.
\end{itemize}

With this notation established, the randomness in the $t$-th communication round, which starts with the initial model $\param{t}$ and yields the updated model $\param{t+1}$, is fully captured by the sets $\mathds{1}^{(t)}$ and $\xi^{(t)}$. 
Thus, the stochastic progression of our algorithm from the first round to round $t$ can be comprehensively described by the tuple:
\begin{align}
    \mathcal{H}^{(t)} \coloneqq \left( \mathds{1}^{(1)}, \dots, \mathds{1}^{(t-1)}; \batch{1}, \dots, \batch{t-1} \right),
\end{align}
which represents the historical information up to the $t$-th communication round.

\begin{remark}
    For any algorithm, $\texttt{Algm} \in [\texttt{FedAvg},~\texttt{FedVARP},~\texttt{FedStale}]$, the global pseudo-gradient $\Delta_{\texttt{Algm}}^{(t)}$ is unbiased with respect to both sources of randomness---client participation and stochastic gradients:
\begin{align*}
    \berexp [\Delta_{\texttt{Algm}}^{(t)}]
    &=
    \pg[]{t}, \\
    \gexp [\Delta_{\texttt{Algm}}^{(t)}]
    &=
    \barpg[]{t},
\end{align*}
and consequently, the global model $\param[]{t+1}$ is also unbiased:
\begin{align*}
    \gexp[\param[]{t+1}]
    = 
    \param[]{t} - \lr[] \gexp [\Delta_{\texttt{Algm}}^{(t)}]
    =
    \param[]{t} - \lr[] \barpg[]{t}.
\end{align*}
\end{remark}

\subsection{Supporting Lemmas}

In this section, we present key lemmas that underpin the theoretical analysis and facilitate the proof of Theorem~\ref{thm:fedhist}.

\begin{lemma}[Descent lemma]
\label{lem:descent_nc}
Let $F: \R^d \rightarrow \R$ be an $L$-smooth function (Assumption~\ref{asm:app:smoothness}), optimized via the sequence of parameters $\{\param{t}\}$. At each iteration $t$, an SGD update is made according to a learning rate $\eta$ and a stochastic gradient $\gps{t}$. Let $\gexp \left[ \gps{t} \right] = \barpg[]{t}$. Then, the expected reduction in $F$ after one iteration is bounded by: 
\begin{align}
    &\gexp \left[ \obj[]{t+1} \right]\notag \\
    &\leq
    \obj[]{t}
    - \frac{\lr[]}{2} \left[ \norm{\grad[]{\param[]{t}}{}}^2 + \norm{\barpg[]{t}}^2 - \norm{\barpg[]{t} - \grad[]{\param[]{t}}{}}^2  \right] 
    + \frac{\lr[]^2 L}{2} \gexp \norm{\gps{t}}^2.
\end{align}
\end{lemma}
\begin{proof}[Proof of Lemma~\ref{lem:descent_nc}] 
By the $L$-smoothness of $F$, it follows that:
    \begin{align}
        \obj[]{t+1}
        &\leq
        \obj[]{t} + \scalar{\grad[]{\param[]{t}}{}}{\param[]{t+1} - \param[]{t}} + \frac{L}{2} \norm{\param[]{t+1} - \param[]{t}}^2 \label{eq:descent_smoothness} \\
        &\leq
        \obj[]{t} - \lr[] \scalar{\grad[]{\param[]{t}}{}}{\gps{t}} + \frac{\lr[]^2 L}{2} \norm{\gps{t}}^2 \label{eq:descent_global_step},
    \end{align}
where Eq.~\eqref{eq:descent_global_step} applies the update rule $\param[]{t+1} = \param[]{t} - \lr[] \Delta^{(t)}$.

Taking expectation over the randomness at the $t$-th round, due to client participation (inherent in $\ber[]$) and to stochastic gradients (inherent in $\batch[]{t} \coloneqq \{ \batch[\cl]{t,k} \}_{i,k}$), yields:
\begin{align}
    &\gexp \left[ \obj[]{t+1} \right]
    \notag \\
    &\leq
    \obj[]{t} - \lr[] \gexp \scalar{\grad[]{\param[]{t}}{}}{\gps{t}} + \frac{\lr[]^2 L}{2} \gexp \norm{\gps{t}}^2 \\
    &\leq
    \obj[]{t} - \lr[] \left[ \scalar{\grad[]{\param[]{t}}{}}{\barpg[]{t}} \right] + \frac{\lr[]^2 L}{2} \gexp \norm{\gps{t}}^2 \label{eq:descent_expectation} \\
    &\leq
    \obj[]{t} 
    - \frac{\lr[]}{2} \left[ \norm{\grad[]{\param[]{t}}{}}^2 + \norm{\barpg[]{t}}^2 - \norm{\barpg[]{t} - \grad[]{\param[]{t}}{}}^2 \right] 
    + \frac{\lr[]^2 L}{2} \gexp \norm{\gps{t}}^2 \label{eq:descent_law_cosines},
\end{align}
where Eq.~\eqref{eq:descent_expectation} uses $\gexp \left[ \gps{t} \right] = \barpg[]{t}$ and Eq.~\eqref{eq:descent_law_cosines} applies the identity $\norm{\bm{a}-\bm{b}}^2 = \norm{\bm{a}}^2 + \norm{\bm{b}}^2 - 2 \scalar{\bm{a}}{\bm{b}}$.
\end{proof}

\begin{lemma}[Expected value of the local stochastic pseudo-gradients]
\label{lem:stochastic_exp_local}
Let $\pg[\cl]{t}$ and $\barpg[\cl]{t}$ be defined in Eqs.~\eqref{def:local_stochastic_pg} and~\eqref{def:local_pg}, respectively. 
If the stochastic gradients are unbiased (Assumption~\ref{asm:app:stochastic_grad}), the following identity holds:
\begin{align}
    \E_{\batch[\cl]{t} \mid \mathcal{H}^{(t)}} \left[ \pg[\cl]{t} \right]  = \barpg[\cl]{t}.
\end{align}
\begin{proof}[Proof of Lemma~\ref{lem:stochastic_exp_local}]
We observe that the randomness in the iterate for a specific client, $\param[\cl]{t,k}$, is influenced both by the sequence of events up to time $t$ (denoted as $\mathcal{H}^{(t)}$) and by the random batches used for training up to the $k$-th iteration ($\batch[\cl]{t,0:k-1}$). \\
We then rely on a fundamental property of expectations to decompose the expected value of the gradient $\grad[\cl]{\param[\cl]{t,k}}{\batch[\cl]{t,k}}$ as:
\begin{align}
    % \E_{\batch[\cl]{t,1},\dots,\batch[\cl]{t,k}\mid\mathcal{H}^{(t)}} \left[\grad[\cl]{\param[\cl]{t,k}}{\batch[\cl]{t,k}}\right] = \E_{\batch[\cl]{t,0},\dots,\batch[\cl]{t,k-1}\mid\mathcal{H}^{(t)}} \left[ \E_{\batch[\cl]{t,k}\mid\batch[\cl]{t,0},\dots,\batch[\cl]{t,k-1},\mathcal{H}^{(t)}} \left[\grad[\cl]{\param[\cl]{t,k}}{\batch[\cl]{t,k}}\right] \right].
    % \label{eq:law_total_expectation}
    \E_{\batch[\cl]{t,0:k}\mid\mathcal{H}^{(t)}} \left[\grad[\cl]{\param[\cl]{t,k}}{\batch[\cl]{t,k}}\right] = \E_{\batch[\cl]{t,0:k-1}\mid\mathcal{H}^{(t)}} \left[ \E_{\batch[\cl]{t,k}\mid\batch[\cl]{t,0:k-1},\mathcal{H}^{(t)}} \left[\grad[\cl]{\param[\cl]{t,k}}{\batch[\cl]{t,k}}\right] \right].
    \label{eq:law_total_expectation}
\end{align}
We finally use Assumption~\ref{asm:app:stochastic_grad} to conclude that $\E_{\batch[\cl]{t,k}\mid\batch[\cl]{t,0:k-1},\mathcal{H}^{(t)}} [\grad[\cl]{\param[\cl]{t,k}}{\batch[\cl]{t,k}}] = \grad[\cl]{\param[\cl]{t,k}}{}$. 

Below, we present the detailed derivations of the proof.
\begin{align}
    \E_{\batch[\cl]{t} \mid \mathcal{H}^{(t)}} \left[ \pg[\cl]{t} \right] 
    &=
    \frac{1}{K} \sum_{k=0}^{K-1} \E_{\batch[\cl]{t} \mid \mathcal{H}^{(t)}} \left[ \grad[\cl]{\param[\cl]{t,k}}{\batch[\cl]{t,k}} \right] \label{eq:lem_exp_def_local_stochastic_pg} \\
    &=
    \frac{1}{K} \Biggl[ 
    \underbrace{\E_{\batch[\cl]{t,0} \mid \mathcal{H}^{(t)}} \left[\grad[\cl]{\param[]{t}}{\batch[\cl]{t,0}}\right]}_{\text{bounded by Assumption~\ref{asm:app:stochastic_grad}}} 
    + 
    \E_{\batch[\cl]{t,0},\batch[\cl]{t,1}\mid \mathcal{H}^{(t)}} \left[\grad[\cl]{\param[\cl]{t,1}}{\batch[\cl]{t,1}}\right] 
    + \cdots +
    \notag \\
    &\hspace{7.4cm} 
    + \cdots +
    \E_{\batch[\cl]{t,0:K-1} \mid \mathcal{H}^{(t)}} \left[\grad[\cl]{\param[\cl]{t,K-1}}{\batch[\cl]{t,K-1}}\right] \Biggr] \label{eq:lem_exp_explicit_exp} \\
    &= 
    \frac{1}{K} \Biggl[ 
    \grad[\cl]{\param[]{t}}{} 
    + 
    \E_{\batch[\cl]{t,0} \mid \mathcal{H}^{(t)}} \biggl[ \underbrace{\E_{\batch[\cl]{t,1} \mid \batch[\cl]{t,0}, \mathcal{H}^{(t)}} \left[\grad[\cl]{\param[\cl]{t,1}}{\batch[\cl]{t,1}}\right] \biggr]}_{\text{bounded by Assumption~\ref{asm:app:stochastic_grad}}} 
    + \cdots + 
    \notag \\
    &\hspace{4cm}
    + \cdots + 
    \E_{\batch[\cl]{t,0:K-2} \mid \mathcal{H}^{(t)}} \biggl[ \underbrace{\E_{\batch[\cl]{t,K-1} \mid \batch[\cl]{t,0:K-2}, \mathcal{H}^{(t)}} \left[\grad[\cl]{\param[\cl]{t,K-1}}{\batch[\cl]{t,K-1}}\right]}_{\text{bounded by Assumption~\ref{asm:app:stochastic_grad}}} \biggr] \Biggr] \label{eq:lem_exp_property_exp} \\
    &= 
    \frac{1}{K} \left[ 
    \grad[\cl]{\param[]{t}}{} 
    + 
    \E_{\batch[\cl]{t,0} \mid \mathcal{H}^{(t)}} \left[\grad[\cl]{\param[\cl]{t,1}}{}\right] 
    + \cdots + 
    \E_{\batch[\cl]{t,0:K-2} \mid \mathcal{H}^{(t)}} \left[\grad[\cl]{\param[\cl]{t,K-1}}{}\right] \right] \label{eq:lem_exp_assumption_exp} \\
    &= 
    \frac{1}{K} \sum_{k=0}^{K-1} \grad[\cl]{\param[\cl]{t,k}}{}
    =
    \barpg[\cl]{t} \vphantom{\frac{1}{K}}, \label{eq:lem_exp_def_local_pg}
\end{align}
where Eq.~\eqref{eq:lem_exp_def_local_stochastic_pg} uses the definition of $\pg[\cl]{t}$ given in~\eqref{def:local_stochastic_pg}, Eq.~\eqref{eq:lem_exp_explicit_exp} makes explicit the dependency of the iterate $\param[\cl]{t,k}$ on the random batches $\batch[\cl]{t,0:k-1}$, Eq.~\eqref{eq:lem_exp_property_exp} uses the law of total expectation given in~\eqref{eq:law_total_expectation}, Eq.~\eqref{eq:lem_exp_assumption_exp} applies the unbiasedness of the stochastic gradient (Assumption~\ref{asm:app:stochastic_grad}), and Eq.~\eqref{eq:lem_exp_def_local_pg} uses the definition of $\barpg[\cl]{t}$ given in~\eqref{def:local_pg}.
\end{proof}
\end{lemma}

\begin{lemma}[Variance of the local stochastic pseudo-gradients]
\label{lem:stochastic_variance_local}
Let $\pg[\cl]{t}$ and $\barpg[\cl]{t}$ be defined as in Eqs.~\eqref{def:local_stochastic_pg} and~\eqref{def:local_pg}, respectively. 
If the variance of the local stochastic gradients is bounded by $\sigma^2$ (Assumption~\ref{asm:app:stochastic_grad}), the following inequality holds:
\begin{align}
    \E_{\batch[\cl]{t} \mid \mathcal{H}^{(t)}} \norm{\pg[\cl]{t} - \barpg[\cl]{t}}^2 
    \leq
    \frac{\sigma^2}{K}.
\end{align}
\begin{proof}[Proof of Lemma~\ref{lem:stochastic_variance_local}]
The proof builds on similar observations to those presented in Lemma~\ref{lem:stochastic_exp_local},
but additionally relies on the bounded variance of local stochastic gradients (Assumption~\ref{asm:app:stochastic_grad}).

Below, the detailed derivations.
\begin{align}
    &\E_{\batch[\cl]{t} \mid \mathcal{H}^{(t)}} \norm{\pg[\cl]{t} - \barpg[\cl]{t}}^2 \notag \\
    &=
    \E_{\batch[\cl]{t} \mid \mathcal{H}^{(t)}} \norm{\frac{1}{K} \sum_{k=0}^{K-1} \left[ \grad[\cl]{\param[\cl]{t,k}}{\batch[\cl]{t,k}} - \grad[\cl]{\param[\cl]{t,k}}{} \right]}^2 \label{eq:lem_loc_var_def} \\
    &= 
    \frac{1}{K^2} \sum_{k=0}^{K-1} \E_{\batch[\cl]{t} \mid \mathcal{H}^{(t)}} \norm{\grad[\cl]{\param[\cl]{t,k}}{\batch[\cl]{t,k}} - \grad[\cl]{\param[\cl]{t,k}}{}}^2 \notag \\
    &\quad
    + 
    \frac{1}{K^2} \sum_{k=0}^{K-1} \sum_{\substack{k'=0\\k'\neq k}}^{K-1} \E_{\batch[\cl]{t} \mid \mathcal{H}^{(t)}} \scalar{\grad[\cl]{\param[\cl]{t,k}}{\batch[\cl]{t,k}} - \grad[\cl]{\param[\cl]{t,k}}{}}{\grad[\cl]{\param[\cl]{t,k'}}{\batch[\cl]{t,k'}} - \grad[\cl]{\param[\cl]{t,k'}}{}}, \label{eq:lem_loc_var_squares}
\end{align}
where Eq.~\eqref{eq:lem_loc_var_def} applies the definitions for $\pg[\cl]{t}$ and $\barpg[\cl]{t}$ given in~\eqref{def:local_stochastic_pg} and~\eqref{def:local_pg}, and Eq.~\eqref{eq:lem_loc_var_squares} expands the squared norm.
To show that the second term in~\eqref{eq:lem_loc_var_squares} is zero, we use the law of total expectation in a similar way as in~\eqref{eq:law_total_expectation}. Indeed, denote $k''=\max\{k, k'\}$. The following relation holds:
\begin{align}
    % &\E_{\batch[\cl]{t,0},\dots,\batch[\cl]{t,k''}\mid\mathcal{H}^{(t)}}\left[ \grad[\cl]{\param[\cl]{t,k''}}{\batch[\cl]{t,k''}} - \grad[\cl]{\param[\cl]{t,k''}}{} \right] \notag \\
    % &=
    % \E_{\batch[\cl]{t,0},\dots,\batch[\cl]{t,k''-1}\mid\mathcal{H}^{(t)}}\left[ \E_{\batch[\cl]{t,k''}\mid\batch[\cl]{t,0},\dots,\batch[\cl]{t,k''-1},\mathcal{H}^{(t)}} \left[ \grad[\cl]{\param[\cl]{t,k''}}{\batch[\cl]{t,k''}} - \grad[\cl]{\param[\cl]{t,k''}}{} \right] \right]
    % =
    % 0.
    &\E_{\batch[\cl]{t,0:k''}\mid\mathcal{H}^{(t)}}\left[ \grad[\cl]{\param[\cl]{t,k''}}{\batch[\cl]{t,k''}} - \grad[\cl]{\param[\cl]{t,k''}}{} \right] \notag \\
    &=
    \E_{\batch[\cl]{t,0:k''-1}\mid\mathcal{H}^{(t)}} \Biggl[ \underbrace{\E_{\batch[\cl]{t,k''}\mid\batch[\cl]{t,0:k''-1},\mathcal{H}^{(t)}} \left[ \grad[\cl]{\param[\cl]{t,k''}}{\batch[\cl]{t,k''}} - \grad[\cl]{\param[\cl]{t,k''}}{} \right]}_{=0 \text{ by Assumption~\ref{asm:app:stochastic_grad}}} \Biggr]
    =
    0.
    \label{eq:lem_loc_var_exp}
\end{align}
Therefore, only the first term remains:
\begin{align}
    &\E_{\batch[\cl]{t} \mid \mathcal{H}^{(t)}} \norm{\pg[\cl]{t} - \barpg[\cl]{t}}^2 \notag \\
    &=
    \frac{1}{K^2} \sum_{k=0}^{K-1} \E_{\batch[\cl]{t} \mid \mathcal{H}^{(t)}} \norm{\grad[\cl]{\param[\cl]{t,k}}{\batch[\cl]{t,k}} - \grad[\cl]{\param[\cl]{t,k}}{}}^2 \label{eq:lem_loc_var_first_term} \\
    &=
    \frac{1}{K^2} \Biggl[ 
    \underbrace{\E_{\batch[\cl]{t,0} \mid \mathcal{H}^{(t)}} \norm{\grad[\cl]{\param[]{t}}{\batch[\cl]{t,0}} - \grad[\cl]{\param[]{t}}{}}^2}_{\leq \sigma^2 \text{ by Assumption~\ref{asm:app:stochastic_grad}}}
    +
    \E_{\batch[\cl]{t,0},\batch[\cl]{t,1}\mid \mathcal{H}^{(t)}} \norm{\grad[\cl]{\param[\cl]{t,1}}{\batch[\cl]{t,1}} - \grad[\cl]{\param[\cl]{t,1}}{}}^2 
    + \cdots +
    \notag \\
    &\hspace{6.5cm} 
    + \cdots +
    \E_{\batch[\cl]{t,0:K-1} \mid \mathcal{H}^{(t)}} \norm{\grad[\cl]{\param[\cl]{t,K-1}}{\batch[\cl]{t,K-1}} - \grad[\cl]{\param[\cl]{t,K-1}}{}}^2 \Biggr] 
    \label{eq:lem_loc_var_explicit_exp} \\
    &\leq 
    \frac{1}{K^2} \Biggl[ 
    \sigma^2
    + 
    \E_{\batch[\cl]{t,0} \mid \mathcal{H}^{(t)}} \biggl[ \underbrace{\E_{\batch[\cl]{t,1} \mid \batch[\cl]{t,0}, \mathcal{H}^{(t)}} \norm{\grad[\cl]{\param[\cl]{t,1}}{\batch[\cl]{t,1}} - \grad[\cl]{\param[\cl]{t,1}}{}}^2}_{\leq \sigma^2 \text{ by Assumption~\ref{asm:app:stochastic_grad}}} \biggr]
    + \cdots + 
    \notag \\
    &\hspace{3.15cm}
    + \cdots + 
    \E_{\batch[\cl]{t,0:K-2} \mid \mathcal{H}^{(t)}} \biggl[ \underbrace{\E_{\batch[\cl]{t,K-1} \mid \batch[\cl]{t,0:K-2}, \mathcal{H}^{(t)}} \norm{\grad[\cl]{\param[\cl]{t,K-1}}{\batch[\cl]{t,K-1}} - \grad[\cl]{\param[\cl]{t,K-1}}{}}^2 \biggr]}_{\leq \sigma^2 \text{ by Assumption~\ref{asm:app:stochastic_grad}}} \Biggr] 
    \label{eq:lem_loc_var_property_exp} \\
    &\leq
    \frac{1}{K^2} \sum_{k=0}^{K-1} \sigma^2 = \frac{\sigma^2}{K} \vphantom{\frac{1}{K}},
    \label{eq:lem_loc_var_assumption_loc_var}
\end{align}
where Eq.~\eqref{eq:lem_loc_var_first_term} uses~\eqref{eq:lem_loc_var_exp}, Eq.~\eqref{eq:lem_loc_var_explicit_exp} explicits the dependency of the iterate $\param[\cl]{t,k}$ on the random batches $\batch[\cl]{t,0:k-1}$, Eq.~\eqref{eq:lem_loc_var_property_exp} applies the law of total expectation given in~\eqref{eq:law_total_expectation}, and Eq.~\eqref{eq:lem_loc_var_assumption_loc_var} uses the uniform bound on the variance of local stochastic gradients (Assumption~\ref{asm:app:stochastic_grad}). 
\end{proof}
\end{lemma}

\begin{lemma}[Variance of the global stochastic pseudo-gradient]
\label{lem:stochastic_variance}
Let $\pg[\cl]{t}$ and $\barpg[\cl]{t}$ be defined as in Eqs.~\eqref{def:local_stochastic_pg} and~\eqref{def:local_pg}, respectively. 
Assuming that client participation outcomes ($\ber[\cl]$) are Bernoulli-distributed with parameter $\prob$, and that the variance of the local stochastic gradients is bounded by $\sigma^2$ (Assumption~\ref{asm:app:stochastic_grad}), the following inequality holds:
\begin{align}
    \gexp 
    \norm{\avgcl \frac{\ber[\cl]}{\prob} \left( \pg[\cl]{t} - \barpg[\cl]{t} \right)}^2 
    \leq
    \left(\avgcl \frac{1}{\prob} \right) \frac{\sigma^2}{N K}.
\end{align}
\end{lemma}
\begin{proof}[Proof of Lemma~\ref{lem:stochastic_variance}]
The proof starts by expanding the squared norm of the average stochastic gradient deviations into a variance term accounting for individual client gradients and a covariance term between gradients from different clients:
\begin{align}
    &\gexp 
    \norm{\avgcl \frac{\ber[\cl]}{\prob} \left( \pg[\cl]{t} - \barpg[\cl]{t} \right)}^2 \notag \\
    &= \gexp  
    \left[\frac{1}{N^2} \sumcl \frac{\left[\ber[\cl]\right]^2}{\prob^2} \norm{\pg[\cl]{t} - \barpg[\cl]{t}}^2 + 
    \frac{1}{N^2} \sumcl \sum_{\substack{\cl'=1 \\ \cl' \neq \cl}}^{N} \frac{\ber[\cl]\ber[\cl']}{p_{\cl} p_{\cl'}}
    \scalar{\pg[\cl]{t} - \barpg[\cl]{t}}{\pg[\cl']{t} - \barpg[\cl']{t}} \right].\label{eq:stochastic_squares}
\end{align}
We leverage the linearity of expectation, the independence of client participation ($\ber[]$) and batch sampling ($\batch[]{t}$), the independence of batch sampling among clients ($\batch[\cl]{t}$ and $\batch[\cl']{t}$), and Lemma~\ref{lem:stochastic_exp_local} to show that:
\begin{align}
    \gexp \left[ \frac{\ber[\cl]\ber[\cl']}{p_{\cl} p_{\cl'}}
    \scalar{\pg[\cl]{t} - \barpg[\cl]{t}}{\pg[\cl']{t} - \barpg[\cl']{t}} \right]
    &=
    \frac{\E_{\xi^{(t)}|\mathcal{H}^{(t)}}[\ber[\cl]\ber[\cl']]}{p_{\cl} p_{\cl'}}
    \E_{\batch[]{t}\mid\mathcal{H}^{(t)}} \left[ \scalar{\pg[\cl]{t} - \barpg[\cl]{t}}{\pg[\cl']{t} - \barpg[\cl']{t}} \right] \\
    &=
    \frac{\E_{\xi^{(t)}|\mathcal{H}^{(t)}}[\ber[\cl]\ber[\cl']]}{p_{\cl} p_{\cl'}}
   \scalar{\underbrace{\E_{\batch[\cl]{t}\mid\mathcal{H}^{(t)}}[\pg[\cl]{t} - \barpg[\cl]{t}]}_{=0 \text{ by Lemma~\ref{lem:stochastic_exp_local}}}}{\underbrace{\E_{\batch[\cl']{t}\mid\mathcal{H}^{(t)}}[\pg[\cl']{t} - \barpg[\cl']{t}]}_{=0 \text{ by Lemma~\ref{lem:stochastic_exp_local}}}}
   = 0.
   \label{eq:lem_glob_var_exp}
\end{align}

Finally, we bound the remaining term using Lemma~\ref{lem:stochastic_variance_local}:
\begin{align}
    \gexp 
    \norm{\avgcl \frac{\ber[\cl]}{\prob} \left( \pg[\cl]{t} - \barpg[\cl]{t} \right)}^2
    &=
    \frac{1}{N^2} \sumcl \frac{\E_{\xi_{\cl}^{(t)} \mid \mathcal{H}^{(t)}} \left[\left(\ber[\cl]\right)^2\right]}{\prob^2} \underbrace{\E_{\batch[\cl]{t} \mid \mathcal{H}^{(t)}} \norm{\pg[\cl]{t} - \barpg[\cl]{t}}^2}_{\text{bounded in Lemma~\ref{lem:stochastic_variance_local}}} \label{eq:stochastic_expectation_split} \\
    &\leq
    \left(\avgcl \frac{1}{\prob} \right) \frac{\sigma^2}{N K},
    \label{eq:lem_glob_var_final}
\end{align}
where Equation~\eqref{eq:stochastic_expectation_split} derives from~\eqref{eq:lem_glob_var_exp} and requires the independence of client participation ($\ber[]$) and batch sampling ($\batch[]{t}$); Equation~\eqref{eq:lem_glob_var_final} replaces the Bernoulli's second-order moment ($\prob$) and applies Lemma~\ref{lem:stochastic_variance_local}.
\end{proof}

\begin{lemma}[Client drift due to multiple local iterations]
\label{lem:local_iterations}
Under bounded local stochastic gradient variance ($\sigma^2$, as per Assumption~\ref{asm:app:stochastic_grad}) and the client learning rate $\lr[c] \leq \frac{1}{2LK}$, the expected squared deviation of a client's pseudo-gradient ($\barpg[\cl]{t}$) from its local gradient ($\grad[\cl]{\param[]{t}}{}$) is bounded as:
\begin{align}
    \E_{\batch[\cl]{t} \mid \mathcal{H}^{(t)}} \norm{\barpg[\cl]{t} - \grad[\cl]{\param[]{t}}{}}^2 
    &\leq 
    2 \lr[c]^2 L^2 K(K - 1) \left[ \frac{\sigma^2}{K} + 2 \norm{\grad[\cl]{\param[]{t}}{}}^2 \right].
    \label{eq:lem_drift_part1}
\end{align}
Additionally, if the variance of local gradients is uniformly bounded across clients (by $\sigma_g^2$, as per Assumption~\ref{asm:app:heterogeneity}):
\begin{align}
    \E_{\batch[\cl]{t} \mid \mathcal{H}^{(t)}} \norm{\barpg[\cl]{t} - \grad[\cl]{\param[]{t}}{}}^2 
    &\leq 
    2 \lr[c]^2 L^2 K(K - 1) \left[ \frac{\sigma^2}{K} + 4 \sigma_g^2 + 4 \norm{\grad[]{\param[]{t}}{}}^2 \right].
    \label{eq:lem:local_iterations}
\end{align}
\end{lemma}
The bound in Eq.~\eqref{eq:lem:local_iterations} captures that, when the number of local iterations $K$ equals 1, $\barpg[\cl]{t}$ and $\grad[\cl]{\param[]{t}}{}$ become equivalent.
\begin{proof}[Proof of Lemma~\ref{lem:local_iterations}]
This proof is borrowed from~\cite{wangTacklingObjectiveInconsistency2020},~\cite[Lemma~6]{jhunjhunwalaFedvarpTacklingVariance2022}. It is included for completeness. \\
The proof starts by replacing the definition of $\barpg[\cl]{t}$ given in~\eqref{def:local_pg}:
\begin{align}
    \E_{\batch[\cl]{t} \mid \mathcal{H}^{(t)}} \norm{\barpg[\cl]{t} - \grad[\cl]{\param[]{t}}{}}^2 
    &=
    \E_{\batch[\cl]{t} \mid \mathcal{H}^{(t)}} \norm{\frac{1}{K} \sum_{k=0}^{K-1} \left( \grad[\cl]{\param[\cl]{t,k}}{} - \grad[\cl]{\param[]{t}}{} \right)}^2 \\
    &\leq
    \frac{1}{K} \sum_{k=0}^{K-1} \E_{\batch[\cl]{t} \mid \mathcal{H}^{(t)}} \norm{\grad[\cl]{\param[\cl]{t,k}}{} - \grad[\cl]{\param[]{t}}{}}^2 \label{eq:lem_drift_jensen} \\
    &\leq
    \frac{L^2}{K} \sum_{k=0}^{K-1} \E_{\batch[\cl]{t} \mid \mathcal{H}^{(t)}} \norm{\param[\cl]{t,k} - \param[]{t}}^2
    \label{eq_proof:lem:FedAvg_grad_ht_err_1},
\end{align}
where Eq.~\eqref{eq:lem_drift_jensen} follows from the Jensen's inequality; Eq.~\eqref{eq_proof:lem:FedAvg_grad_ht_err_1} uses the $L$-smoothness of local gradients (Assumption~\ref{asm:app:smoothness}). \\
Next, the individual difference is bounded as:
\begin{align}
    &\E_{\batch[\cl]{t} \mid \mathcal{H}^{(t)}} \norm{\param[\cl]{t,k} - \param[]{t}}^2 \notag \\
    &= 
    \lr[c]^2 \E_{\batch[\cl]{t} \mid \mathcal{H}^{(t)}} 
    \norm{\sum_{k'=0}^{k-1} \grad[\cl]{\param[\cl]{t,k'}}{\batch[\cl]{t,k'}}}^2 
    \label{eq:lem_drift_update} \\
    &= 
    \lr[c]^2 \left[ \E_{\batch[\cl]{t} \mid \mathcal{H}^{(t)}} 
    \norm{\sum_{k'=0}^{k-1} \left[ \grad[\cl]{\param[\cl]{t,k'}}{\batch[\cl]{t,k'}} - \grad[\cl]{\param[\cl]{t,k'}}{} \right]}^2 
    + \E_{\batch[\cl]{t} \mid \mathcal{H}^{(t)}} \norm{\sum_{k'=0}^{k-1} \grad[\cl]{\param[\cl]{t,k'}}{}}^2 \right] 
    \label{eq:lem_drift_exp} \\
    &\leq 
    \lr[c]^2 \Biggl[ \underbrace{\sum_{k'=0}^{k-1} \E_{\batch[\cl]{t} \mid \mathcal{H}^{(t)}} 
    \norm{\grad[\cl]{\param[\cl]{t,k'}}{\batch[\cl]{t,k'}} - \grad[\cl]{\param[\cl]{t,k'}}{}}^2}_{\text{bounded by Assumption~\ref{asm:app:stochastic_grad}, in a similar way as Lemma~\ref{lem:stochastic_variance_local}}} 
    + k \sum_{k'=0}^{k-1} \E_{\batch[\cl]{t} \mid \mathcal{H}^{(t)}} \norm{\grad[\cl]{\param[\cl]{t,k'}}{}}^2 \Biggr] 
    \label{eq:lem_drift_jensen2} \\
    &\leq 
    \lr[c]^2 \left[ k \sigma^2 + k \sum_{k'=0}^{k-1} \E_{\batch[\cl]{t} \mid \mathcal{H}^{(t)}}
    \norm{\grad[\cl]{\param[\cl]{t,k'}}{} - \grad[\cl]{\param[]{t}}{} + \grad[\cl]{\param[]{t}}{}}^2 \right] 
    \label{eq:lem_drift_plus_minus} \\
    &\leq
    \lr[c]^2 \left[ k \sigma^2 + 2k \sum_{k'=0}^{k-1} \biggl[ L^2 \E_{\batch[\cl]{t} \mid \mathcal{H}^{(t)}} \norm{\param[\cl]{t,k'} - \param[]{t}}^2 + \norm{\grad[\cl]{\param[]{t}}{}}^2 \biggr] \right]
    \label{eq:lem_drift_smoothness},
\end{align}
where Eq.~\eqref{eq:lem_drift_update} applies the local update rule $\param[\cl]{t,k} = \param[]{t} - \lr[c] \sum_{k'=0}^{k-1} \grad[\cl]{\param[\cl]{t,k'}}{\batch[\cl]{t,k'}}$; Eq.~\eqref{eq:lem_drift_exp} leverages the local stochastic gradient unbiasedness (as per Lemma~\ref{lem:stochastic_exp_local}) and its bias-variance decomposition; Eq.~\eqref{eq:lem_drift_jensen2} involves squaring the former term, zeroing the cross terms as per~\eqref{eq:lem_loc_var_exp}, and applying Jensen's inequality to the latter term; Eq.~\eqref{eq:lem_drift_plus_minus} accounts for the bounded variance of local stochastic gradients in the former term (Assumption~\ref{asm:app:stochastic_grad}, as in Lemma~\ref{lem:stochastic_variance_local}), and modifies the latter term by adding and subtracting the initial local gradient ($\grad[\cl]{\param[]{t}}{}$); finally, Eq.~\eqref{eq:lem_drift_smoothness} uses the norm inequality ($\norm{\bm{a}+\bm{b}}^2 \leq 2 \norm{\bm{a}}^2 + 2 \norm{\bm{b}}^2$) and the $L$-smoothness of local objectives (Assumption~\ref{asm:app:smoothness}).

Summing over $k = 0, \dots, K-1$, it yields:
\begin{align}
    &\frac{1}{K} \sum_{k=0}^{K-1} \E_{\batch[\cl]{t} \mid \mathcal{H}^{(t)}} 
    \norm{\param[\cl]{t,k} - \param[]{t}}^2 \notag \\
    &\leq 
    \frac{\lr[c]^2\sigma^2}{K} \sum_{k=0}^{K-1} k + \frac{2 \lr[c]^2 L^2}{K} \sum_{k=0}^{K-1} k \sum_{k'=0}^{k-1} \E_{\batch[\cl]{t} \mid \mathcal{H}^{(t)}} \norm{\param[\cl]{t,k'} - \param[]{t}}^2 + \frac{2 \lr[c]^2}{K} \sum_{k=0}^{K-1} k \sum_{k'=0}^{k-1} \norm{\grad[\cl]{\param[]{t}}{}}^2 
    % \\
    % &\leq
    % \frac{\lr[c]^2 (K - 1) \sigma^2}{2} + \lr[c]^2 L^2 K (K - 1) \left[ \frac{1}{K} \sum_{k=0}^{K-1} \E_{\batch[\cl]{t} \mid \mathcal{H}^{(t)}} 
    % \norm{\param[\cl]{t,k} - \param[]{t}}^2 \right] + \lr[c]^2 K (K - 1) \norm{\grad[\cl]{\param[]{t}}{}}^2 
    \\
    &\leq 
    \lr[c]^2 (K - 1) \sigma^2 + 2 \lr[c]^2 L^2 K (K - 1) \left[ \frac{1}{K} \sum_{k=0}^{K-1} \E_{\batch[\cl]{t} \mid \mathcal{H}^{(t)}} 
    \norm{\param[\cl]{t,k} - \param[]{t}}^2 \right] + 2 \lr[c]^2 K (K - 1) \norm{\grad[\cl]{\param[]{t}}{}}^2,
    \label{eq:lem_drift_sums}
\end{align}
where Eq.~\eqref{eq:lem_drift_sums} uses $\sum_{k'=0}^{k-1} \norm{\param[\cl]{t,k'} - \param[]{t}}^2 \leq \sum_{k=0}^{K-1} \norm{\param[\cl]{t,k} - \param[]{t}}^2$ and $\sum_{k=0}^{K-1} k = \frac{1}{2}(K-1)K$. 

Define $D \coloneqq 2 \lr[c]^2 L^2 K (K - 1)$. Choose $\lr[c]$ small enough such that $D \leq 1/2$ ($\Rightarrow \lr[c] \leq \frac{1}{2LK}$). Rearranging the terms:
\begin{align}
    \frac{1}{K} \sum_{k=0}^{K-1} \E_{\batch[\cl]{t} \mid \mathcal{H}^{(t)}} 
    \norm{\param[\cl]{t,k} - \param[]{t}}^2 
    &\leq 
    \frac{\lr[c]^2 (K - 1) \sigma^2}{1-D} + \frac{2 \lr[c]^2 K (K - 1)}{1-D} \norm{\grad[\cl]{\param[]{t}}{}}^2
    \label{eq:lemma_drift_partial}
\end{align}
Substituting~\eqref{eq:lemma_drift_partial} back into~\eqref{eq_proof:lem:FedAvg_grad_ht_err_1}:
\begin{align}
    \E_{\batch[\cl]{t} \mid \mathcal{H}^{(t)}} \norm{\barpg[\cl]{t} - \grad[\cl]{\param[]{t}}{}}^2 
    &\leq 
    \frac{D}{2(1-D)} \frac{\sigma^2}{K} + \frac{D}{1-D} \norm{\grad[\cl]{\param[]{t}}{}}^2 \\
    & \leq D \frac{\sigma^2}{K} + 2D \norm{\grad[\cl]{\param[]{t}}{}}^2,
    \label{eq:lem_drift_part1_proof}
\end{align}
where Eq.~\eqref{eq:lem_drift_part1_proof} uses $D \leq 1/2$. Replacing $D \coloneqq 2 \lr[c]^2 L^2 K (K - 1)$ into~\eqref{eq:lem_drift_part1_proof} completes the proof of Inequality~\eqref{eq:lem_drift_part1}. \\
Additionally, inequality~\eqref{eq:lem:local_iterations} removes the dependency on $\grad[\cl]{\param[]{t}}{}$ by adding and subtracting $\grad[]{\param[]{t}}{}$ in the squared norm:
\begin{align}
    \E_{\batch[\cl]{t} \mid \mathcal{H}^{(t)}} \norm{\barpg[\cl]{t} - \grad[\cl]{\param[]{t}}{}}^2 
    &\leq 
    D \frac{\sigma^2}{K} + 2D \norm{\grad[\cl]{\param[]{t}}{} - \grad[]{\param[]{t}}{} + \grad[]{\param[]{t}}{}}^2 \\
    &\leq 
    D \frac{\sigma^2}{K} + 4D \norm{\grad[\cl]{\param[]{t}}{} - \grad[]{\param[]{t}}{}}^2 + 4D \norm{\grad[]{\param[]{t}}{}}^2 
    \label{eq:lem_drift_norm_inequality} \\
    & \leq D \frac{\sigma^2}{K} + 4 D \sigma_g^2 + 4D \norm{\grad[]{\param[]{t}}{}}^2
    \label{eq:lemma_drift_final},
\end{align}
where Eq.~\eqref{eq:lem_drift_norm_inequality} uses the norm inequality ($\norm{\bm{a}+\bm{b}}^2 \leq 2 \norm{\bm{a}}^2 + 2 \norm{\bm{b}}^2$) and Eq.~\eqref{eq:lemma_drift_final} leverages the uniform variance bound of local gradients across clients ($\sigma_g^2$, from Assumption~\ref{asm:app:heterogeneity}).

Replacing $D \coloneqq 2 \lr[c]^2 L^2 K (K - 1)$ into~\eqref{eq:lemma_drift_final} concludes the proof of inequality~\eqref{eq:lem:local_iterations}.
\end{proof}

\begin{lemma}[Variance of \texttt{FedStale}'s update]
\label{lem:variance}
Let $\Delta^{(t)}$ denote \texttt{FedStale}'s global update with randomness from client participation ($\ber[]$) and batch sampling ($\batch[]{}$), and $\barpg[]{t}$ its unbiased counterpart, as per Eqs.~\eqref{eq:fedstale2} and~\eqref{def:global_pg}, respectively. Under Assumptions~\ref{asm:app:smoothness}--\ref{asm:app:participation}, we bound the variance of \texttt{FedStale}'s pseudo-gradient, due to partial participation and batch sampling, as:
\begin{align}
    &\gexp \norm{\Delta^{(t)} - \barpg[]{t}}^2 \notag \\
    &\leq 
    \frac{6(1-\beta)^2}{N} \Biggl[ 
    \avgcl \frac{1-\prob}{\prob} \underbrace{\E_{\batch[\cl]{t}\mid\mathcal{H}^{(t)}} \norm{\barpg[\cl]{t} - \grad[\cl]{\param[]{t}}{}}^2}_{\text{uniformly bounded in Lemma~\ref{lem:local_iterations}}}
    + \left(\avgcl \frac{1-\prob}{\prob} \right) \sigma_g^2
    + \left(\avgcl \frac{1-\prob}{\prob} \right) \norm{\grad[]{\param[]{t}}{}}^2 \Biggr]
    \notag \\
    &
    + \frac{6\beta^2}{N} \Biggl[ \avgcl \frac{1-\prob}{\prob} \underbrace{\E_{\batch[\cl]{t}\mid\mathcal{H}^{(t)}} \norm{\barpg[\cl]{t} - \grad[\cl]{\param[]{t}}{}}^2}_{\text{uniformly bounded in Lemma~\ref{lem:local_iterations}}}
    + 
    \eta^2L^2 \left( \avgcl \frac{1-\prob}{\prob} \right) \norm{\gps{t-1}}^2
    +
    N \left( \avgcl \frac{1-\prob}{\prob} \right) H^{(t)} \Biggr] \notag \\
    & +\left(\avgcl \frac{1}{\prob} \right) \frac{\sigma^2}{N K}.
    \label{eq:lem_variance_statement}
\end{align}
\end{lemma}

\begin{proof}[Proof of Lemma~\ref{lem:variance}] 
The proof starts by substituting the definitions for $\Delta^{(t)}$ and $\barpg[]{t}$, given in~\eqref{eq:fedstale2} and~\eqref{def:global_pg}:
\begin{align}
    &\gexp \norm{\Delta^{(t)} - \barpg[]{t}}^2 \notag \\
    &= 
    \gexp  
    \norm{\avgcl \frac{\ber[i]}{\prob} \left( \pg[\cl]{t} - \beta \mem[\cl]{t} \right) - \avgcl \left( \barpg[\cl]{t} - \beta \mem[\cl]{t} \right)}^2 \label{eq:lem_fedhist_variance_def}
    \\
    &= 
    \gexp  
    \norm{\avgcl \frac{\ber[i]}{\prob} \left( \pg[\cl]{t} - \barpg[\cl]{t} + \barpg[\cl]{t} - \beta \mem[\cl]{t} \right) -  \avgcl \left( \barpg[\cl]{t} - \beta \mem[\cl]{t} \right)}^2 \label{eq:lem_fedhist_variance_add_sub}
    \\
    &= 
    \gexp  
    \norm{\avgcl \frac{\ber[i]}{\prob} \left( \barpg[\cl]{t} - \beta \mem[\cl]{t} \right) - \avgcl \left( \barpg[\cl]{t} - \beta \mem[\cl]{t} \right)}^2 
    +
    \underbrace{ \gexp 
    \norm{\avgcl \frac{\ber[i]}{\prob} \left( \pg[\cl]{t} - \barpg[\cl]{t} \right)}^2}_{\text{bounded by Lemma~\ref{lem:stochastic_variance}}} 
    \label{eq:lem_fedhist_variance_unbiased},
\end{align}
where Eq.~\eqref{eq:lem_fedhist_variance_def} uses definitions~\eqref{eq:fedstale2} and~\eqref{def:global_pg}; Eq.~\eqref{eq:lem_fedhist_variance_add_sub} involves adding and subtracting $\barpg[\cl]{t}$ within the squared norm; Eq.~\eqref{eq:lem_fedhist_variance_unbiased} is based on $\barpg[\cl]{t}$ being an unbiased estimator of $\pg[\cl]{t}$ (Lemma~\ref{lem:stochastic_exp_local}). 
The latter term is bounded by Lemma~\ref{lem:stochastic_variance}.

Conditioning on $\mathcal{H}^{(t)}$, $\mem[\cl]{t}$ is constant, and the first term in~\eqref{eq:lem_fedhist_variance_unbiased} represents a variance, due to client participation ($\ber[]$) and stochastic gradients ($\batch[]{t}$). Moreover, conditioning on $\batch[]{t}$, the randomness in Eq.~\eqref{eq:lem_fedhist_variance_varvec} is only due to client participation ($\ber[]$):
\begin{align}
    \VarVec_{\xi^{(t)} \mid \batch[]{t}, \mathcal{H}^{(t)}} 
    \left( \avgcl \frac{\ber[i]}{\prob}  \left( \barpg[\cl]{t} - \beta \mem[\cl]{t} \right) \right) 
    &=
    \VarVec_{\xi^{(t)} \mid \batch[]{t}, \mathcal{H}^{(t)}} 
    \left( \frac{1}{N}  \frac{\ber[i]}{\prob}  \left[ (1-\beta) \barpg[\cl]{t} + \beta \left( \barpg[\cl]{t} - \mem[\cl]{t} \right) \right] \right)  
    \label{eq:lem_fedhist_variance_varvec} \\
    &=
    \frac{1}{N^2} \sumcl \frac{1-\prob}{\prob} \norm{(1-\beta) \barpg[\cl]{t} + \beta \left( \barpg[\cl]{t} - \mem[\cl]{t} \right)}^2 
    \label{eq:lem_fedhist_variance_ber} \\
    &\leq
    \underbrace{\frac{2(1-\beta)^2}{N^2} \sumcl \frac{1-\prob}{\prob} \norm{\barpg[\cl]{t}}^2}_{T_1}
    +
    \underbrace{\frac{2\beta^2}{N^2} \sumcl \frac{1-\prob}{\prob} \norm{\barpg[\cl]{t} - \mem[\cl]{t}}^2}_{T_2},
    \label{eq:lem_fedhist_variance_norm}
\end{align}
where Eq.~\eqref{eq:lem_fedhist_variance_varvec} adds and subtracts $\beta \barpg[\cl]{t}$, then rearranges terms for $\beta$; Eq.~\eqref{eq:lem_fedhist_variance_ber} derives from the Bernoulli variance ($\Var(\ber[\cl])=\prob(1-\prob)$), under Assumption~\ref{asm:app:participation}; Eq.~\eqref{eq:lem_fedhist_variance_norm} leverages the norm inequality $\norm{\bm{a}+\bm{b}}^2 \leq 2 \norm{\bm{a}}^2 + 2 \norm{\bm{b}}^2$.

We proceed by bounding the first term of Eq.~\eqref{eq:lem_fedhist_variance_norm} as follows:
\begin{align}
    &T_1 = \frac{2(1-\beta)^2}{N^2} \sumcl \frac{1-\prob}{\prob} \norm{\barpg[\cl]{t}}^2 \notag \\
    &=
    \frac{2(1-\beta)^2}{N^2} \sumcl \frac{1-\prob}{\prob} \norm{\barpg[\cl]{t} - \grad[\cl]{\param[]{t}}{} + \grad[\cl]{\param[]{t}}{} - \grad[]{\param[]{t}}{} + \grad[]{\param[]{t}}{}}^2 
    \label{eq:lem_fedhist_variance_add_sub2} \\
    &\leq
    \frac{6(1-\beta)^2}{N^2} \sumcl \frac{1-\prob}{\prob} \Biggl[
    \norm{\barpg[\cl]{t} - \grad[\cl]{\param[]{t}}{}}^2
    +
    \underbrace{\norm{\grad[\cl]{\param[]{t}}{} - \grad[]{\param[]{t}}{}}^2}_{\leq \sigma_g^2 \text{ by Assumption~\ref{asm:app:heterogeneity}}}
    + 
    \norm{\grad[]{\param[]{t}}{}}^2 \Biggr]
    \label{eq:lem_fedhist_variance_norm2} \\
    &\leq
    \frac{6(1-\beta)^2}{N} \Biggl[
    \avgcl \frac{1-\prob}{\prob} \underbrace{\norm{\barpg[\cl]{t} - \grad[\cl]{\param[]{t}}{}}^2}_{\text{bounded in expectation by Lemma~\ref{lem:local_iterations}}}
    + \left(\avgcl \frac{1-\prob}{\prob} \right) \sigma_g^2
    + \left(\avgcl \frac{1-\prob}{\prob} \right) \norm{\grad[]{\param[]{t}}{}}^2 \Biggr],
    \label{eq:lem_fedhist_variance_het}
\end{align}
where Eq.~\eqref{eq:lem_fedhist_variance_add_sub2} adds and subtracts the local gradient ($\grad[\cl]{\param[]{t}}{}$) and the global gradient ($\grad[\cl]{\param[]{t}}{}$) within the squared norm; Eq.~\eqref{eq:lem_fedhist_variance_norm2} leverages the norm inequality $\norm{\bm{a}+\bm{b}+\bm{c}}^2 \leq 3 \norm{\bm{a}}^2 + 3 \norm{\bm{b}}^2 + 3 \norm{\bm{c}}^2$; Eq.~\eqref{eq:lem_fedhist_variance_het} applies Assumption~\ref{asm:app:heterogeneity}.

We separately bound the second term of~\eqref{eq:lem_fedhist_variance_norm} as follows:
\begin{align}
    &T_2 = \frac{2\beta^2}{N^2} \sumcl \frac{1-\prob}{\prob} \norm{\barpg[\cl]{t} - \mem[\cl]{t}}^2 \notag \\
    &=
    \frac{2\beta^2}{N^2} \sumcl \frac{1-\prob}{\prob} \norm{\barpg[\cl]{t} - \grad[\cl]{\param[]{t}}{} + \grad[\cl]{\param[]{t}}{} - \grad[\cl]{\param[]{t-1}}{} + \grad[\cl]{\param[]{t-1}}{} - \mem[\cl]{t}}^2
    \label{eq:lem_fedhist_variance_add_sub3}
    \\
    &=
    \frac{6\beta^2}{N^2} \sumcl \frac{1-\prob}{\prob} \left[
    \norm{\barpg[\cl]{t} - \grad[\cl]{\param[]{t}}{}}^2
    + \norm{\grad[\cl]{\param[]{t}}{} - \grad[\cl]{\param[]{t-1}}{}}^2 
    + \norm{\grad[\cl]{\param[]{t-1}}{} - \mem[\cl]{t}}^2 
    \right]
    \label{eq:lem_fedhist_variance_norm3}
    \\
    &\leq
    \frac{6\beta^2}{N^2} \sumcl \frac{1-\prob}{\prob} \norm{\barpg[\cl]{t} - \grad[\cl]{\param[]{t}}{}}^2
    +
    \frac{6\beta^2L^2}{N} \left( \avgcl \frac{1-\prob}{\prob} \right) \norm{\param[]{t} - \param[]{t-1}}^2 \notag \\
    &\quad
    + \frac{6\beta^2}{N^2} \sumcl \frac{1-\prob}{\prob} \norm{\grad[\cl]{\param[]{t-1}}{} - \mem[\cl]{t}}^2 
    \label{eq:lem_fedhist_variance_smoothness} \\
    &\leq
    \frac{6\beta^2}{N^2} \sumcl \frac{1-\prob}{\prob} \norm{\barpg[\cl]{t} - \grad[\cl]{\param[]{t}}{}}^2
    + 
   \frac{6\beta^2\eta^2L^2}{N} \left( \avgcl \frac{1-\prob}{\prob} \right) \norm{\gps{t-1}}^2 \notag \\
    &\quad
    + 6\beta^2 \left( \avgcl \frac{1-\prob}{\prob} \right) \underbrace{\left( \avgcl \norm{\grad[\cl]{\param[]{t-1}}{} - \mem[\cl]{t}}^2 \right)}_{\triangleq H^{(t)}}
    \label{eq:lem_fedhist_variance_bound} \\
    &\leq
    \frac{6\beta^2}{N^2} \sumcl \frac{1-\prob}{\prob} \underbrace{\norm{\barpg[\cl]{t} - \grad[\cl]{\param[]{t}}{}}^2}_{\text{bounded in expectation by Lemma~\ref{lem:local_iterations}}}
    + 
   \frac{6\beta^2\eta^2L^2}{N} \left( \avgcl \frac{1-\prob}{\prob} \right) \norm{\gps{t-1}}^2 \notag \\
    &\quad
    +
    6\beta^2 \left( \avgcl \frac{1-\prob}{\prob} \right) H^{(t)},
    \label{eq:lem_fedhist_variance_hist}
\end{align}
where Eqs.~\eqref{eq:lem_fedhist_variance_add_sub3} and~\eqref{eq:lem_fedhist_variance_norm3} follow the steps of Eqs.~\eqref{eq:lem_fedhist_variance_add_sub2} and~\eqref{eq:lem_fedhist_variance_norm2}; Eq.~\eqref{eq:lem_fedhist_variance_smoothness} is based on the $L$-smoothness of local objectives (Assumption~\ref{asm:app:smoothness}); Eq.~\eqref{eq:lem_fedhist_variance_bound} applies the inequality $\sum_{i=1}^{N} a_i b_i \leq (\sum_{i=1}^{N} a_i)(\sum_{i=1}^{N} b_i)$ for positive $a_i$ and $b_i$; Eq.~\eqref{eq:lem_fedhist_variance_hist} defines
\begin{align}
    H^{(t)} \triangleq \avgcl \norm{\grad[\cl]{\param[]{t-1}}{} - \mem[\cl]{t}}^2.
    \label{def:hist_deviation}
\end{align}

Finally, the bound in Eq.~\eqref{eq:lem_variance_statement} combines Eqs.~\eqref{eq:lem_fedhist_variance_unbiased},~\eqref{eq:lem_fedhist_variance_het}, and~\eqref{eq:lem_fedhist_variance_hist}.
\end{proof}

\begin{lemma}[Bound on the memory term]
\label{lem:memory}
Let $H^{(t)}$, the divergence between the local gradient and the historical pseudo-gradient at time $t$, be defined in Eq.~\eqref{def:hist_deviation}. Under Assumptions~\ref{asm:app:smoothness},~\ref{asm:app:stochastic_grad}, and~\ref{asm:app:participation}, the expected historical error $H^{(t+1)}$ is recursively bounded as:
\begin{align}
    \gexp \left[ H^{(t+1)} \right]
    &\leq
    \left( \avgcl \prob \right) \frac{\sigma^2}{K} 
    +
    \avgcl \prob \E_{\batch[\cl]{t}\mid\mathcal{H}^{(t)}} \norm{\barpg[\cl]{t} - \grad[\cl]{\param[]{t}}{}}^2  \notag \\
    &\quad
    +
    \eta^2 L^2 \left(1+\frac{1}{C}\right) \left(1- \frac{1}{N}\sumcl \prob \right) \norm{\gps{t-1}}^2
    +
    \left(1+C\right) \left(1-p_{\text{min}}\right) H^{(t)}.
    \label{eq:lem_memory}
\end{align}
\end{lemma}
\begin{proof}[Proof of Lemma~\ref{lem:memory}]
The proof starts by definition of $H^{(t+1)}$:
\begin{align}
    &\gexp \left[ H^{(t+1)} \right] \notag \\
    &=
    \avgcl \E_{\batch[\cl]{t}\mid\mathcal{H}^{(t)}} \left[ \E_{\ber[\cl]\mid\batch[\cl]{t},\mathcal{H}^{(t)}} \left[ \norm{\grad[\cl]{\param[]{t}}{} - \mem[\cl]{t+1}}^2 \right] \right]
    \label{eq:lem_memory_law_exp} \\
    &=
    \avgcl \left[\prob \E_{\batch[\cl]{t}\mid\mathcal{H}^{(t)}} \norm{\grad[\cl]{\param[]{t}}{} - \pg[\cl]{t}}^2 + (1-\prob) \norm{\grad[\cl]{\param[]{t}}{} - \mem[\cl]{t}}^2 \right] 
    \label{eq:lem_memory_exp} \\
    &\leq
    \avgcl \prob \underbrace{\E_{\batch[\cl]{t}\mid\mathcal{H}^{(t)}}\norm{\pg[\cl]{t} - \barpg[\cl]{t}}^2}_{\text{bounded by Lemma~\ref{lem:stochastic_variance_local}}}
    +
    \avgcl \prob \E_{\batch[\cl]{t}\mid\mathcal{H}^{(t)}} \norm{\barpg[\cl]{t} - \grad[\cl]{\param[]{t}}{}}^2 \notag \\
    &\quad
    +
    \frac{\left(1+\frac{1}{C}\right)}{N} \sumcl (1-\prob) \norm{\grad[\cl]{\param[]{t}}{} - \grad[\cl]{\param[]{t-1}}{}}^2
    + 
    \frac{\left(1+C\right)}{N} \sumcl (1-\prob) \norm{\grad[\cl]{\param[]{t-1}}{} - \mem[\cl]{t}}^2 
    \label{eq:lem_memory_add_sub} \\
    &\leq
    \left( \avgcl \prob \right) \frac{\sigma^2}{K}
    + 
    \avgcl \prob \underbrace{\E_{\batch[\cl]{t}\mid\mathcal{H}^{(t)}} \norm{\barpg[\cl]{t} - \grad[\cl]{\param[]{t}}{}}^2}_{\text{uniformly bounded by Lemma~\ref{lem:local_iterations}}} \notag \\
    &\quad
    + \frac{\eta^2 L^2 \left(1+\frac{1}{C}\right)}{N} \sumcl (1-\prob) \norm{\gps{t-1}}^2
    + \frac{\left(1+C\right)}{N} \sumcl (1-\prob) \norm{\grad[\cl]{\param[]{t-1}}{} - \mem[\cl]{t}}^2,
    \label{eq:lem_memory_smooth}
\end{align}
where Eq.~\eqref{eq:lem_memory_law_exp} uses the law of total expectation to separate expectations on client participation ($\ber[\cl]$) and batch sampling ($\batch[\cl]{t}$); Eq.~\eqref{eq:lem_memory_exp} solves the inner expectation with respect to client participation ($\ber[\cl]$); 
Eq.~\eqref{eq:lem_memory_add_sub} manipulates the first term by adding and subtracting $\barpg[\cl]{t}$, then leverages the bounded variance of the local stochastic pseudo-gradients (Lemma~\ref{lem:stochastic_variance_local}, Assumption~\ref{asm:app:stochastic_grad}), and similarly corrects the second term with $\grad[\cl]{\param[]{t-1}}{}$, then applies the norm inequality $\norm{\bm{a}+\bm{b}}^2 \leq (1+\frac{1}{C}) \norm{\bm{a}}^2 + (1+C) \norm{\bm{b}}^2$ for any positive~$C$; 
Eq.~\eqref{eq:lem_memory_smooth} is derived from the $L$-smoothness property of local objectives (Assumption~\ref{asm:app:smoothness}).

The final expression in Eq.~\eqref{eq:lem_memory} is derived by observing that $\sum_{i=1}^{N}(1-a_i) b_i \leq (1-a_{\text{min}}) \sum_{i=1}^{N} b_i$.
\end{proof}

\begin{lemma}[Variance of \texttt{FedStale}'s update - Initial condition]
\label{lem:variance_initial}
Denote $\Delta^{(1)}$ and $\barpg[]{1}$ in Eq.~\eqref{eq:fedstale2} and~\eqref{def:global_pg}, respectively. Under Assumptions~\ref{asm:app:stochastic_grad}--\ref{asm:app:participation}, we bound the initial variance of \texttt{FedStale} update, due to partial participation and batch sampling, as:
\begin{align}
    &\E_{\mathds{1}^{(1)},\batch{1}} \norm{\Delta^{(1)} - \barpg[]{1}}^2 
    \leq 
    \left(\avgcl \frac{1}{\prob} \right) \frac{\sigma^2}{N K}
    \notag \\
    &\quad 
    + \frac{3}{N^2} \sumcl \frac{1-\prob}{\prob} \underbrace{\E_{\batch{1}} \norm{\barpg[\cl]{1} - \grad[\cl]{\param[]{1}}{}}^2}_{\text{uniformly bounded by Lemma~\ref{lem:local_iterations}}}
    + \frac{3}{N} \left(\avgcl \frac{1-\prob}{\prob} \right) \sigma_g^2
    + \frac{3}{N} \left(\avgcl \frac{1-\prob}{\prob} \right) \norm{\grad[]{\param[]{1}}{}}^2.
\end{align}
\end{lemma}

\begin{proof}[Proof of Lemma~\ref{lem:variance_initial}]
Similarly to Lemma~\ref{lem:variance}, we start by the definitions of $\Delta^{(1)}$ and $\barpg[]{1}$ given in Eqs.~\eqref{eq:fedstale2} and~\eqref{def:global_pg}:
\begin{align}
    &\E_{\mathds{1}^{(1)},\batch{1}} \norm{\Delta^{(1)} - \barpg[]{1}}^2 \notag \\
    &= 
    \E_{\mathds{1}^{(1)},\batch{1}} 
    \norm{\avgcl \frac{\xi_{\cl}^{(1)}}{\prob} \left( \pg[\cl]{1} - \beta \mem[\cl]{1} \right) - \avgcl \left( \barpg[\cl]{1} - \beta \mem[\cl]{1} \right)}^2 
    \label{eq:lem_fedhist_variance_init_add_sub}
    \\
    &= 
    \E_{\mathds{1}^{(1)},\batch{1}}
    \norm{\avgcl \frac{\xi_{\cl}^{(1)}}{\prob} \barpg[\cl]{1} - \avgcl \barpg[\cl]{1}}^2
    +
    \underbrace{\E_{\mathds{1}^{(1)},\batch{1}} 
    \norm{\avgcl \frac{\xi_{\cl}^{(1)}}{\prob} \left( \pg[\cl]{1} - \barpg[\cl]{1} \right)}^2}_{\text{bounded by Lemma~\ref{lem:stochastic_variance_local}}},
    \label{eq:lem_fedhist_variance_init_unbiased}
\end{align}
where 
Eq.~\eqref{eq:lem_fedhist_variance_init_add_sub} adds and subtracts $\barpg[\cl]{1}$ within the squared norm; 
Eq.~\eqref{eq:lem_fedhist_variance_init_unbiased} is based on $\barpg[\cl]{1}$ being an unbiased estimator of $\pg[\cl]{1}$ (Lemma~\ref{lem:stochastic_exp_local}, Assumption~\ref{asm:app:stochastic_grad}). 
The latter term is bounded by Lemma~\ref{lem:stochastic_variance_local}.

Similarly to Eq.~\eqref{eq:lem_fedhist_variance_varvec}, we observe that the first term in Eq.~\eqref{eq:lem_fedhist_variance_init_unbiased} represents a variance. Conditioning on the first batch sample ($\batch[]{1}$),
we solve the expectation with respect to initial client participation ($\mathds{1}^{(1)}$):
\begin{align}
    &\VarVec_{\mathds{1}^{(1)} \mid \batch[]{1}} \left( \avgcl \frac{\xi_{\cl}^{(1)}}{\prob}  \barpg[\cl]{1} \right) \notag \\
    &=
    \frac{1}{N^2} \sumcl \frac{1-\prob}{\prob} \norm{\barpg[\cl]{1}}^2
    \label{eq:lem_fedhist_variance_init_ber} \\
    &\leq
    \frac{3}{N^2} \sumcl \frac{1-\prob}{\prob} \Biggl[ 
    \norm{\barpg[\cl]{1} - \grad[\cl]{\param[]{1}}{}}^2
    +
    \underbrace{\norm{\grad[\cl]{\param[]{1}}{} - \grad[]{\param[]{1}}{}}^2}_{\leq \sigma_g^2 \text{ by Assumption~\ref{asm:app:heterogeneity}}}
    + 
    \norm{\grad[]{\param[]{1}}{}}^2 
    \Biggr]
    \label{eq:lem_fedhist_variance_init_norm} \\
     &\leq
    \frac{3}{N^2} \sumcl \frac{1-\prob}{\prob} \underbrace{\norm{\barpg[\cl]{1} - \grad[\cl]{\param[]{1}}{}}^2}_{\text{bounded in expectation by Lemma~\ref{lem:local_iterations}}} 
    + \frac{3}{N} \left(\avgcl \frac{1-\prob}{\prob} \right) \sigma_g^2
    + \frac{3}{N} \left(\avgcl \frac{1-\prob}{\prob} \right) \norm{\grad[]{\param[]{1}}{}}^2,
    \label{eq:lem_fedhist_variance_init_het}
\end{align}
where Eq.~\eqref{eq:lem_fedhist_variance_init_ber} derives from the Bernoulli variance ($\Var(\xi_{\cl}^{(1)})=\prob(1-\prob)$), under Assumption~\ref{asm:app:participation}; Eq.~\eqref{eq:lem_fedhist_variance_init_norm} adds and subtracts the local and global initial gradients ($\grad[\cl]{\param[]{1}}{}$ and $\grad[]{\param[]{1}}{}$), then leverages the norm inequality $\norm{\bm{a}+\bm{b}+\bm{c}}^2 \leq 3 \norm{\bm{a}}^2 + 3 \norm{\bm{b}}^2 + 3 \norm{\bm{b}}^2$; finally, Eq.~\eqref{eq:lem_fedhist_variance_init_het} uses Assumption~\ref{asm:app:heterogeneity}.
\end{proof}

\begin{lemma}[Bound on the memory term - Initial condition]
\label{lem:memory_initial}
Let $H^{(1)}$, the initial error due to the historical pseudo-gradients, be defined in Eq.~\eqref{def:hist_deviation}. Under Assumptions~\ref{asm:app:stochastic_grad} and~\ref{asm:app:participation}, the expected error $H^{(2)}$ is bounded as:
\begin{align}
    \E_{\mathds{1}^{(1)},\batch[]{1}} \left[ H^{(2)} \right]
    &\leq
    \left( \avgcl \prob \right) \frac{\sigma^2}{K} 
    +
    \avgcl \prob \underbrace{\E_{\batch[\cl]{1}} \norm{\grad[\cl]{\param[]{1}}{} - \barpg[\cl]{1}}^2}_{\text{uniformly bounded by Lemma~\ref{lem:local_iterations}}}  
    +
    (1 - p_{\text{min}}) \avgcl \norm{\grad[\cl]{\param[]{1}}{} - \mem[\cl]{1}}^2.
    \label{eq:lem_memory_init}
\end{align}
\end{lemma}
\begin{proof}[Proof of Lemma~\ref{lem:memory_initial}]
Similarly to Lemma~\ref{lem:memory}, the proof starts with the definition of $H^{(2)}$:
\begin{align}
    &\E_{\mathds{1}^{(1)},\batch[]{1}} \left[ H^{(2)} \right] \notag \\
    &=
    \avgcl \E_{\batch[\cl]{1}} \left[  \E_{\xi_{\cl}^{(1)}\mid\batch[\cl]{1}} \left[ \norm{\grad[\cl]{\param[]{1}}{} - \mem[\cl]{2}}^2 \right] \right] 
    \label{eq:lem_memory_init_law_exp} \\
    &=
    \avgcl \left[\prob \E_{\batch[\cl]{1}} \norm{\grad[\cl]{\param[]{1}}{} - \pg[\cl]{1}}^2 + (1-\prob) \norm{\grad[\cl]{\param[]{1}}{} - \mem[\cl]{1}}^2 \right] 
    \label{eq:lem_memory_init_exp} \\
    &=
    \avgcl \prob \underbrace{\E_{\batch[\cl]{1}} \norm{\pg[\cl]{1} - \barpg[\cl]{1}}^2}_{\text{bounded by Lemma~\ref{lem:stochastic_variance_local}}}
    +
    \avgcl \prob \underbrace{\E_{\batch[\cl]{1}} \norm{\grad[\cl]{\param[]{1}}{} - \barpg[\cl]{1}}^2 }_{\text{uniformly bounded by Lemma~\ref{lem:local_iterations}}}
    + 
    \avgcl (1-\prob) \norm{\grad[\cl]{\param[]{1}}{} - \mem[\cl]{1}}^2,
    \label{eq:lem_memory_init_add_sub}
\end{align}
where Eq.~\eqref{eq:lem_memory_init_law_exp} uses the law of total expectation to separate expectations on client participation ($\xi_{\cl}^{(1)}$) and batch sampling ($\batch[\cl]{1}$); Eq.~\eqref{eq:lem_memory_init_exp} solves the inner expectation with respect to client participation ($\xi_{\cl}^{(1)}$); 
Eq.~\eqref{eq:lem_memory_init_add_sub} adds and subtracts $\barpg[\cl]{1}$ to the first term, then leverages the local pseudo-gradients' unbiased property (Lemma~\ref{lem:stochastic_exp_local}, Assumption~\ref{asm:app:stochastic_grad}) to separate the two components.

The expression in Eq.~\eqref{eq:lem_memory_init} is finally achieved by observing that $\sum_{i=1}^{N}(1-a_i) b_i \leq (1-a_{\text{min}}) \sum_{i=1}^{N} b_i$.
\end{proof}

\begin{lemma}[\texttt{FedStale}: Per Round Progress]
\label{lem:round_progress}
Under Assumptions~\ref{asm:app:smoothness}--\ref{asm:app:participation}, and appropriate client and server learning rates
\begin{align}
    \textstyle
    \lr[c] \leq \frac{1}{8LK}
    ~ \land ~
    \lr[s] \leq
    \min \left\{
    \frac{N}{12(1-\beta)^2\left( \avgcl \frac{1-\prob}{\prob}\right)},
    \frac{2N}{3 \left( \avgcl \frac{1-\prob}{\prob} \right)},
    \frac{1}{3\beta^2 \left( \avgcl \frac{1-\prob}{\prob} \right) \left( \frac{p_{\text{avg}}}{p_{\text{min}}} \right)}
    \right\},
    \label{eq:lem_progress_lr}
\end{align}
define the following Lyapunov function including the objective value, squared global pseudo-gradient, and historical error term:
\begin{align}
    \lyap{t+1} 
    \coloneqq
    \obj[]{t+1} 
    + 
    \frac{\lr[]^2 L}{2} \norm{\gps{t}}^2 
    + 
    12\eta^2L \beta^2 \left( \avgcl \frac{1-\prob}{\prob} \right)\left(\frac{1}{p_{\text{min}}}\right) H^{(t+1)}.
    \label{eq:lem_progress_lyap}
\end{align}
We bound \texttt{FedStale}'s per-round performance into a progress term---accounting for the decrement in the objective value---and a deviation term---from the stochastic gradient noise and data heterogeneity:
\begin{align}
    &\gexp [\lyap{t+1}] \notag \\
    &\leq 
    \lyap{t} - \frac{\lr[]}{4} \norm{\grad[]{\param[]{t}}{}}^2 \notag \\
    &\quad 
    + \left[ 
    \frac{\lr[]^2L}{N} \left( \avgcl \frac{1}{\prob} \right) 
    +
    12\beta^2\eta^2L \left( \avgcl \frac{1-\prob}{\prob} \right)\left( \avgcl \prob \right)\left(\frac{1}{p_{\text{min}}}\right)
    \right]
    \frac{\sigma^2}{K}
    \notag \\
    &\quad
    + \left[ 
    \frac{\lr[]}{2} 
    + \frac{6\lr[]^2 L}{N} \left( \avgcl \frac{1-\prob}{\prob} \right)
    + 12 \beta^2 \lr[]^2 L \left( \avgcl \frac{1-\prob}{\prob} \right) 
    \left( \avgcl \prob \right) \left( \frac{1}{p_{\text{min}}} \right)
    \right] 2 \lr[c]^2 L^2 K(K - 1) \frac{\sigma^2}{K} 
    \notag \\
    &\quad 
    + \frac{6\lr[]^2L(1-\beta)^2}{N} \left( \avgcl \frac{1-\prob}{\prob}\right) \sigma_g^2 
    \notag \\ 
    &\quad
    + \left[ 
    \frac{\lr[]}{2} 
    + \frac{6\lr[]^2 L}{N} \left( \avgcl \frac{1-\prob}{\prob} \right)
    + 12 \beta^2 \lr[]^2 L \left( \avgcl \frac{1-\prob}{\prob} \right) 
    \left( \avgcl \prob \right) \left( \frac{1}{p_{\text{min}}} \right)
    \right] 8 \lr[c]^2 L^2 K(K - 1) \sigma_g^2.
    \label{eq:lem_progress}
\end{align}
\end{lemma}

\begin{proof}[Proof of Lemma~\ref{lem:round_progress}]
We introduce the following Lyapunov function, also adopted by~\cite{jhunjhunwalaFedvarpTacklingVariance2022}, for any $\frac{\eta^2 L}{2} < \delta \leq \frac{\eta}{2} $ and $\alpha \geq 0$:
\begin{align}
    \lyap{t+1} 
    \coloneqq
    \obj[]{t+1} 
    + 
    \left( \delta - \frac{\lr[]^2 L}{2} \right) \norm{\gps{t}}^2 
    + 
    \alpha \underbrace{\avgcl \norm{\grad[\cl]{\param[]{t}}{}-\mem[\cl]{t+1}}^2}_{\triangleq H^{(t+1)}}.
\end{align}
Considering expectation over the randomness at the $t$-th round and invoking the standard descent lemma for smooth objectives (Assumption~\ref{asm:app:smoothness} and Lemma~\ref{lem:descent_nc}):
\begin{align}
    &\gexp [\lyap{t+1}] \notag \\
    &=
    \gexp \Biggl[
    \obj[]{t+1} 
    + 
    \left( \delta - \frac{\eta^2 L}{2} \right) \norm{\gps{t}}^2 
    + 
    \frac{\alpha}{N} \sumcl \norm{\grad[\cl]{\param[]{t}}{}-\mem[\cl]{t+1}}^2
    \Biggr] \\
    &\leq
    \obj[]{t}
    - \frac{\eta}{2} \norm{\grad[]{\param[]{t}}{}}^2 
    - \frac{\eta}{2} \E_{\batch[]{t}\mid\mathcal{H}^{(t)}} \norm{\barpg[]{t}}^2 
    + \frac{\eta}{2} \E_{\batch[]{t}\mid\mathcal{H}^{(t)}} \norm{\barpg[]{t} - \grad[]{\param[]{t}}{}}^2 \notag \\
    &\quad
    + \frac{\eta^2 L}{2} \gexp \norm{\gps{t}}^2
    + \left( \delta - \frac{\eta^2 L}{2} \right) \gexp \norm{\gps{t}}^2
    + \frac{\alpha}{N} \sumcl \E_{\ber[\cl],\batch[\cl]{t}\mid\mathcal{H}^{(t)}} \norm{\grad[\cl]{\param[]{t}}{}-\mem[\cl]{t+1}}^2 
    \label{eq:lem_progress_lem1} \\
    &\leq
    \obj[]{t} 
    - \frac{\eta}{2} \norm{\grad[]{\param[]{t}}{}}^2 
    - \frac{\eta}{2} \E_{\batch[]{t}\mid\mathcal{H}^{(t)}} \norm{\barpg[]{t}}^2  
    + \frac{\eta}{2N} \sumcl \E_{\batch[\cl]{t}\mid\mathcal{H}^{(t)}} \norm{\barpg[\cl]{t} - \grad[\cl]{\param[]{t}}{}}^2 \notag \\
    &\quad
    + \delta \gexp \norm{\gps{t} - \barpg[]{t} + \barpg[]{t}}^2 
    + \frac{\alpha}{N} \sumcl \E_{\ber[\cl],\batch[\cl]{t}\mid\mathcal{H}^{(t)}} \norm{\grad[\cl]{\param[]{t}}{}-\mem[\cl]{t+1}}^2 
    \label{eq:lem_progress_jensen} \\
    &\leq
    \obj[]{t} 
    - \frac{\eta}{2} \norm{\grad[]{\param[]{t}}{}}^2  
    + \left(\delta-\frac{\eta}{2}\right) \E_{\batch[]{t}\mid\mathcal{H}^{(t)}} \norm{\barpg[]{t}}^2 
    + \frac{\eta}{2N} \sumcl \E_{\batch[\cl]{t}\mid\mathcal{H}^{(t)}} \norm{\barpg[\cl]{t} - \grad[\cl]{\param[]{t}}{}}^2 \notag \\
    &\quad
    + \delta \gexp \norm{\gps{t} - \barpg[]{t}}^2 
    + \alpha \gexp \left[ \avgcl \norm{\grad[\cl]{\param[]{t}}{}-\mem[\cl]{t+1}}^2 \right]
    \label{eq:lem_progress_unbiased} \\
    &\leq
    \obj[]{t} 
    - \frac{\eta}{2} \norm{\grad[]{\param[]{t}}{}}^2
    + \frac{\eta}{2N} \sumcl \E_{\batch[\cl]{t}\mid\mathcal{H}^{(t)}}  \norm{\barpg[\cl]{t} - \grad[\cl]{\param[]{t}}{}}^2 \notag \\
    &\quad
    + \delta \underbrace{\gexp \norm{\gps{t} - \barpg[]{t}}^2}_{\text{bounded by Lemma~\ref{lem:variance}}} 
    + \alpha \underbrace{\gexp \left[ H^{(t+1)} \right]}_{\text{bounded by Lemma~\ref{lem:memory}}},
    \label{eq:lem_progress_req_delta}
\end{align}
where 
Eq.~\eqref{eq:lem_progress_lem1} applies Lemma~\ref{lem:descent_nc}; 
Eq.~\eqref{eq:lem_progress_jensen} follows from Jensen's inequality, and introduces the global pseudo-gradient $\barpg[]{t}$;
Eq.~\eqref{eq:lem_progress_unbiased} requires $\gps{t}$ as an unbiased estimator for $\barpg[]{t}$;
and Eq.~\eqref{eq:lem_progress_req_delta} holds for $\delta \leq \frac{\eta}{2}$.

Next, we apply Lemmas~\ref{lem:variance} and~\ref{lem:memory} into Eq.~\eqref{eq:lem_progress_req_delta}:
\begin{align}
    \gexp [\lyap{t+1}] 
    &\leq
    \obj[]{t} \notag \\
    &\quad
    + \left[ 
    \frac{6\delta\beta^2}{N} \left( \avgcl \frac{1-\prob}{\prob}\right)
    + \alpha \left( 1 + \frac{1}{C} \right) \left( 1 - \avgcl \prob \right)
    \right] 
    \lr[]^2 L^2 \norm{\gps{t-1}}^2
    \notag \\
    &\quad
    + \left[ 
    6\delta\beta^2 \left( \avgcl \frac{1-\prob}{\prob}\right)
    + 
    \alpha \left(1+C\right) \left(1 - p_{\text{min}} \right)
    \right]
    H^{(t)} 
    \notag \\
    &\quad
    - \frac{\lr[]}{2} \left[ 1 - \frac{12\delta(1-\beta)^2}{\eta N}  \left( \avgcl \frac{1-\prob}{\prob}\right) \right] \norm{\grad[]{\param[]{t}}{}}^2 \notag \\
    &\quad
    + \frac{\eta}{2N} \sumcl \E_{\batch[\cl]{t}\mid\mathcal{H}^{(t)}} \norm{\barpg[\cl]{t} - \grad[\cl]{\param[]{t}}{}}^2 
    + \frac{6\delta}{N^2} \sumcl \frac{1-\prob}{\prob} \E_{\batch[\cl]{t}\mid\mathcal{H}^{(t)}} \norm{\barpg[\cl]{t} - \grad[\cl]{\param[]{t}}{}}^2 \notag \\
    &\quad
    + \frac{\alpha}{N} \sumcl \prob \E_{\batch[\cl]{t}\mid\mathcal{H}^{(t)}} \norm{\barpg[\cl]{t} - \grad[\cl]{\param[]{t}}{}}^2
    \notag \\
    &\quad
    + \left[ 
    \frac{\delta}{N} \left( \avgcl \frac{1}{\prob} \right) 
    +
    \alpha \left( \avgcl \prob \right)
    \right]
    \frac{\sigma^2}{K}
    \notag \\
    &\quad
    + \frac{6\delta(1-\beta)^2}{N} \left( \avgcl \frac{1-\prob}{\prob}\right) \sigma_g^2,
    \label{eq:lem_progress_lemmas_update}
\end{align}
where Eq.~\eqref{eq:lem_progress_lemmas_update} is derived by straightforward reordering of terms.

The initial segment of Eq.~\eqref{eq:lem_progress_lemmas_update}---comprising the objective value, squared global pseudo-gradient norm, and historical error at round $t$---qualifies for bounding within the Lyapunov recursive framework.
The conditions for this recursion step are:
\begin{align}
    \left[ 
    \frac{6\delta\beta^2}{N} \left( \avgcl \frac{1-\prob}{\prob}\right)
    + \alpha \left( 1 + \frac{1}{C} \right) \left( 1 - \avgcl \prob \right)
    \right] 
    \lr[]^2 L^2 
    &\leq
    \delta - \frac{\lr[]^2L}{2}; 
    \label{eq:lem_progress_recursion_delta} \\
    6\delta\beta^2 \left( \avgcl \frac{1-\prob}{\prob}\right)
    + 
    \alpha \left(1+C\right) \left(1 - p_{\text{min}} \right)
    &\leq
    \alpha.
    \label{eq:lem_progress_recursion_alpha}
\end{align}
Reversing Eq.~\eqref{eq:lem_progress_recursion_alpha}, the resulting condition on $\alpha$ is:
\begin{align}
    \alpha
    \geq
    \frac{6\delta\beta^2 \left( \avgcl \frac{1-\prob}{\prob} \right)}{\left[ 1 - \left(1+C\right) \left(1 - p_{\text{min}} \right) \right]}.
    \label{eq:lem_progress_condition_alpha}
\end{align}
To ensure Eq.~\eqref{eq:lem_progress_condition_alpha} is positive ($\alpha>0$), a suitable choice for $C$ must satisfy $[1 - (1+C) (1 - p_{\text{min}} )] > 0$:
\begin{align*}
    C < \frac{p_{\text{min}}}{1-p_{\text{min}}} 
    \quad
    \Longrightarrow 
    \quad
    \text{choose } 
    \quad
    C = \frac{p_{\text{min}}}{2(1-p_{\text{min}})}.
\end{align*}
It follows that $1+C = \frac{2-p_{\text{min}}}{2(1-p_{\text{min}})}$, and
\begin{align}
    \alpha 
    % =
    % \frac{6\delta\beta^2 \left( \avgcl \frac{1-\prob}{\prob} \right)}{1-\frac{2-p_{\text{min}}}{2}}
    =
    12\delta\beta^2 \left( \avgcl \frac{1-\prob}{\prob} \right) \left( \frac{1}{p_{\text{min}}} \right).
    \label{eq:lem_progress_alpha_beta}
\end{align}
Reversing Eq.~\eqref{eq:lem_progress_recursion_delta}, the resulting condition on $\delta$ is:
\begin{align}
    \left[ 
    \frac{6\delta\beta^2}{N} \left( \avgcl \frac{1-\prob}{\prob}\right)
    + \alpha \left( 1 + \frac{1}{C} \right) \left( 1 - \avgcl \prob \right)
    \right] 
    \lr[]^2 L^2 
    \leq
    \delta - \frac{\lr[]^2L}{2}.
    \label{eq:lem_progress_condition_delta}
\end{align}
By replacing $\alpha$ (as per Eq.~\eqref{eq:lem_progress_alpha_beta}) and $1+\frac{1}{C} = \frac{2-p_{\text{min}}}{p_{\text{min}}}$ into~\eqref{eq:lem_progress_condition_delta}, we have:
\begin{align}
    \frac{6\delta\beta^2\lr[]^2L^2}{N} \left( \avgcl \frac{1-\prob}{\prob}\right)
    +
    \frac{12\delta\beta^2\lr[]^2L^2 \left( \avgcl \frac{1-\prob}{\prob} \right)\left( 1 - \avgcl \prob \right)\left(2-p_{\text{min}}\right)}{\left(p_{\text{min}}\right)^2}
    \leq
    \delta - \frac{\lr[]^2L}{2},
\end{align}
therefore:
\begin{align}
    \delta 
    \geq
    \frac
    {\frac{\lr[]^2L}{2}}
    {1 
    - \frac{6\beta^2\lr[]^2L^2}{N} \left( \avgcl \frac{1-\prob}{\prob}\right) 
    - \frac{12\beta^2\lr[]^2L^2 \left( \avgcl \frac{1-\prob}{\prob} \right)\left( 1 - \avgcl \prob \right)\left(2-p_{\text{min}}\right)}{\left(p_{\text{min}}\right)^2}}.
\end{align}
Choosing $\delta = \lr[]^2L$ requires the following constraints on the step-size:
\begin{align}
    &\frac{6\beta^2\lr[]^2L^2}{N} \left( \avgcl \frac{1-\prob}{\prob}\right) 
    \leq
    \frac{1}{4}
    \Longrightarrow
    \lr[]^2
    \leq
    \frac{N}{24\beta^2L^2\left( \avgcl \frac{1-\prob}{\prob}\right)}; 
    \label{eq:lem_progress_lr_1}
\end{align}
\begin{align}
    \frac{12\beta^2\lr[]^2L^2 \left( \avgcl \frac{1-\prob}{\prob} \right)\left( 1 - p_{\text{avg}} \right)\left(2-p_{\text{min}}\right)}{\left(p_{\text{min}}\right)^2}
    \leq
    \frac{1}{4} 
    \iff
    \lr[]^2
    \leq
    \frac{\left(p_{\text{min}}\right)^2}{48\beta^2L^2 \left( \avgcl \frac{1-\prob}{\prob} \right)\left( 1 - p_{\text{avg}} \right)\left(2-p_{\text{min}}\right)}.
    \label{eq:lem_progress_lr_2}
\end{align}
Under these requirements, we are able to pick: 
\begin{align}
    &\delta = \lr[]^2L;
    \label{eq:lem_progress_delta} \\
    &\alpha = 12 \eta^2 L \beta^2 \left( \avgcl \frac{1-\prob}{\prob} \right) \left( \frac{1}{p_{\text{min}}} \right).
    \label{eq:lem_progress_alpha}
\end{align}
Replacing the selected values for $\delta$ and $\alpha$ into Eq.~\eqref{eq:lem_progress_lemmas_update} (as per Eq.~\eqref{eq:lem_progress_delta} and~\eqref{eq:lem_progress_alpha}, respectively), yields:
\begin{align}
    \gexp [\lyap{t+1}]
    &\leq
    \lyap{t}
    - \frac{\lr[]}{2} \left[ 1 - \frac{12\lr[] L (1-\beta)^2}{N}  \left( \avgcl \frac{1-\prob}{\prob}\right) \right] \norm{\grad[]{\param[]{t}}{}}^2 \notag \\
    &\quad
    + \frac{\eta}{2N} \sumcl \underbrace{\E_{\batch[\cl]{t}\mid\mathcal{H}^{(t)}} \norm{\barpg[\cl]{t} - \grad[\cl]{\param[]{t}}{}}^2}_{\text{uniformly bounded by Lemma~\ref{lem:local_iterations}}}
    + \frac{6\lr[]^2L}{N^2} \sumcl \frac{1-\prob}{\prob} \underbrace{\E_{\batch[\cl]{t}\mid\mathcal{H}^{(t)}} \norm{\barpg[\cl]{t} - \grad[\cl]{\param[]{t}}{}}^2}_{\text{uniformly bounded by Lemma~\ref{lem:local_iterations}}}
    \notag \\
    &\quad
    + 12\beta^2\eta^2L \left( \avgcl \frac{1-\prob}{\prob} \right) \left(\frac{1}{p_{\text{min}}}\right) \avgcl \prob \underbrace{\E_{\batch[\cl]{t}\mid\mathcal{H}^{(t)}} \norm{\barpg[\cl]{t} - \grad[\cl]{\param[]{t}}{}}^2}_{\text{uniformly bounded by Lemma~\ref{lem:local_iterations}}}
    \notag \\
    &\quad
    + \left[ 
    \frac{\lr[]^2L}{N} \left( \avgcl \frac{1}{\prob} \right) 
    +
    12\beta^2\eta^2L \left( \avgcl \frac{1-\prob}{\prob} \right)\left( \avgcl \prob \right)\left(\frac{1}{p_{\text{min}}}\right)
    \right]
    \frac{\sigma^2}{K}
    \notag \\
    &\quad
    + \frac{6\lr[]^2L(1-\beta)^2}{N} \left( \avgcl \frac{1-\prob}{\prob}\right) \sigma_g^2.
    \label{eq:lem_progress_rep_delta_alpha}
\end{align}
Next, applying Lemma~\ref{lem:local_iterations} into Eq.~\eqref{eq:lem_progress_rep_delta_alpha}:
\begin{align}
    &\gexp [\lyap{t+1}] \notag \\
    &\leq 
    \lyap{t} 
    - \frac{\lr[]}{2} 
    \left[
    1 
    - \frac{12\lr[] L (1-\beta)^2}{N}  \left( \avgcl \frac{1-\prob}{\prob}\right) 
    \right] \norm{\grad[]{\param[]{t}}{}}^2 \notag \\
    &\quad 
    + \frac{\lr[]}{2}
    \left[ 
    1 
    + \frac{12\lr[] L}{N} \left( \avgcl \frac{1-\prob}{\prob} \right)
    + 24 \beta^2 \lr[] L \left( \avgcl \frac{1-\prob}{\prob} \right) 
    \left( \avgcl \prob \right) \left( \frac{1}{p_{\text{min}}} \right)
    \right] 8 \lr[c]^2 L^2 K(K - 1) \norm{\grad[]{\param[]{t}}{}}^2
    \notag \\
    &\quad 
    + \left[ 
    \frac{\lr[]^2L}{N} \left( \avgcl \frac{1}{\prob} \right) 
    +
    12\beta^2\eta^2L \left( \avgcl \frac{1-\prob}{\prob} \right)\left( \avgcl \prob \right)\left(\frac{1}{p_{\text{min}}}\right)
    \right]
    \frac{\sigma^2}{K}
    \notag \\
    &\quad
    + \left[ 
    \frac{\lr[]}{2} 
    + \frac{6\lr[]^2 L}{N} \left( \avgcl \frac{1-\prob}{\prob} \right)
    + 12 \beta^2 \lr[]^2 L \left( \avgcl \frac{1-\prob}{\prob} \right) 
    \left( \avgcl \prob \right) \left( \frac{1}{p_{\text{min}}} \right)
    \right] 2 \lr[c]^2 L^2 K(K - 1) \frac{\sigma^2}{K} 
    \notag \\
    &\quad 
    + \frac{6\lr[]^2L(1-\beta)^2}{N} \left( \avgcl \frac{1-\prob}{\prob}\right) \sigma_g^2 
    \notag \\ 
    &\quad
    + \left[ 
    \frac{\lr[]}{2} 
    + \frac{6\lr[]^2 L}{N} \left( \avgcl \frac{1-\prob}{\prob} \right)
    + 12 \beta^2 \lr[]^2 L \left( \avgcl \frac{1-\prob}{\prob} \right) 
    \left( \avgcl \prob \right) \left( \frac{1}{p_{\text{min}}} \right)
    \right] 8 \lr[c]^2 L^2 K(K - 1) \sigma_g^2,
    \label{eq:lem_progress_lemma_het}
\end{align}
where Eq.~\eqref{eq:lem_progress_lemma_het} requires minor rearrangements.

Achieving Eq.~\eqref{eq:lem_progress}'s per-round progress requires the gradient squared norm's coefficient should not exceed $-\frac{\eta}{4}$. This leads to the following step-size requirements:
% require $-\frac{\lr[]}{2}[1-x]\leq-\frac{\lr[]}{4} \rightarrow x \leq \frac{1}{2}$.
\begin{align}
    &8 \lr[c]^2 L^2 K(K-1) \leq \frac{1}{8}
    \iff
    \lr[c]^2 \leq \frac{1}{64L^2K^2}; 
    \label{eq:lem_progress_lr_3} \\
    &\frac{12\lr[] L (1-\beta)^2}{N}  \left( \avgcl \frac{1-\prob}{\prob}\right) \leq \frac{1}{8}
    \iff
    \lr[] \leq \frac{N}{96(1-\beta)^2L\left( \avgcl \frac{1-\prob}{\prob}\right)}; 
    \label{eq:lem_progress_lr_4} \\
    &\frac{96 \lr[] \lr[c]^2 L^3 K(K-1)}{N} \left( \avgcl \frac{1-\prob}{\prob} \right) 
    \leq \frac{1}{8}
    \iff 
    \lr[] \leq \frac{N}{12L \left( \avgcl \frac{1-\prob}{\prob} \right)}; 
    \label{eq:lem_progress_lr_5} \\
    &192 \beta^2 \lr[] \lr[c]^2 L^3 K(K-1) \left( \avgcl \frac{1-\prob}{\prob} \right) \left( \frac{p_{\text{avg}}}{p_{\text{min}}} \right) \leq \frac{1}{8} 
    \iff 
    \lr[] \leq \frac{1}{24\beta^2 L \left( \avgcl \frac{1-\prob}{\prob} \right) \left( \frac{p_{\text{avg}}}{p_{\text{min}}} \right)}.
    \label{eq:lem_progress_lr_6}
\end{align}
In summary, combining conditions~\eqref{eq:lem_progress_lr_1},~\eqref{eq:lem_progress_lr_2}, and~\eqref{eq:lem_progress_lr_3}--\eqref{eq:lem_progress_lr_6}, the necessary step-size requirements are:
\begin{align}
    &\lr[c] \leq \frac{1}{8LK}; \\
    &
    \lr[s] \leq
    \min \left\{
    \frac{N}{12(1-\beta)^2\left( \avgcl \frac{1-\prob}{\prob}\right)},
    \frac{2N}{3 \left( \avgcl \frac{1-\prob}{\prob} \right)},
    \frac{1}{3\beta^2 \left( \avgcl \frac{1-\prob}{\prob} \right) \left( \frac{p_{\text{avg}}}{p_{\text{min}}} \right)}
    \right\}.
\end{align}
Under these conditions, we bound the gradient squared norm's coefficient with $-\frac{\eta}{4}$, thus deriving Eq.~\eqref{eq:lem_progress}.
\end{proof}

\begin{lemma}[\texttt{FedStale}: Initial progress]
\label{lem:initial_progress}
Under Assumptions~\ref{asm:app:smoothness}--\ref{asm:app:participation} and the specified client-server learning rates (Eq.~\eqref{eq:lem_progress_lr}), recall the definition of Lyapunov function $\lyap{t}$ in Eq.~\eqref{eq:lem_progress_lyap}. 
Define the initial error resulting from the memory term's initialization as:
\begin{align}
    H^{(1)} \coloneqq \avgcl \norm{\grad[\cl]{\param[]{1}}{} - \mem[\cl]{1}}^2.
\end{align}
We decompose \texttt{FedStale}'s initial progress into three main terms: the objective's initial decrease, the initial memory error, and \texttt{FedAvg}'s error from stochastic gradients and data heterogeneity---which \texttt{FedStale} also encounters upon memory initialization:
\begin{align}
    &\E_{\mathds{1}^{(1)},\batch{1}} \left[ \lyap{2} \right] \notag \\
    &\leq
    \obj[]{1}
    - \frac{\lr[]}{4} 
    \norm{\grad[]{\param[]{1}}{}}^2 \notag \\
    &\quad 
    + \left[ 
    \frac{\lr[]^2L}{N} \left( \avgcl \frac{1}{\prob} \right) 
    +
    12\beta^2\eta^2L \left( \avgcl \frac{1-\prob}{\prob} \right)\left( \avgcl \prob \right)\left(\frac{1}{p_{\text{min}}}\right)
    \right]
    \frac{\sigma^2}{K}
    \notag \\
    &\quad
    + \left[ 
    \frac{\lr[]}{2} 
    + \frac{6\lr[]^2 L}{N} \left( \avgcl \frac{1-\prob}{\prob} \right)
    + 12 \beta^2 \lr[]^2 L \left( \avgcl \frac{1-\prob}{\prob} \right) 
    \left( \avgcl \prob \right) \left( \frac{1}{p_{\text{min}}} \right)
    \right] 2 \lr[c]^2 L^2 K(K - 1) \frac{\sigma^2}{K} 
    \notag \\
    &\quad 
    + \frac{3\lr[]^2L}{N} \left( \avgcl \frac{1-\prob}{\prob}\right) \sigma_g^2 
    \notag \\ 
    &\quad
    + \left[ 
    \frac{\lr[]}{2} 
    + \frac{3\lr[]^2 L}{N} \left( \avgcl \frac{1-\prob}{\prob} \right)
    + 12 \beta^2 \lr[]^2 L \left( \avgcl \frac{1-\prob}{\prob} \right) 
    \left( \avgcl \prob \right) \left( \frac{1}{p_{\text{min}}} \right)
    \right] 8 \lr[c]^2 L^2 K(K - 1) \sigma_g^2 \\
    &\quad
    + 12\beta^2\eta^2L \left( \avgcl \frac{1-\prob}{\prob} \right) \left(\frac{1-p_{\text{min}}}{p_{\text{min}}}\right) \avgcl \norm{\grad[\cl]{\param[]{1}}{} - \mem[\cl]{1}}^2.
    \label{eq:initial_progress}
\end{align}
\end{lemma}

\begin{proof}[Proof of Lemma~\ref{lem:initial_progress}]
We bound $\E_{\mathds{1}^{(1)},\batch{1}} [\lyap{2}]$ following Lemma~\ref{lem:round_progress}'s methodology, starting with $\lyap{t}$'s definition from Eq.~\eqref{eq:lem_progress_lyap}:
\begin{align}
    &\E_{\mathds{1}^{(1)},\batch{1}} \left[ \lyap{2} \right] \notag \\
    &= 
    \E_{\mathds{1}^{(1)},\batch{1}} \left[ \obj[]{2} \right] 
    + 
    \left( \delta - \frac{\lr[]^2 L}{2} \right) \E_{\mathds{1}^{(1)},\batch{1}} \norm{\gps{1}}^2 
    + 
    \alpha \E_{\mathds{1}^{(1)},\batch{1}} \underbrace{\left[\avgcl \norm{\grad[\cl]{\param[]{1}}{}-\mem[\cl]{2}}^2\right]}_{\triangleq H^{(2)}} \\
    &\leq
    \obj[]{1}
    - \frac{\lr[]}{2} \norm{\grad[]{\param[]{1}}{}}^2 
    + \frac{\lr[]}{2N} \sumcl \E_{\batch[\cl]{1}} \norm{\grad[\cl]{\param[]{1}}{} - \barpg[\cl]{1}}^2
    \notag \\
    &\quad
    + \delta \underbrace{\E_{\mathds{1}^{(1)},\batch{1}} \norm{\gps{1} - \barpg[]{1}}^2}_{\text{bounded by Lemma~\ref{lem:variance_initial}}} 
    + \alpha \underbrace{\E_{\mathds{1}^{(1)},\batch{1}} \left[ \avgcl \norm{\grad[\cl]{\param[]{1}}{} - \mem[\cl]{2}}^2 \right]}_{\text{bounded by Lemma~\ref{lem:memory_initial}}},
    \label{eq:lem_initial_progress_init}
\end{align}
where Eq.\eqref{eq:lem_initial_progress_init} leverages Lemma\ref{lem:descent_nc} and Jensen's inequality, with $\gps{t}$ as an unbiased estimator of $\barpg[]{t}$ and $\delta \leq \frac{\eta}{2}$, following Eqs.~\eqref{eq:lem_progress_lem1}--\eqref{eq:lem_progress_req_delta}.
Next, we apply Lemmas~\ref{lem:variance_initial} and~\ref{lem:memory_initial} into Eq.~\eqref{eq:lem_initial_progress_init}:
\begin{align}
    &\E_{\mathds{1}^{(1)},\batch{1}} \left[ \lyap{2} \right] \notag \\
    &\leq
    \obj[]{1}
    - \frac{\lr[]}{2} \norm{\grad[]{\param[]{1}}{}}^2 
    + \frac{\lr[]}{2N} \sumcl \underbrace{\E_{\batch[\cl]{1}} \norm{\grad[\cl]{\param[]{1}}{} - \barpg[\cl]{1}}^2}_{\text{uniformly bounded by Lemma~\ref{lem:local_iterations}}}
    \notag \\
    &\quad + \frac{\delta}{N}\left(\avgcl \frac{1}{\prob} \right) \frac{\sigma^2}{K}
    + \frac{3\delta}{N^2} \sumcl \frac{1-\prob}{\prob} \underbrace{\E_{\batch[\cl]{1}} \norm{\grad[\cl]{\param[]{1}}{} - \barpg[\cl]{1}}^2}_{\text{uniformly bounded by Lemma~\ref{lem:local_iterations}}}
    \notag \\
    &\quad 
    + \frac{3\delta}{N} \left(\avgcl \frac{1-\prob}{\prob} \right) \sigma_g^2
    + \frac{3\delta}{N} \left(\avgcl \frac{1-\prob}{\prob} \right) \norm{\grad[]{\param[]{1}}{}}^2 \notag \\
    &\quad
    + \alpha \left( \avgcl \prob \right) \frac{\sigma^2}{K} 
    + \alpha \avgcl \prob \underbrace{\E_{\batch[\cl]{1}} \norm{\grad[\cl]{\param[]{1}}{} - \barpg[\cl]{1}}^2}_{\text{uniformly bounded by Lemma~\ref{lem:local_iterations}}}
    + \alpha (1 - p_{\text{min}}) \avgcl \norm{\grad[\cl]{\param[]{1}}{} - \mem[\cl]{1}}^2.
\end{align}
Then, we invoke Lemma~\ref{lem:local_iterations}:
\begin{align}
    &\E_{\mathds{1}^{(1)},\batch{1}} \left[ \lyap{2} \right] \notag \\
    &\leq
    \obj[]{1}
    - \frac{\lr[]}{2} 
    \left[
    1 
    - \frac{6\lr[] L}{N}  \left( \avgcl \frac{1-\prob}{\prob}\right) 
    \right] \norm{\grad[]{\param[]{1}}{}}^2 \notag \\
    &\quad 
    + \frac{\lr[]}{2}
    \left[ 
    1 
    + \frac{6\lr[] L}{N} \left( \avgcl \frac{1-\prob}{\prob} \right)
    + 24 \beta^2 \lr[] L \left( \avgcl \frac{1-\prob}{\prob} \right) 
    \left( \avgcl \prob \right) \left( \frac{1}{p_{\text{min}}} \right)
    \right] 8 \lr[c]^2 L^2 K(K - 1) \norm{\grad[]{\param[]{1}}{}}^2
    \notag \\
    &\quad 
    + \left[ 
    \frac{\lr[]^2L}{N} \left( \avgcl \frac{1}{\prob} \right) 
    +
    12\beta^2\eta^2L \left( \avgcl \frac{1-\prob}{\prob} \right)\left( \avgcl \prob \right)\left(\frac{1}{p_{\text{min}}}\right)
    \right]
    \frac{\sigma^2}{K}
    \notag \\
    &\quad
    + \left[ 
    \frac{\lr[]}{2} 
    + \frac{6\lr[]^2 L}{N} \left( \avgcl \frac{1-\prob}{\prob} \right)
    + 12 \beta^2 \lr[]^2 L \left( \avgcl \frac{1-\prob}{\prob} \right) 
    \left( \avgcl \prob \right) \left( \frac{1}{p_{\text{min}}} \right)
    \right] 2 \lr[c]^2 L^2 K(K - 1) \frac{\sigma^2}{K} 
    \notag \\
    &\quad 
    + \frac{3\lr[]^2L}{N} \left( \avgcl \frac{1-\prob}{\prob}\right) \sigma_g^2 
    \notag \\ 
    &\quad
    + \left[ 
    \frac{\lr[]}{2} 
    + \frac{3\lr[]^2 L}{N} \left( \avgcl \frac{1-\prob}{\prob} \right)
    + 12 \beta^2 \lr[]^2 L \left( \avgcl \frac{1-\prob}{\prob} \right) 
    \left( \avgcl \prob \right) \left( \frac{1}{p_{\text{min}}} \right)
    \right] 8 \lr[c]^2 L^2 K(K - 1) \sigma_g^2 \notag \\
    &\quad
    + 12\beta^2\eta^2L \left( \avgcl \frac{1-\prob}{\prob} \right) \left(\frac{1-p_{\text{min}}}{p_{\text{min}}}\right) \avgcl \norm{\grad[\cl]{\param[]{1}}{} - \mem[\cl]{1}}^2,
    \label{eq:lem_initial_progress_local_steps}
\end{align}
where Eq.~\eqref{eq:lem_initial_progress_local_steps} results from rearrangement of terms.

Finally, applying the step-size criteria in Eq.~\eqref{eq:lem_progress_lr}, Eq.~\eqref{eq:initial_progress} achieves \texttt{FedStale}'s per-round progress of $-\frac{\eta}{4} \norm{\grad[]{\param[]{1}}{}}^2$.
% \begin{align}
%     &\E_{\xi^{(1)},\batch{1}} \left[ \lyap{2} \right] \notag \\
%     &\leq
%     \obj[]{1}
%     - \frac{\lr[]}{4} 
%     \norm{\grad[]{\param[]{1}}{}}^2 \notag \\
%     &\quad 
%     + \left[ 
%     \frac{\lr[]^2L}{N} \left( \avgcl \frac{1}{\prob} \right) 
%     +
%     12\beta^2\eta^2L \left( \avgcl \frac{1-\prob}{\prob} \right)\left( \avgcl \prob \right)\left(\frac{1}{p_{\text{min}}}\right)
%     \right]
%     \frac{\sigma^2}{K}
%     \notag \\
%     &\quad
%     + \left[ 
%     \frac{\lr[]}{2} 
%     + \frac{6\lr[]^2 L}{N} \left( \avgcl \frac{1-\prob}{\prob} \right)
%     + 12 \beta^2 \lr[]^2 L \left( \avgcl \frac{1-\prob}{\prob} \right) 
%     \left( \avgcl \prob \right) \left( \frac{1}{p_{\text{min}}} \right)
%     \right] 2 \lr[c]^2 L^2 K(K - 1) \frac{\sigma^2}{K} 
%     \notag \\
%     &\quad 
%     + \frac{3\lr[]^2L}{N} \left( \avgcl \frac{1-\prob}{\prob}\right) \sigma_g^2 
%     \notag \\ 
%     &\quad
%     + \left[ 
%     \frac{\lr[]}{2} 
%     + \frac{3\lr[]^2 L}{N} \left( \avgcl \frac{1-\prob}{\prob} \right)
%     + 12 \beta^2 \lr[]^2 L \left( \avgcl \frac{1-\prob}{\prob} \right) 
%     \left( \avgcl \prob \right) \left( \frac{1}{p_{\text{min}}} \right)
%     \right] 8 \lr[c]^2 L^2 K(K - 1) \sigma_g^2 \notag \\
%     &\quad
%     + 12\beta^2\eta^2L \left( \avgcl \frac{1-\prob}{\prob} \right) \left(\frac{1-p_{\text{min}}}{p_{\text{min}}}\right) \avgcl \norm{\grad[\cl]{\param[]{1}}{} - \mem[\cl]{1}}^2
% \end{align}
\end{proof}

\subsection{Proof of Theorem~\ref{thm:fedhist}}

\begin{theorem}[\texttt{FedStale}'s Convergence]
\label{thm:fedhist}
Within Assumptions~\ref{asm:app:smoothness}--\ref{asm:app:participation} and specified learning rates (Eq.~\eqref{eq:lem_progress_lr}), \texttt{FedStale}'s expected squared gradient norm over $T$ rounds is influenced by the initial errors (related to $\param[]{1}$ and $H^{(1)}$), deviations from stochastic gradient variance ($\sigma^2$) and data heterogeneity ($\sigma_g^2$), and the critical hyper-parameter $\beta$ controlling stale updates influence:
% Under Assumptions~\ref{asm:app:smoothness}--\ref{asm:app:participation} and adequate learning rates (Eq.~\eqref{eq:lem_progress_lr}), 
% \texttt{FedStale}'s minimal expected squared gradient norm over $T$ rounds is constrained as follows,
% as function of the initial error---reflecting the initial choices for the iterate ($\param[]{1}$) and memory ($H^{(1)}$)---, deviation terms---from the stochastic gradients variance ($\sigma^2$) and data heterogeneity ($\sigma_g^2$)---, and the historical gradient importance hyper-parameter ($\beta$):
\begin{align}
    &\min_{t\in[1,T]} \E \norm{\grad[]{\param[]{t}}{}}^2 \notag \\
    &\leq
    \frac{4\left( \obj[]{1} - \E[\lyap{T}]\right)}{\lr[] T} 
    + 
    \beta^2 \left [ \frac{48 \lr[] L }{T} \left( \avgcl \frac{1-\prob}{\prob} \right) \left(\frac{1-p_{\text{min}}}{p_{\text{min}}}\right) \avgcl \norm{\grad[\cl]{\param[]{1}}{} - \mem[\cl]{1}}^2 \right]
    \notag \\
    &\quad
    + \frac{4\lr[] L}{N} \left( \avgcl \frac{1}{\prob} \right) \frac{\sigma^2}{K}
    + 4 \lr[c]^2 L^2 K(K - 1) \frac{\sigma^2}{K} 
    + \frac{48 \lr[] \lr[c]^2 L^3 K(K - 1)}{N} \left( \avgcl \frac{1-\prob}{\prob} \right) \frac{\sigma^2}{K} \notag \\
    &\quad
    + \beta^2 \left[ 
    48\lr[] L \left( \avgcl \frac{1-\prob}{\prob} \right)\left( \avgcl \prob \right)\left(\frac{1}{p_{\text{min}}}\right)
    \right]
    \frac{\sigma^2}{K}
    \notag \\
    &\quad
    + \beta^2 \left[ 
    96 \lr[] L \left( \avgcl \frac{1-\prob}{\prob} \right) 
    \left( \avgcl \prob \right) \left( \frac{1}{p_{\text{min}}} \right)
    \right] \lr[c]^2 L^2 K(K - 1) \frac{\sigma^2}{K} 
    \notag \\
    &\quad
    + 16 \lr[c]^2 L^2 K(K - 1) \sigma_g^2
    + \frac{192 \lr[] \lr[c]^2 L^3 K(K - 1)}{N} \left( \avgcl \frac{1-\prob}{\prob} \right) \sigma_g^2
    \notag \\
    &\quad 
    + (1-\beta)^2 \left[ \frac{24\lr[] L}{N} \left( \avgcl \frac{1-\prob}{\prob}\right) \right] \sigma_g^2 
    \notag \\ 
    &\quad
    + \beta^2 \left[ 
    384 \lr[] L \left( \avgcl \frac{1-\prob}{\prob} \right) 
    \left( \avgcl \prob \right) \left( \frac{1}{p_{\text{min}}} \right)
    \right] \lr[c]^2 L^2 K(K - 1) \sigma_g^2
    + \frac{12\lr[]L}{NT} \left( \avgcl \frac{1-\prob}{\prob}\right) \sigma_g^2.
    \label{eq:thm_fedhist}
\end{align}
\end{theorem}

\begin{proof}[Proof of Theorem~\ref{thm:fedhist}]
The proof relies on Lemmas~\ref{lem:round_progress} and~\ref{lem:initial_progress}, under Assumptions~\ref{asm:app:smoothness}--\ref{asm:app:participation} and the specified learning rates (Eq.~\eqref{eq:lem_progress_lr}).

From Lemma~\ref{lem:round_progress}, unfolding the recursion for $t=2,\dots,T$ through the law of total expectation, it yields:
\begin{align}
    &\E_{\mathds{1}^{(2:T)},\batch{2:T}\mid\mathcal{H}^{(1)}} \left[ \lyap{T} \right] \notag \\
    &\leq
    \underbrace{\lyap{2}}_{\text{bounded in expectation by Lemma~\ref{lem:initial_progress}}} 
    + \sum_{t=2}^{T} \Biggl(
    -\frac{\lr[]}{4} \E_{\mathds{1}^{(2:T)},\batch{2:T}\mid\mathcal{H}^{(1)}} \norm{\grad[]{\param[]{t}}{}}^2
    \notag \\
    &\quad
    + \left[ 
    \frac{\lr[]^2L}{N} \left( \avgcl \frac{1}{\prob} \right) 
    +
    12\beta^2\eta^2L \left( \avgcl \frac{1-\prob}{\prob} \right)\left( \avgcl \prob \right)\left(\frac{1}{p_{\text{min}}}\right)
    \right]
    \frac{\sigma^2}{K}
    \notag \\
    &\quad
    + \left[ 
    \frac{\lr[]}{2} 
    + \frac{6\lr[]^2 L}{N} \left( \avgcl \frac{1-\prob}{\prob} \right)
    + 12 \beta^2 \lr[]^2 L \left( \avgcl \frac{1-\prob}{\prob} \right) 
    \left( \avgcl \prob \right) \left( \frac{1}{p_{\text{min}}} \right)
    \right] 2 \lr[c]^2 L^2 K(K - 1) \frac{\sigma^2}{K} 
    \notag \\
    &\quad 
    + \frac{6\lr[]^2L(1-\beta)^2}{N} \left( \avgcl \frac{1-\prob}{\prob}\right) \sigma_g^2 
    \notag \\ 
    &\quad
    + \left[ 
    \frac{\lr[]}{2} 
    + \frac{6\lr[]^2 L}{N} \left( \avgcl \frac{1-\prob}{\prob} \right)
    + 12 \beta^2 \lr[]^2 L \left( \avgcl \frac{1-\prob}{\prob} \right) 
    \left( \avgcl \prob \right) \left( \frac{1}{p_{\text{min}}} \right)
    \right] 8 \lr[c]^2 L^2 K(K - 1) \sigma_g^2
    \Biggr).
    \label{eq:thm_fedhist_recursion}
\end{align}

Invoking Lemma~\ref{lem:initial_progress} into Eq.~\eqref{eq:thm_fedhist_recursion} and taking the total expectation:
\begin{align}
    &\E_{\mathds{1}^{(1:T)},\batch{1:T}} \left[ \lyap{T} \right] \notag\\
    &\leq
    \obj[]{1} 
    + \sum_{t=1}^{T} \Biggl(
    -\frac{\lr[]}{4} \E_{\mathds{1}^{(1:T)},\batch{1:T}} \norm{\grad[]{\param[]{t}}{}}^2 
    \notag \\
    &\quad
    + \left[ 
    \frac{\lr[]^2L}{N} \left( \avgcl \frac{1}{\prob} \right) 
    +
    12\beta^2\eta^2L \left( \avgcl \frac{1-\prob}{\prob} \right)\left( \avgcl \prob \right)\left(\frac{1}{p_{\text{min}}}\right)
    \right]
    \frac{\sigma^2}{K}
    \notag \\
    &\quad
    + \left[ 
    \frac{\lr[]}{2} 
    + \frac{6\lr[]^2 L}{N} \left( \avgcl \frac{1-\prob}{\prob} \right)
    + 12 \beta^2 \lr[]^2 L \left( \avgcl \frac{1-\prob}{\prob} \right) 
    \left( \avgcl \prob \right) \left( \frac{1}{p_{\text{min}}} \right)
    \right] 2 \lr[c]^2 L^2 K(K - 1) \frac{\sigma^2}{K} 
    \notag \\
    &\quad
    + \left[ 
    \frac{\lr[]}{2} 
    + \frac{6\lr[]^2 L}{N} \left( \avgcl \frac{1-\prob}{\prob} \right)
    + 12 \beta^2 \lr[]^2 L \left( \avgcl \frac{1-\prob}{\prob} \right) 
    \left( \avgcl \prob \right) \left( \frac{1}{p_{\text{min}}} \right)
    \right] 8 \lr[c]^2 L^2 K(K - 1) \sigma_g^2
    \Biggr) \notag \\
    &\quad 
    + \sum_{t=2}^{T} \left[ \frac{6\lr[]^2L(1-\beta)^2}{N} \left( \avgcl \frac{1-\prob}{\prob}\right) \sigma_g^2 \right]
    + \frac{3\lr[]^2L}{N} \left( \avgcl \frac{1-\prob}{\prob}\right) \sigma_g^2 
    \notag \\ 
    &\quad
    + 12\beta^2\eta^2L \left( \avgcl \frac{1-\prob}{\prob} \right) \left(\frac{1-p_{\text{min}}}{p_{\text{min}}}\right) \avgcl \norm{\grad[\cl]{\param[]{1}}{} - \mem[\cl]{1}}^2.
    \label{eq:thm_fedhist_initial}
\end{align}
% \begin{align*}
%     &\lyap{T} 
%     \leq
%     \lyap{1} 
%     + \sum_{t=1}^{T} \Biggl(
%     -\frac{\lr[]}{4} \norm{\grad[]{\param[]{t}}{}}^2
%     \notag \\
%     &\quad
%     + \left[ 
%     \frac{\lr[]^2L}{N} \left( \avgcl \frac{1}{\prob} \right) 
%     +
%     12\beta^2\eta^2L \left( \avgcl \frac{1-\prob}{\prob} \right)\left( \avgcl \prob \right)\left(\frac{1}{p_{\text{min}}}\right)
%     \right]
%     \frac{\sigma^2}{K}
%     \notag \\
%     &\quad
%     + \left[ 
%     \frac{\lr[]}{2} 
%     + \frac{6\lr[]^2 L}{N} \left( \avgcl \frac{1-\prob}{\prob} \right)
%     + 12 \beta^2 \lr[]^2 L \left( \avgcl \frac{1-\prob}{\prob} \right) 
%     \left( \avgcl \prob \right) \left( \frac{1}{p_{\text{min}}} \right)
%     \right] 2 \lr[c]^2 L^2 K(K - 1) \frac{\sigma^2}{K} 
%     \notag \\
%     &\quad 
%     + \frac{6\lr[]^2L(1-\beta)^2}{N} \left( \avgcl \frac{1-\prob}{\prob}\right) \sigma_g^2 
%     \notag \\ 
%     &\quad
%     + \left[ 
%     \frac{\lr[]}{2} 
%     + \frac{6\lr[]^2 L}{N} \left( \avgcl \frac{1-\prob}{\prob} \right)
%     + 12 \beta^2 \lr[]^2 L \left( \avgcl \frac{1-\prob}{\prob} \right) 
%     \left( \avgcl \prob \right) \left( \frac{1}{p_{\text{min}}} \right)
%     \right] 8 \lr[c]^2 L^2 K(K - 1) \sigma_g^2
%     \Biggr)
%     \notag \\
%     &\quad
%     + \lr[]^2 L \norm{\gps{1} - \barpg[]{1}}^2 
%     + 12\beta^2\eta^2L \left( \avgcl \frac{1-\prob}{\prob} \right) \left(\frac{1-p_{\text{min}}}{p_{\text{min}}}\right) \avgcl \norm{\grad[\cl]{\param[]{1}}{} - \mem[\cl]{1}}^2
% \end{align*}
Dividing both sides of Eq.~\eqref{eq:thm_fedhist_initial} by $T$ and rearranging the terms:
\begin{align}
    &\min_{t\in[1,T]} \E \norm{\grad[]{\param[]{t}}{}}^2 \notag \\
    &\leq
    \frac{1}{T} \sum_{t=1}^{T} \E \norm{\grad[]{\param[]{t}}{}}^2 
    \label{eq:thm_fedhist_min} \\
    &\leq
    \frac{4\left( \obj[]{1} - \E[\lyap{T}]\right)}{\lr[] T}
    + \left[ 
    \frac{4\lr[] L}{N} \left( \avgcl \frac{1}{\prob} \right) 
    +
    48\beta^2\lr[] L \left( \avgcl \frac{1-\prob}{\prob} \right)\left( \avgcl \prob \right)\left(\frac{1}{p_{\text{min}}}\right)
    \right]
    \frac{\sigma^2}{K}
    \notag \\
    &\quad
    + \left[ 
    4
    + \frac{48\lr[] L}{N} \left( \avgcl \frac{1-\prob}{\prob} \right)
    + 96 \beta^2 \lr[] L \left( \avgcl \frac{1-\prob}{\prob} \right) 
    \left( \avgcl \prob \right) \left( \frac{1}{p_{\text{min}}} \right)
    \right] \lr[c]^2 L^2 K(K - 1) \frac{\sigma^2}{K} 
    \notag \\
    &\quad
    + \left[ 
    16
    + \frac{192 \lr[] L}{N} \left( \avgcl \frac{1-\prob}{\prob} \right)
    + 384 \beta^2 \lr[] L \left( \avgcl \frac{1-\prob}{\prob} \right) 
    \left( \avgcl \prob \right) \left( \frac{1}{p_{\text{min}}} \right)
    \right] \lr[c]^2 L^2 K(K - 1) \sigma_g^2
    \notag \\
    &\quad 
    + \frac{T-1}{T} \left[ \frac{24\lr[] L(1-\beta)^2}{N} \left( \avgcl \frac{1-\prob}{\prob}\right) \sigma_g^2 \right] + \frac{1}{T} \left[ \frac{12\lr[]L}{N} \left( \avgcl \frac{1-\prob}{\prob}\right) \sigma_g^2 \right]
    \notag \\ 
    &\quad
    + \frac{1}{T} \left[ 
    48 \beta^2\lr[] L \left( \avgcl \frac{1-\prob}{\prob} \right) \left(\frac{1-p_{\text{min}}}{p_{\text{min}}}\right) \avgcl \norm{\grad[\cl]{\param[]{1}}{} - \mem[\cl]{1}}^2
    \right],
\end{align}
where Eq.~\eqref{eq:thm_fedhist_min} follows comparing the minimum expected squared gradient norm across iterations to the average.

Observing that $\lyap{T} \geq \obj[]{T}$, we group errors into common, $(1-\beta^2)$-specific, and $\beta^2$-specific terms:
\begin{align}
    &\min_{t\in[1,T]} \E \norm{\grad[]{\param[]{t}}{}}^2 \notag \\
    &\leq
    \frac{4\left( \obj[]{1} - \E[\obj[]{T}]\right)}{\lr[] T} \notag \\
    &\quad
    + \frac{4\lr[] L}{N} \left( \avgcl \frac{1}{\prob} \right) \frac{\sigma^2}{K}
    + 4 \lr[c]^2 L^2 K(K - 1) \frac{\sigma^2}{K} 
    + \frac{48 \lr[] \lr[c]^2 L^3 K(K - 1)}{N} \left( \avgcl \frac{1-\prob}{\prob} \right) \frac{\sigma^2}{K} \notag \\
    &\quad
    + 16 \lr[c]^2 L^2 K(K - 1) \sigma_g^2
    + \frac{192 \lr[] \lr[c]^2 L^3 K(K - 1)}{N} \left( \avgcl \frac{1-\prob}{\prob} \right) \sigma_g^2
    + \frac{12\lr[]L}{NT} \left( \avgcl \frac{1-\prob}{\prob}\right) \sigma_g^2
    \notag \\
    &\quad 
    + (1-\beta)^2 \left[ \frac{24\lr[] L}{N} \left( \avgcl \frac{1-\prob}{\prob}\right) \right] \sigma_g^2 
    \notag \\ 
    &\quad
    + \beta^2 \left[ 
    48\lr[] L \left( \avgcl \frac{1-\prob}{\prob} \right)\left( \avgcl \prob \right)\left(\frac{1}{p_{\text{min}}}\right)
    \right]
    \frac{\sigma^2}{K}
    \notag \\
    &\quad
    + \beta^2 \left[ 
    96 \lr[] L \left( \avgcl \frac{1-\prob}{\prob} \right) 
    \left( \avgcl \prob \right) \left( \frac{1}{p_{\text{min}}} \right)
    \right] \lr[c]^2 L^2 K(K - 1) \frac{\sigma^2}{K} 
    \notag \\
    &\quad
    + \beta^2 \left[ 
    384 \lr[] L \left( \avgcl \frac{1-\prob}{\prob} \right) 
    \left( \avgcl \prob \right) \left( \frac{1}{p_{\text{min}}} \right)
    \right] \lr[c]^2 L^2 K(K - 1) \sigma_g^2
    \notag \\
    &\quad
    + \beta^2 \left[
    \frac{48 \lr[] L}{T} \left( \avgcl \frac{1-\prob}{\prob} \right) \left(\frac{1-p_{\text{min}}}{p_{\text{min}}}\right) \avgcl \norm{\grad[\cl]{\param[]{1}}{} - \mem[\cl]{1}}^2
    \right].
\end{align}
Finally, organizing errors into initial, stochastic gradient-specific, and data heterogeneity-specific terms yields Eq.~\eqref{eq:thm_fedhist}.
\end{proof}

\newpage
\section{FedStale, Lower bound}

To prove oracle complexities lower bounds for smooth objectives, we consider the function used by Nesterov~\cite{nesterovIntroductoryLecturesConvex2004, bubeckConvexOptimizationAlgorithms2015}
\begin{align}
    F(\param{})
    &= \frac{L}{8} \left( \param{}^\top A_{2t+1} \param{} - 2 \param{}^\top e_1 \right) \\
    &= \frac{L}{8} \left[ (\param[]{})_1^2 + \sum_{j=1}^{2t} \left((\param[]{})_j - (\param[]{})_{j+1}\right)^2 + (\param[]{})_{2t+1}^2 - 2 (\param[]{})_1 \right]
\end{align}
where, for $t \leq \frac{d-1}{2}$, $A_t \in \R^{d \times d}$ is a symmetric and tridiagonal matrix defined as
\begin{align}
    \left(A_t\right)_{ij} = 
    \begin{cases}
        2, & i=j, i \leq t \\
        -1, & |i - j| = 1, i \leq t, j \neq t+1 \\
        0, & \text{otherwise.}
    \end{cases}
\end{align}

Following the methodology introduced in~\cite{scamanOptimalConvergenceRates2019}, our approach relies on distributing the objective function $F(\param{})$ across two clients. This split is designed such that most components of the parameter vector $\param[]{t}{}$ remain zero. Local gradient computations increase the number of non-zero components by at most one whenever a client becomes active, without any additional component revealed until the other client participates. More rigorously, let $\cl_0,\cl_1\in\mathcal{N}$ denote two clients. For every client $\cl\in\mathcal{N}$, we define the objective functions $F_\cl(\param{}): \R^d \rightarrow \R$ as follows:
\begin{align}
    \label{eq:objectives_split}
    F_\cl(\param{}) = 
    \begin{cases}
        \frac{NL}{8} \left[ (\param[]{})_1^2 + \sum_{j=1}^{t} \left((\param[]{})_{2j} - (\param[]{})_{2j+1}\right)^2 - 2 (\param[]{})_1 \right] & \text{if}~i=i_0 \\
        \frac{NL}{8} \left[ \sum_{j=1}^{t} \left((\param[]{})_{2j-1} - (\param[]{})_{2j}\right)^2 + (\param[]{})_{2t+1}^2 \right] & \text{if}~i=i_1 \\
        0 & \text{otherwise.}
    \end{cases}
\end{align}

It is easy to verify that $F(\param{}) = \avgcl F_\cl(\param{})$.

\subsection{Upper bound on $k^{(t)}$}

Our objective is to establish an upper bound for the maximum non-zero index of the parameter vector $\param[]{t}$, which minimizes $F(\param[]{})$ at any given time $t$. We introduce the notations
\begin{align*}
    k_{i}^{(t)} = \max \{k \in \mathbb{N}, k \leq d \mid \exists \param[\cl]{t} \in \R^d \text{ such that } (\param[\cl]{t})_k \neq 0\}
\end{align*}
and 
\begin{align*}
    k^{(t)} = \max \{k \in \mathbb{N}, k \leq d \mid \exists \param[]{t} \in \R^d \text{ such that } (\param[]{t})_k \neq 0\}
\end{align*}
to respectively represent the largest index of non-zero components in $\param[\cl]{t}$ and $\param[]{t}$. 

At the beginning of each round $t$, the server initializes $k_{i}^{(t)}$ to $k^{(t-1)}$ for any participating client $\cl \in \mathcal{S}^{(t)}$. Subsequently, after the local computations conclude, the server updates $k^{(t)}$ to the maximum of all $k_{i}^{(t)}$ values among participating clients, which corresponds to the largest index of non-zero components discovered up to time $t$.

Extending the results of~\cite[Lemma~23]{scamanOptimalConvergenceRates2019} to scenarios with partial client participation, it can be easily shown that:
\begin{align}
    k_{i}^{(t+1)} \leq 
    \begin{cases} 
    k_{i}^{(t)} + \mathds{1}\{k_{i}^{(t)} \equiv 0 \bmod 2\} & \text{if } i = \cl_0 \land \cl_0 \text{ is active}, \\
    k_{i}^{(t)} + \mathds{1}\{k_{i}^{(t)} \equiv 1 \bmod 2\} & \text{if } i = \cl_1 \land \cl_1 \text{ is active}, \\
    k_{i}^{(t)} & \text{otherwise}.
    \end{cases}
    \label{eq:k_it_increment}
\end{align}

Due to the stochastic nature of client participation, the sequences $k_{i}^{(t)}$ and $k^{(t)}$, for $\cl \in \mathcal{N}$ and $t \in \mathcal{T}$, must be treated as stochastic processes. Initially, to facilitate the analysis, we assume a deterministic model for client participation. Subsequently, by leveraging Jensen's inequality, we generalize our results to account for the expected stochastic dynamics of client participation.

In what follows, we consider, without loss of generality, that client $\cl_0$ is always active ($p_{0} = 1$), and client $\cl_1$ exhibits the minimum availability ($p_{1} = \min_{i} \prob = p$). 
In terms of upper bound on $k^{(t)}$, this assumption represents a boundary condition on how fast $k^{(t)}$ can grow in a framework characterized by minimum client participation $p$. 

From the two previous assumptions, we observe that client $\cl_1$ exhibits a deterministic activity pattern, participating in every $\tau = 1/p > 1$ communication cycles. 
We can straightforwardly derive from~\eqref{eq:k_it_increment} the subsequent update rule for $k^{(t)}$:
\begin{align}
    \label{eq:kt_increment}
    k^{(t+1)} = k^{(t)} + \mathds{1}\{t \bmod \tau \in \{ 0, 1\} \}, \quad \forall t \in \mathcal{T}
\end{align}

In the example below, we clarify the notation and illustrate the increment process of $k^{(t)}$.
\begin{example}
\label{example:deterministic_participation}
    We set $p=1/4$, implying $\tau=4$. Moreover, we assume $\cl_1$ is active at rounds $t=\{1,5,9,\dots\}$. \\
    
    \begin{center}
    \begin{tikzpicture}
    % Constants
    \def\tickHeight{0.2} % Height of the ticks
    \def\bigRadius{0.1} % Radius for the bigger circles (unavailable times)
    \def\smallRadius{0.07} % Radius for the smaller circles (available times)
    
    % Draw lines for clients and add continuation dots
    \draw (-0.5,0) -- (10.5,0) node[right] {$\dots$};
    \draw (-0.5,-1) -- (10.5,-1) node[right] {$\dots$};
    \node at (-1,0) {$i_0$};
    \node at (-1,-1) {$i_1$};
    
    % Draw ticks and time values
    \foreach \x in {0,...,10} {
        \draw (\x,\tickHeight) -- (\x,-\tickHeight); % Time ticks
        \draw (\x,-1+\tickHeight) -- (\x,-1-\tickHeight); % Time ticks corrected

        \node at (\x,-1.5) {\tiny \x}; % Time labels
        
        % Client i_0 availability (always available, small circles)
        \filldraw[black] (\x,0) circle (\smallRadius);
        
        % Client i_1 availability (big circles)
        \ifthenelse{\x=1 \OR \x=5 \OR \x=9}{
            \filldraw[black] (\x,-1) circle (\smallRadius); % Available
        }{
            \draw[black,fill=white] (\x,-1) circle (\bigRadius); % Unavailable
        }
    }
    \end{tikzpicture}   
    \end{center}

    Given the initialization parameters $\param[]{0}$ with $k^{(0)}=0$, we observe:
    \begin{description}
        \item[$t=0:$] Client $i_0$ is active and increases $k_{0}^{(0)} = 1$ ($k^{(0)}$ is even). \uline{$k^{(0)} = 1$}.
        \item[$t=1:$] Client $i_0$ is active but does not increase $k_{0}^{(1)}$ ($k^{(0)}$ is odd). Client $i_1$ is active, $k_{1}^{(1)} = 2$ (since $k^{(0)}$ is odd). \uline{$k^{(1)} = 2$}.
        \item[$t=2:$] Client $i_0$ increases $k_{0}^{(2)}$ to 3 (since $k^{(1)}$ is even). Client $i_1$ is inactive. \uline{$k^{(2)} = 3$}.
        \item[$t=3:$] Client $i_0$ does not increase $k_{0}^{(3)}$ (since $k^{(2)}$ is odd). Client $i_1$ is inactive. \uline{$k^{(3)} = 3$}.
        \item[$t=4:$]  Client $i_1$ is inactive. \uline{$k^{(4)} = 3$}.
        \item[$t=5:$] Client $i_1$ is active and increases $k_{1}^{(5)}$ to 4. \uline{$k^{(5)} = 4$}.
        \item[$t=6:$] Client $i_0$ increases $k_{0}^{(6)}$ to 5 (since $k^{(5)}$ is even). Client $i_1$ is inactive. \uline{$k^{(6)} = 5$}.
        \item[$t=7:$] Client $i_0$ is active but does not increase $k_{0}^{(7)}$ ($k^{(6)}$ is odd). Client $i_1$ is inactive. \uline{$k^{(7)} = 5$}.
        \item[$t=8:$] Client $i_1$ is inactive. \uline{$k^{(8)} = 5$}.
        \item[$t=9:$] Client $i_1$ is active and increases $k_{1}^{(9)}$ to 6. \uline{$k^{(9)} = 6$}.
        \item[$t=10:$] Client $i_0$ increases $k_{0}^{(10)}$ to 7 (since $k^{(9)}$ is even). Client $i_1$ is inactive. \uline{$k^{(10)} = 7$}.
    \end{description}

    These observations verify $k^{(t+1)} = k^{(t)} + \mathds{1}\{t \bmod 4 \in \{ 0, 1\} \}$, for all $t \in \mathcal{T}$.
    % and
    % \begin{align*}
    %     \forall t \geq 0, 
    %     \quad k_{t} = 1 + \floor*{\frac{t+2}{4}} + \floor*{\frac{t+3}{4}}.
    % \end{align*}
\end{example}
Moreover, Example~\ref{example:deterministic_participation} motivates two further observations:
\begin{enumerate}
    \item Client $\cl_1$'s initial activation (say $j$) produces $\tau$ distinct temporal sequences (say $k^{(t,j)}$), each delayed from the reference sequence $k^{(t,1)}$ up to $\tau-1$ time steps. Specifically, the following relationship holds: 
     \begin{align*}
        k^{(t,j)} = k^{(t - \left((j + \tau - 1) \bmod \tau\right),1)}, \quad \text{for } j=0,\dots,\tau-1.
    \end{align*}
    In other words, all sequences $k^{(t,j)}$, for $j \neq 1$, progress at a slower pace than $k^{(t,1)}$; thus, the sequence $k^{(t,1)}$ represents the fastest evolution among them. 
    \item Given the objectives partition among clients as detailed in Eq.~\eqref{eq:objectives_split}---with clients $\cl_0$ and $\cl_1$ optimizing even and odd components, respectively---the optimization process in Example~\ref{example:deterministic_participation} initiates with client $\cl_0$ targeting the first (even) component at $t=0$. Should we invert their objectives in Eq.~\eqref{eq:objectives_split}, then client $\cl_1$ would commence at $t=0$. This switch initiates $\tau$ new sequences, $\hat{k}^{(t,j)}$, distinct from the original $k^{(t,j)}$ sequences. For $j=1,\ldots,\tau-1$, each $\hat{k}^{(t,j)}$ is one value below its $k^{(t,j)}$ counterpart except for $j=0$, where:
    \begin{align*}
        (k^{(t,1)})_i = \begin{cases}
            (\hat{k}^{(t,0)})_i, & \text{if } i \bmod \tau \in \{0,1\}, \\
            (\hat{k}^{(t,0)})_i + 1, & \text{otherwise}.
        \end{cases}
    \end{align*}
    In any case, $k^{(t,1)}$ remains consistently the fastest sequence among $\hat{k}^{(t,j)}$ for $j=0,\ldots,\tau-1$.
\end{enumerate}
For notational brevity, we denote $k^{(t,1)}$ simply as $k^{(t)}$ in subsequent discussions.
The following lemma proves that at least $\mathcal{O}\left(\tau\right)$ communication rounds are necessary in-between every gradient computation, in order to optimize the global objective.

\begin{lemma}
    \label{lem:kt_upper_bound}
    Given a set $\mathcal{N}$ of $N$ clients, with $i_0$ (having $p_{0} = 1$) always participating, and $i_1$ (with $p_{1} = \min_{i \in \mathcal{N}} p_i = p$) being the least participating client at a period $\tau=1/p > 1$, let $k^{(t)} = \max \{k \in \mathbb{N}, k \leq d \mid \exists \param[]{t} \in \R^d \text{ such that } (\param[]{t})_k \neq 0\}$ represent the largest index of non-zero components in any parameter vector $\param[]{t} \in \mathbb{R}^d$. Assuming the global objective $F(\param{})$ is partitioned among clients as specified in Eq.~\eqref{eq:objectives_split}, the upper bound for the index $k^{(t)}$ at any time $t \geq 0$ is given by:
    \begin{align}
        \label{eq:k_t_upper_bound}
        k^{(t)} \leq 1 + \left\lfloor\frac{t+\tau-2}{\tau}\right\rfloor + \left\lfloor\frac{t+\tau-1}{\tau}\right\rfloor.
    \end{align}
\end{lemma}
\begin{proof}[Proof of Lemma~\ref{lem:kt_upper_bound}]
    The proof proceeds by induction on the time step $t$.

    \emph{Base case.} At $t=0$, the initial condition yields:
    \begin{align}
        k^{(0)} \leq 1 + \left\lfloor \frac{\tau-2}{\tau} \right\rfloor + \left\lfloor \frac{\tau-1}{\tau} \right\rfloor = 1,
    \end{align}
    since, for $\tau \geq 1$, both the floor terms $\left\lfloor \frac{\tau-2}{\tau} \right\rfloor$ and $\left\lfloor \frac{\tau-1}{\tau} \right\rfloor$ evaluate to zero.

    \emph{Inductive step.} Assume Eq.~\eqref{eq:k_t_upper_bound} is valid for an arbitrary $t \geq 0$. Our goal is to show that the relationship holds for $t+1$ as well. From the induction hypothesis for $k^{(t)}$, we have:
    \begin{align}
        k^{(t+1)} &\leq 1 + \left\lfloor \frac{t+\tau-1}{\tau} \right\rfloor + \left\lfloor \frac{t+\tau}{\tau} \right\rfloor \\
        &= k^{(t)} + \left\lfloor \frac{t}{\tau} + 1 \right\rfloor - \left\lfloor \frac{t-2}{\tau} + 1 \right\rfloor \label{eq:kt_recursion_proof},
    \end{align}
    where, in~\eqref{eq:kt_recursion_proof}, we applied the definition of $k^{(t)}$ from Eq.~\eqref{eq:k_t_upper_bound}.

    Next, we observe that the difference $\left\lfloor \frac{t+\tau}{\tau} \right\rfloor - \left\lfloor \frac{t+\tau-2}{\tau} \right\rfloor$ only depends on the congruence class of $t \bmod \tau$, as the
    $\tau$ term simplifies in the subtraction. Specifically,
    % common multiple of $\tau$ cancels out in the subtraction.
    \begin{itemize}
    \item The expression $\left\lfloor \frac{t \bmod \tau}{\tau} + 1 \right\rfloor$ consistently equals one, since $0 \leq \frac{t \bmod \tau}{\tau} < 1$.
    \item Conversely, $\left\lfloor \frac{t \bmod \tau - 2}{\tau} + 1 \right\rfloor$ is one for $t \bmod \tau \geq 2$, and zero for $t \bmod \tau \in \{0,1\}$.
    \end{itemize}

    Thus, the difference $\left\lfloor \frac{t}{\tau} + 1 \right\rfloor - \left\lfloor \frac{t-2}{\tau} + 1 \right\rfloor$ equals one for $t \equiv 0, 1 \pmod{\tau}$, and zero otherwise. \\
    This observation aligns with the incremental rule of Eq.~\eqref{eq:kt_increment}, thus concluding the proof. 
\end{proof}

\subsection{Lower bound on $\norm{\nabla F(\param{t})}^2$}

\begin{lemma}
    \label{thm:lower_bound}
    For any time step $t \leq \frac{d-1}{2}$ in a $d$-dimensional space, and Lipschitz constant $L>0$, there exists an $L$-smooth convex function $F: \mathbb{R}^d \rightarrow \mathbb{R}$, for which the minimum squared norm of the gradient, evaluated at any point within the first $t$ steps of any first-order black-box optimization procedure, satisfies:
    \begin{align}
        \min_{1\leq s\leq t} \norm{\nabla F(\param{s})}^2 
        \geq
        \frac{6 L \left( F(\param{1}) - F^* \right)}{ t (2t+1)^2},
    \end{align}
    where $\param{s}$ represents the parameter vector at step $s$, and $F^*$ denotes the minimum value of $F$.
\end{lemma}

\begin{proof} [Proof of Lemma~\ref{thm:lower_bound}]
We begin by recalling the global objective function~\cite{nesterovIntroductoryLecturesConvex2004, bubeckConvexOptimizationAlgorithms2015}
\begin{align*}
    F(\param{}) = \frac{L}{8} \param{}^\top A_{2t+1} \param{} - \frac{L}{4} \param{}^\top e_1,
\end{align*}
where $A_t \in \mathbb{R}^{d \times d}$ is the symmetric, tridiagonal matrix, defined for $t \leq \frac{d-1}{2}$ as:
\begin{align*}
    (A_t)_{ij} = 
    \begin{cases}
        2, & \text{if } i=j \text{ and } i \leq t, \\
        -1, & \text{if } |i - j| = 1 \text{ and } i \leq t, j \neq t+1, \\
        0, & \text{otherwise}.
    \end{cases}
\end{align*}

\begin{proposition}
    The function $F(\param{})$ satisfies Assumption~\ref{asm:smoothness}.
\end{proposition}
\begin{proof}
    The proof follows directly from~\cite[Theorem 3.14]{bubeckConvexOptimizationAlgorithms2015}.
\end{proof}

Our objective is to derive a lower bound on the squared norm of the gradient $\nabla F(\param{t})$, specifically, the minimum gradient norm observed up to and including time step $t$.

Given the black-box procedure's assumption, we note that $\param{t}$ is restricted to the linear span of $\mathbf{e}_1, \ldots, \mathbf{e}_{t-1}$, implying:
\begin{align*}
    \param[]{t} = \left( (\param[]{t})_1, \dots, (\param[]{t})_{t-1}, 0, \dots, 0 \right).
\end{align*}

% From Lemma~\ref{lem:kt_upper_bound}, we know that $\left(\param{t}\right)_{i}=0$ for $i > k^{(t)}$. In other words, $\param{t}$ must lie in the linear span of $e_1,\dots,e_{k^{(t)}}$:
% \begin{align*}
%     \param[]{t} = \left( (\param[]{t})_1, \dots, (\param[]{t})_{k^{(t)}}, 0, \dots, 0 \right).
% \end{align*}

We define $\param{t,*} = \argmin\limits_{\param{} \in \Span{\mathbf{e}_1,\dots,\mathbf{e}_{t-1}}} \norm{\nabla F(\param{})}^2$, and $\param{*} = \argmin\limits_{\param{} \in \R^{d}} \norm{\nabla F(\param{})}^2$. The following inequality holds:
\begin{align*}
    \min_{1\leq s\leq t} \norm{\nabla F(\param{s})}^2 
    \geq 
    \norm{\nabla F(\param{t,*})}^2
    \geq
    \norm{\nabla F(\param{*})}^2
\end{align*}

Moving forward in the derivation of the lower bound, our focus shifts towards evaluating $\norm{\nabla F(\param{t,*})}^2$. The first step involves identifying $\param{t,*}$, the parameter vector that minimizes the squared norm of the gradient within the span of $\{\mathbf{e}_1,\dots,\mathbf{e}_{t-1}\}$.

\subsubsection{Finding $\param{t,*}$}

Considering $\param{t}\in\Span{\mathbf{e}_1,\dots,\mathbf{e}_{t-1}}$, we calculate the gradient of $F$ at $\param{t}$ as follows:
\begin{align*}
    \frac{\partial F(\param{t})}{\partial(\param{t})_{i}} = \begin{cases}
        \frac{L}{4} \left[ -1 + 2(\param{t})_{i} - (\param{t})_{i+1} \right] & \text{for } i=1, \\
        \frac{L}{4} \left[ -(\param{t})_{i-1} + 2(\param{t})_{i} - (\param{t})_{i+1} \right] & \text{for } i=2,\dots,t-2, \\
        \frac{L}{4} \left[ -(\param{t})_{i-1} + 2(\param{t})_{i} \right] & \text{for } i=t-1, \\
        \frac{L}{4} \left[ -(\param{t})_{i} \right] & \text{for } i=t, \\
        0 & \text{for } i=t+2,\dots,2t+1,
    \end{cases}
\end{align*}
where the gradient evaluations explicitly reflect their dependence on adjacent components $i-1$, $i$, and $i+1$, as a consequence of the structural properties of the symmetric, tridiagonal matrix $A$, and the boundary conditions imposed on the vector $\param{t}$. \\
The squared gradient norm, $\norm{\nabla F(\param{t})}^2$, is then given by:
\begin{align*}
    \norm{\nabla F(\param{t})}^2 
    = 
    &~\frac{L^2}{16} \Biggl[ 
    \left( 2(\param{t})_{1} - (\param{t})_{2} - 1 \right)^2 
    + \sum_{i=2}^{t-2} \left( (\param{t})_{i-1} + 2(\param{t})_{i} - (\param{t})_{i+1} \right)^2  \notag \\
    &~\qquad+\left( (\param{t})_{t-2} + 2(\param{t})_{t-1} \right)^2
    + (\param{t})_{t-1}^2
    \Biggr].
\end{align*}

Minimizing this expression with respect to $\param{t}$ involves setting its partial derivatives to zero, leading to the system of equations: 
\begin{align*}
    \frac{\partial \norm{\nabla F(\param{t})}^2}{\partial (\param{t})_i} = \begin{cases}
        \frac{L^2}{16} \left[-4 + 10(\param{t})_{i} - 8(\param{t})_{i+1} + 2(\param{t})_{i+2} \right] & \text{for } i=1, \\
        \frac{L^2}{16} \left[2 - 8(\param{t})_{i-1} + 12(\param{t})_{i} - 8(\param{t})_{i+1} + 2(\param{t})_{i+2}\right] & \text{for } i=2, \\
        \frac{L^2}{16} \left[2(\param{t})_{i-2} - 8(\param{t})_{i-1} + 12(\param{t})_{i} - 8(\param{t})_{i+1} + 2(\param{t})_{i+2}\right] & \text{for } i=3,\dots,t-3, \\
        \frac{L^2}{16} \left[2(\param{t})_{i-2} - 8(\param{t})_{i-1} + 12(\param{t})_{i} - 8(\param{t})_{i+1}\right]  & \text{for } i=t-2, \\
        \frac{L^2}{16} \left[2(\param{t})_{i-2} - 8(\param{t})_{i-1} + 12(\param{t})_{i}\right] & \text{for } i=t-1.
    \end{cases}
\end{align*}
The minimizer of $\norm{\nabla F(\param{t})}^2$ now depends on adjacent indices $i-2$, $i-1$, $i$, $i+1$, and $i+2$:
% \begin{align*}
%     (\param{t,*})_{i+2} = \begin{cases}
%         4 (\param{t,*})_{i+1} - 5 (\param{t,*})_{i} + 2 & \text{for } i=1, \\
%         4 (\param{t,*})_{i+1} - 6 (\param{t,*})_{i} + 4 (\param{t,*})_{i-1} - 1 & \text{for } i=2, \\
%         4 (\param{t,*})_{i+1} - 6 (\param{t,*})_{i} + 4 (\param{t,*})_{i-1} - (\param{t,*})_{i-2} & \text{for } i=3,\dots,t-5, \\
%         (\param{t,*})_{i+1} - \frac{3}{10} (\param{t,*})_{i} & \text{for } i=t-4, \\
%         \frac{2}{3} (\param{t,*})_{i+1} - \frac{1}{6} (\param{t,*})_{i} & \text{for } i=t-3.
%     \end{cases}
% \end{align*}
\begin{align*}
    (\param{t,*})_{i} = \begin{cases}
        \hphantom{-}\frac{2}{5} + \frac{4}{5} (\param{t,*})_{i+1} - \frac{1}{5} (\param{t,*})_{i+2} & \text{for } i=1, \\
        -\frac{1}{6} + \frac{2}{3} (\param{t,*})_{i-1} + \frac{2}{3} (\param{t,*})_{i+1} - \frac{1}{6} (\param{t,*})_{i+2} & \text{for } i=2, \\
        -\frac{1}{6} (\param{t,*})_{i-2} + \frac{2}{3} (\param{t,*})_{i-1} + \frac{2}{3} (\param{t,*})_{i+1} - \frac{1}{6} (\param{t,*})_{i+2} & \text{for } i=3,\dots,t-3, \\
        -\frac{1}{6} (\param{t,*})_{i-2} + \frac{2}{3} (\param{t,*})_{i-1} + \frac{2}{3} (\param{t,*})_{i+1} & \text{for } i=t-2, \\
        -\frac{1}{6} (\param{t,*})_{i-2} + \frac{2}{3} (\param{t,*})_{i-1} & \text{for } i=t-1.
    \end{cases}
\end{align*}
The optimal vector $\param{t,*}$ emerges as solution of this system, relating the $i$-th component directly to $(\param{t,*})_{1}$ and $(\param{t,*})_{2}$. This yields the following recursive formula for $i=3,\dots,t-1$:
\begin{align*}
    (\param{t,*})_{i} = \frac{1}{6}(i^3-i)(\param{t,*})_{2} - \frac{1}{3}(i^2-4)(\param{t,*})_{1} + \frac{1}{6}(i^3-7i+6).
\end{align*}
Finally, leveraging the boundary conditions on $(\param{t,*})_{t-2}$ and $(\param{t,*})_{t-1}$, we solve for the initial components $(\param{t,*})_{1}$ and $(\param{t,*})_{2}$. The generalized expression for $\param{t,*}$ is:
\begin{align*}
    (\param{t,*})_{i} = \begin{cases}
    \frac{2t^3 - 3(i-1)t^2 - (3i-1)t + i^3 - i}{t(t+1)(2t+1)} & \text{for } i=1,\dots,t-1, \\
    0 & \text{otherwise,}
    \end{cases}
\end{align*}
within the linear span of $\mathbf{e}_1,\dots,\mathbf{e}_{t-1}$.

\subsubsection{Evaluating $\norm{\nabla F(\param{t,*})}^2$}

To complete the lower bound, we first derive the explicit form of the gradient of $F$ at $\param{t,*}$: 
\begin{align*}
    \frac{\partial F(\param{t,*})}{\partial (\param{t,*})_{i}} = \begin{cases}
        \frac{L}{4} \left[ 2 (\param{t,*})_{i} - (\param{t,*})_{i+1} - 1 \right] & \text{for } i=1, \\
        \frac{L}{4} \left[ - (\param{t,*})_{i-1} + 2 (\param{t,*})_{i} - (\param{t,*})_{i+1} \right] & \text{for } i=2,\dots,t-2, \\
        \frac{L}{4} \left[ - (\param{t,*})_{i-1} + 2 (\param{t,*})_{i} \right] & \text{for } i=t-1, \\
        \frac{L}{4} \left[ - (\param{t,*})_{i-1} \right] & \text{for } i=t, \\
        0 & \text{for } i=t+1,\dots,2t+1.
    \end{cases}
\end{align*}
This yields the gradient's $i$-th component as:
\begin{align*}
    \frac{\partial F(\param{t,*})}{\partial (\param{t,*})_{i}} = \begin{cases}
        -\frac{3L i}{2t(t+1)(2t+1)} & \text{for } i=1,\dots,t \\
        0 & \text{for } i > t.
    \end{cases}
\end{align*}
Subsequently, the squared norm of the gradient, $\norm{\nabla F(\param{t,*})}^2$, is calculated as follows:
\begin{align*}
    \norm{\nabla F(\param{t,*})}^2 
    &=
    \sum_{i=1}^{t} \left( -\frac{3L i}{2t(t+1)(2t+1)} \right)^2 \\
    &=
    \frac{9 L^2}{4t^2(t+1)^2(2t+1)^2} \sum_{i=1}^{t} i^2 \\
    &=
    \frac{3 L^2}{8t(t+1)(2t+1)},
\end{align*}
by summing the squares of the gradient components and observing that $\sum_{i=1}^t i^2 = \frac{1}{6}t (t + 1) (2 t + 1)$.

% Additionally, considering the norm of $\param{*}$:
% \begin{align*}
%     \norm{\param{1} - \param{*}}^2 
%     = 
%     \sum_{i=1}^{2t+1} \left[ (\param{2t+1,*})_{i} \right]^2
%     % =
%     % \sum_{i=1}^{2t+1} \frac{\left[ 16t^3 - 12(i-3)t^2 - 2(9i-13)t + i^3 - 7i + 6 \right]^2}{4(8t^3 + 18t^2 + 13t + 3)^2}
%     = 
%     \frac{t(264t^3 + 512t^2 + 288t + 31)}{70(8t^2 + 18t^2 + 13t + 3)} 
%     % \leq
%     % \frac{33t}{70} 
%     \leq 
%     \frac{t}{2},
% \end{align*}
Additionally, considering the initial error:
\begin{align*}
    F(\param{1}) - F^* = 0 - \frac{L}{8} \left( 1 - \frac{1}{2t+2} \right) = \frac{L(2t+1)}{16(t+1)},
\end{align*}
% we derive the final expression:
% \begin{align*}
%     \min_{1\leq s\leq t} \norm{\nabla f(\param{s})}^2 
%     \geq
%     \norm{\nabla f(\param{t,*})}^2 
%     % &=
%     % \frac{3 L^2}{8t(t+1)(2t+1)} \\
%     % &=
%     % \frac{3 L^2 t}{8t^2(t+1)(2t+1)} \\
%     % &\geq
%     \geq
%     \frac{3 L^2 \norm{\param{1} - \param{}^{*}}^2}{4 t^2 (t+1) (2t+1)},
% \end{align*}
we derive the final expression:
\begin{align*}
    \min_{1\leq s\leq t} \norm{\nabla F(\param{s})}^2 
    \geq
    \norm{\nabla F(\param{t,*})}^2 
    % &=
    % \frac{3 L^2}{8t(t+1)(2t+1)} \\
    % &=
    % \frac{3 L^2 t}{8t^2(t+1)(2t+1)} \\
    % &\geq
    \geq
    \frac{6 L \left( F(\param{1}) - F^* \right)}{ t (2t+1)^2},
\end{align*}
thus concluding the proof.
\end{proof}

\subsection{Proof of Theorem~\ref{corollary:lower_bound}}

% Let $\mathcal{N}=\{1,\ldots,N\}$ be a set of $N$ clients, with the least available client having participation probability $p_{\text{min}} \coloneqq \min_{\cl\in\mathcal{N}} \prob$, and $L>0$.
\begin{theorem}
\label{corollary:lower_bound}
In a federated learning setting involving a set of $N$ clients ($\mathcal{N}=\{1,\ldots,N\}$), where each client $i$ is associated with a participation probability $p_i$, let $p_{\text{min}}$ be the minimum participation probability among these clients, i.e., $p_{\text{min}} \coloneqq \min_{i\in\mathcal{N}} p_i$. 
There exists $N$ local functions $F_i:\mathbb{R}^d \rightarrow \mathbb{R}$, contributing to the global objective function $F$, such that:
\begin{enumerate}
    \item The function $F:\mathbb{R}^d \rightarrow \mathbb{R}$ is $L$-smooth;
    \item Under any first-order black-box optimization procedure up to time-step $t$, $t \leq \frac{d-1}{2}$, the minimum squared norm of the gradient of $F$ evaluated at any parameter vector $\param{s}$ within this time interval satisfies:
    \begin{align}
    \min_{1\leq s\leq t} \E \norm{\nabla F(\param{s})}^2 
    \geq
    \frac{3 L \left( F(\param{1}) - F^* \right)}{ (p_{\text{min}}t+2) (4p_{\text{min}}t+9)^2},
\end{align}
where $\param{s}$ denotes the parameter vector at step $s$, and $F^*$ represents the minimum value of $F$.
\end{enumerate}
\end{theorem}

\begin{proof}[Proof of Theorem~\ref{corollary:lower_bound}]
Given the linear span restriction of $\param{t}$ to the first $k^{(t)}$ basis vectors, we have
\begin{align*}
    \param[]{t} = \left( (\param[]{t})_1, \dots, (\param[]{t})_{k^{(t)}}, 0, \dots, 0 \right).
\end{align*}
From our previous result (Lemma~\ref{thm:lower_bound}), the minimum squared gradient norm is bounded below by:
\begin{align}
        \min_{1\leq s\leq t} \norm{\nabla F(\param{s})}^2 
        \geq
        \frac{6 L \left( F(\param{1}) - F^* \right)}{ (k^{(t)}+1) (2k^{(t)}+3)^2}.
\end{align}

Considering $k^{(t)}$ as a random variable depending on the geometric distribution of success time with expected value $\tau=1/p_{\text{min}}$, and leveraging the upper bound established in Lemma~\ref{lem:kt_upper_bound}, we have:
\begin{align}
    \E[k^{(t)}] 
    \leq
    3(1-p_{\text{min}}) + 2 p_{\text{min}} t
    \leq
    3 + 2 p_{\text{min}} t.
\end{align}
Application of Jensen's inequality to the function $f(x)=\frac{1}{(x+1)(2x+3)^2}$, which is convex over $x\geq0$, completes the proof:
\begin{align}
        \min_{1\leq s\leq t} \E \norm{\nabla F(\param{s})}^2 
        \geq
        \frac{6 L \left( F(\param{1}) - F^* \right)}{ (\E[k^{(t)}]+1) (2\E[k^{(t)}]+3)^2}
        \geq
        \frac{3 L \left( F(\param{1}) - F^* \right)}{(p_{\text{min}}t+2) (4p_{\text{min}}t+9)^2}
        .
\end{align}
\end{proof}

\endgroup

\end{document}